\documentclass[11pt]{now}

\usepackage{algorithmic}
\usepackage{algorithm}
\usepackage{amssymb}
\usepackage{amsmath}
\usepackage{graphicx}
\usepackage[numbers]{natbib}

\newenvironment{dedication}
  {\clearpage           
   \thispagestyle{empty}
   \vspace*{\stretch{1}}
   \itshape             
   \raggedleft          
  }
  {\par 
   \vspace{\stretch{3}} 
   \clearpage           
  }

\usepackage{color}
\usepackage{wrapfig}

\usepackage{xcolor}
\definecolor{dark-red}{rgb}{0.4,0.15,0.15}
\definecolor{dark-blue}{rgb}{0,0,0.7}
\usepackage{url}

\usepackage[numbers]{natbib}
\usepackage{mathtools}
\mathtoolsset{showonlyrefs} 
\usepackage{sidecap} 

\usepackage{xfrac}	

\usepackage{graphicx}
\usepackage{caption}
\usepackage{subcaption} 

\newtheorem{defn}{Definition}[section]
\newtheorem{thm}{Theorem}[section]
\newtheorem{lemma}{Lemma}[section]
\usepackage{theoremref} 
\renewenvironment{proof}{{\noindent \emph{Proof.}}}

\begin{document}

\isbn{xxxxxxxxxxx}

\DOI{xxxxxx}


\abstract{Reinforcement learning is a simple, and yet, comprehensive theory of learning that simultaneously models the adaptive behavior of artificial agents, such as robots and autonomous software programs, as well as  attempts to explain the emergent behavior of biological systems. It also gives rise to computational ideas that provide a powerful tool to solve problems involving sequential prediction and decision making. Temporal difference learning is the most widely used method to solve reinforcement learning problems, with a rich history dating back  more than three decades. For these and many other reasons, developing a complete theory of reinforcement learning, one that is both rigorous and useful has been an ongoing research investigation for several decades. In this paper, we set forth a new vision of reinforcement learning developed by us over the past few years, one that yields mathematically rigorous solutions to longstanding important questions that have remained unresolved: (i) how to design reliable, convergent, and robust reinforcement learning algorithms (ii) how to guarantee that reinforcement learning satisfies pre-specified ``safely" guarantees, and remains in a stable region of the parameter space (iii) how to design ``off-policy" temporal difference learning algorithms in a reliable and stable manner, and finally (iv) how to integrate the study of reinforcement learning into the rich theory of stochastic optimization. In this paper, we provide detailed answers to all these questions using the powerful framework of {\em proximal operators}.

The  most important idea that emerges is the use of {\em primal dual spaces} connected through the use of a {\em Legendre} transform. This allows temporal difference updates to occur in dual spaces, allowing a variety of important technical advantages. The Legendre transform, as we show, elegantly generalizes past algorithms for solving reinforcement learning problems, such as {\em natural gradient} methods, which we show relate closely to the previously unconnected framework of {\em mirror descent} methods. Equally importantly, proximal operator theory enables the systematic development of {\em operator splitting} methods that show how to safely and reliably decompose complex products of gradients that occur in recent variants of gradient-based temporal difference learning. This key technical innovation makes it possible to finally design ``true" stochastic gradient methods for reinforcement learning. Finally, Legendre transforms enable a variety of other benefits, including modeling sparsity and domain geometry. Our work builds extensively on recent work on the convergence of saddle-point algorithms, and on the theory of {\em  monotone operators} in Hilbert spaces,  both in optimization and for variational inequalities. The latter framework, the subject of another ongoing investigation by our group, holds the promise of an even more elegant framework for reinforcement learning. Its explication is currently the topic of a further monograph that will appear in due course.}



\articletitle{Proximal Reinforcement Learning: A New Theory of Sequential Decision Making in Primal-Dual Spaces \footnote{This article is currently not under review for the journal Foundations and Trends in ML, but will be submitted for formal peer review at some point in the future, once the draft reaches a stable ``equilibrium" state. }}

\authorname1{Sridhar Mahadevan, Bo Liu, Philip Thomas,  Will Dabney, Steve Giguere, Nicholas Jacek, Ian Gemp}
\affiliation1{School of Computer Science}
\author1address2ndline{School of Computer Science, 140 Governor's Drive,  }
\author1city{ Amherst, MA}
\author1zip{ 01003}
\author1country{  USA}
\author1email{mahadeva,boliu,imgemp@cs.umass.edu, PThomasCS,sgiguere9,amarack@gmail.com}

\authorname2{Ji Liu}
\affiliation2{Hajim School of Engineering and Applied Sciences,}
\author2address2ndline{Hajim School of Engineering and Applied Sciences, University of Rochester, }
\author2city{Rochester, NY}
\author2zip{ 14627}
\author2country{  USA}
\author2email{jliu@cs.rochester.edu}

\relax \input sample.jtl 
\volume{xx}
\issue{xx}
\copyrightowner{xxxxxxxxx}
\pubyear{xxxx}

\maketitle

\begin{dedication}
Dedicated to Andrew Barto and Richard Sutton for inspiring a  generation of researchers to the study of reinforcement learning.

\begin{figure}[ht]
\centering
\begin{minipage}[t]{0.7\textwidth}
\includegraphics[height=0.15\textheight,width=1in]{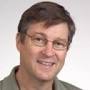} \hskip 0.5in \includegraphics[height=0.15\textheight,width=1in]{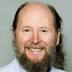}
\end{minipage}
\end{figure}

\begin{algorithm}
\caption{TD (1984)}
\begin{enumerate}
\item $\delta_t = r_t + \gamma {\phi'_t}^T \theta_t - \phi_t^T \theta_t$
\item $\theta_{t+1} = \theta_t + \beta_t \delta_t$
\end{enumerate}
\end{algorithm}
\vspace{1in}
\begin{algorithm}
\caption{GTD2-MP (2014)}
\begin{enumerate}
\item ${w_{t+\frac{1}{2}}}={{w_{t}}+{\beta_{t}}({\delta_{t}}-\phi_{t}^{T}{w_{t}}){\phi_{t}}},\\ \;{\theta_{t+\frac{1}{2}}}={{\rm{pro}}{{\rm{x}}_{{\alpha _t}h}}}\left({{\theta_{t}}+{\alpha_{t}}({\phi_{t}}-\gamma{\phi'_{t}})(\phi_{t}^{T}{w_{t}})}\right)$
\item ${\delta_{t+\frac{1}{2}}}={r_{t}}+\gamma{\phi'_{t}}^{T}{\theta_{t+\frac{1}{2}}}-\phi_{t}^{T}{\theta_{t+\frac{1}{2}}}$
\item $\begin{array}{l}
{w_{t+1}}={w_{t}}+{\beta_{t}}({\delta_{t+\frac{1}{2}}}-\phi_{t}^{T}{w_{t+\frac{1}{2}}}){\phi_{t}}
{\;,\;}\\
{\theta_{t+1}}={{\rm{pro}}{{\rm{x}}_{{\alpha _t}h}}}\left({{\theta_{t}}+{\alpha_{t}}({\phi_{t}}-\gamma{\phi'_{t}})(\phi_{t}^{T}{w_{t+\frac{1}{2}}})}\right)
\end{array}
$\end{enumerate}
\end{algorithm}

\end{dedication}

\cleardoublepage \pagenumbering{roman}

\tableofcontents

\clearpage

\setcounter{page}{0}
\pagenumbering{arabic}

\newtheorem{theorem}{Theorem}[chapter]
\newtheorem{definition}{Definition}[chapter]

\chapter{Introduction}
\label{intro}

In this chapter, we lay out the elements of our novel framework for reinforcement learning \cite{sutton-barto:book}, based on doing temporal difference learning not in the primal space, but in a {\em dual space} defined by a so-called {\em mirror map}. We show how this technical device holds the fundamental key to solving a whole host of unresolved issues in reinforcement learning, from designing stable and reliable off-policy algorithms, to making algorithms achieve safety guarantees, and finally to making them scalable in high dimensions.  This new vision of reinforcement learning developed by us over the past few years yields mathematically rigorous solutions to longstanding important questions in the field, which have remained unresolved for almost three decades. We introduce the main concepts in this chapter, from {\em proximal operators} to the {\em mirror descent} and the {\em extragradient method} and its non-Euclidean generalization, the {\em mirror-prox} method. We introduce a powerful decomposition strategy based on {\em operator splitting}, exploiting deep properties of monotone operators in Hilbert spaces. This technical device, as we show later, is  fundamental in designing ``true" stochastic gradient methods for reinforcement learning, as it helps to decompose the complex product of terms that occur in recent work on gradient temporal difference learning. We provide examples of the benefits of our framework, showing each of the four key pieces of our solution: the improved performance of our new off-policy temporal difference methods over previous gradient TD methods, like TDC and GTD2  \cite{Sutton09fastgradient-descent};  how we are able to generalize natural gradient actor critic methods using mirror maps, and achieve safety guarantees to control learning in complex robots; and finally, elements of our saddle point reformulation of temporal difference learning. The goal of this chapter is to lay out the sweeping power of our primal dual framework for reinforcement learning. The details of our approach, including technical proofs, algorithms, and experimental validations are relegated to future chapters.

\section{Elements of the Overall Framework}

\subsection{Primal Dual Mirror Maps}

In this section, we provide a succinct explanation of the overall framework, leaving many technical details to future chapters. Central to the proposed framework is the notion of {\em mirror maps}, which facilitates doing temporal learning updates not just in the usual primal space, but also in a dual space. More precisely, $\Phi: {\cal D} \rightarrow \mathbb{R}$ for some domain ${\cal D}$ is a {\em mirror map} if it is strongly convex, differentiable, and the gradient of $\Phi$ has the range $\mathbb{R}^n$ (i.e., takes on all possible vector values).  Instead of doing gradient updates in the primal space, we do gradient updates in the dual space, which correspond to:

\[ \nabla \Phi(y) = \nabla \Phi(x) - \alpha \nabla f(x) \]

The step size or learning rate $\alpha$ is a tunable parameter. To get back to the primal space, we use the conjugate mapping $\nabla \Phi^*$, which can be shown to also correspond to the inverse mapping $(\nabla \Phi)^{-1}$, where the conjugate of a function $f(x)$ is defined as

\[ f^*(y) = \sup_x \left( \langle x, y \rangle - f(x) \right). \]

Here $\langle x, y \rangle = x^T y$, the standard inner product on $\mathbb{R}^n$. When $f(x)$ is differentiable and smooth, the conjugate function $f^*(y)$ achieves the maximum value at $x^* = \nabla f(x)$. This is a special instance of the ``Legendre" transform \cite{boyd}. To achieve ``safety" guarantees in reinforcement learning, such as ensuring a robot learning a task never moves into dangerous values of the parameter space, we need to ensure that when domain constraints are not violated. We use Bregman divergences \cite{bregman} to ensure that safety constraints are adhered to, where the projection is defined as:

\[ \Pi^\Phi_{\cal X}(y) = \mbox{argmin}_{{\cal X} \cap {\cal D}} D_\Phi(x,y). \]

A distance generating function $\Phi(x)$ is defined as a strongly convex function which is differentiable. Given such a function $\Phi$, the Bregman divergence associated with it is defined as:

\begin{equation}
\label{bregmand1}
D_\Phi(x,y) = \Phi(x)  - \Phi(y) - \langle \nabla \Phi(y), x - y \rangle
\end{equation}

Intuitively, the Bregman divergence measures the difference between the value of a strongly convex function $\Phi(x)$ and the estimate derived from the first-order Taylor series expansion at $\Phi(y)$. Many widely used distance measures turn out to be special cases of Bregman divergences, such as Euclidean distance (where $\Phi(x) = \frac{1}{2} \| x\|^2$ ) and Kullback Liebler divergence (where $\Phi(x) = \sum_i x_i \log_2 x_i$, the negative entropy function). In general, Bregman divergences are non-symmetric, but projections onto a convex set with respect to a Bregman divergence is well-defined.

\subsection{Mirror Descent, Extragradient, and Mirror Prox Methods}

The framework of {\em mirror descent} \cite{nemirovski-yudin:book,Beck:2003p2359} plays a central role in our framework, which includes not just the original mirror descent method, but also the mirror-prox method \cite{nemirovski2005prox}, which generalizes the {\em extragradient} method to non-Euclidean geometries \cite{extragradient:1976}. Figure~\ref{mdfig} illustrates the mirror descent method, and Figure~\ref{egfig} illustrates the extragradient method.

\begin{figure}[ht]
\centering
\begin{minipage}[t]{0.7\textwidth}
\includegraphics[height=0.35\textheight,width=4in]{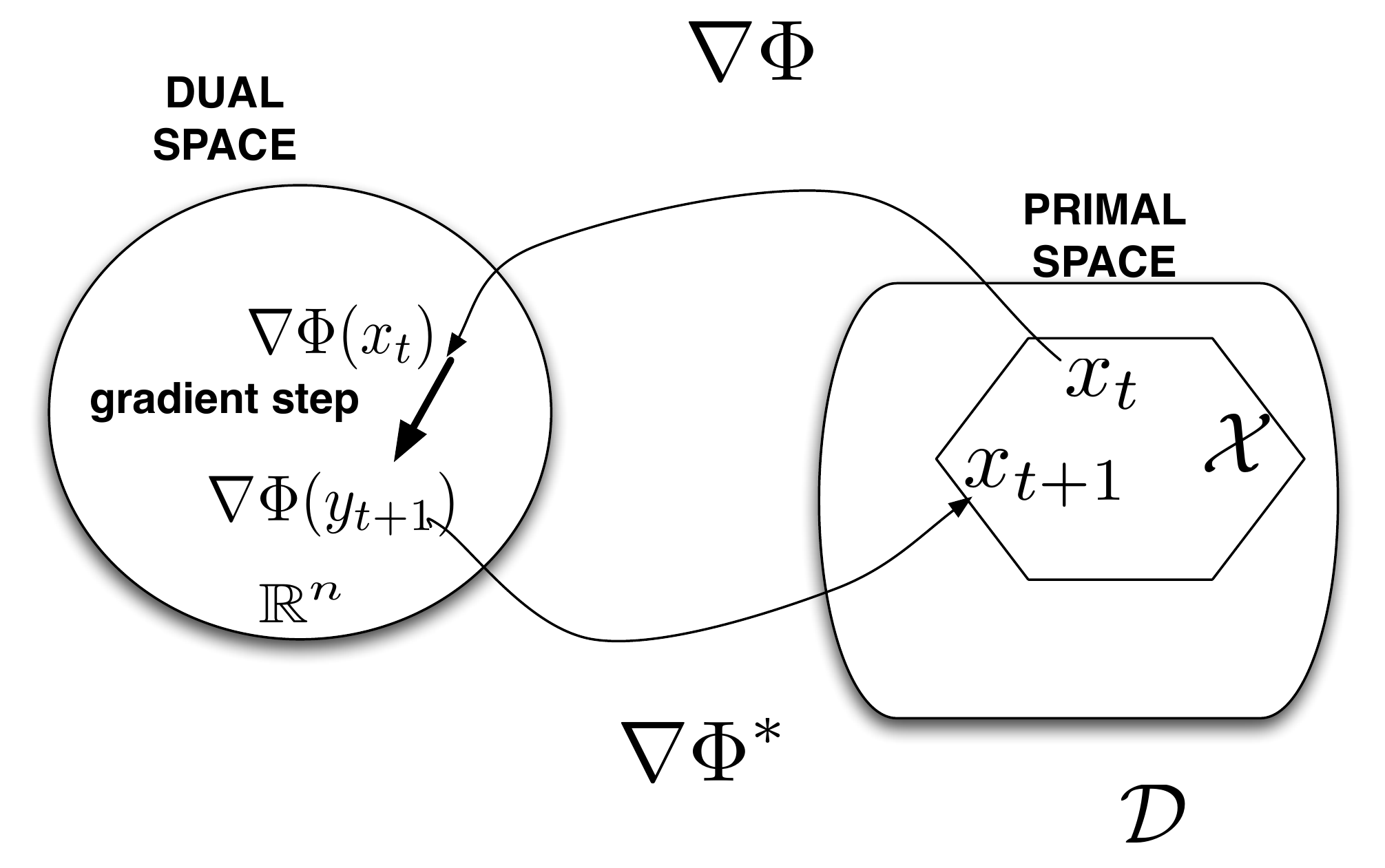}
\end{minipage}
\caption{The mirror descent method. This figure is adapted from \cite{bubeck}.}
\label{mdfig}
\end{figure}

The extragradient method was developed to solve {\em variational inequalities} (VIs), a beautiful generalization of optimization.  Variational inequalities, in the infinite-dimensional setting,  were originally proposed by Hartman and Stampacchia \cite{hartman-stampacchia:acta} in the mid-1960s in the context of solving partial differential equations in mechanics. Finite-dimensional VIs rose in popularity in the 1980s partly as a result of work by Dafermos \cite{dafermos}, who showed that the traffic network equilibrium problem  could be formulated as a finite-dimensional VI. This advance inspired much follow-on research, showing that a variety of equilibrium problems in economics, game theory, sequential decision-making etc. could also be formulated as finite-dimensional VIs -- the books by Nagurney \cite{nagurney:vibook} and Facchinei and Pang \cite{facchinei-pang:vi} provide  a detailed introduction to the theory and applications of finite-dimensional VIs. While we leave the full explication of the VI approach to reinforcement learning to a subsequent monograph, we discuss in the last chapter a few intriguing aspects of this framework that is now the subject of another investigation by our group. A VI(F,K) is specified by a vector field F and a feasible set K. Solving a VI means finding an element $x^*$ within the feasible set K where the vector field $F(x^*)$ is pointed inwards and makes an acute angle with all vectors $x - x^*$. Equivalently, $-F(x^*)$ belongs in the normal cone of the convex feasible set K at the point $x^*$. Any optimization problem reduces to a VI, but the converse is only true for vector fields F whose Jacobians are symmetric. A more detailed discussion of VIs is beyond the scope of this paper, but a longer summary is given in Chapter~\ref{vis}.

In Figure~\ref{egfig}, the concept of extragradient is illustrated. A simple way to understand the figure is to imagine the vector field F here is defined as the gradient $\nabla f(x)$ of some function being minimized. In that case, the mapping $-F(x_k)$ points as usual in the direction of the negative gradient. However, the clever feature of extragradient is that it moves not in the direction of the negative gradient at $x_k$, but rather in the direction of the negative gradient at the point $y_k$, which is the projection of the original gradient step onto the feasible set K. We will see later how this property of extragradient makes its appearance in accelerating gradient temporal difference learning algorithms, such as TDC  \cite{Sutton09fastgradient-descent}.

\begin{figure}[ht]
\centering
\begin{minipage}[t]{0.7\textwidth}
\includegraphics[height=0.25\textheight,width=4in]{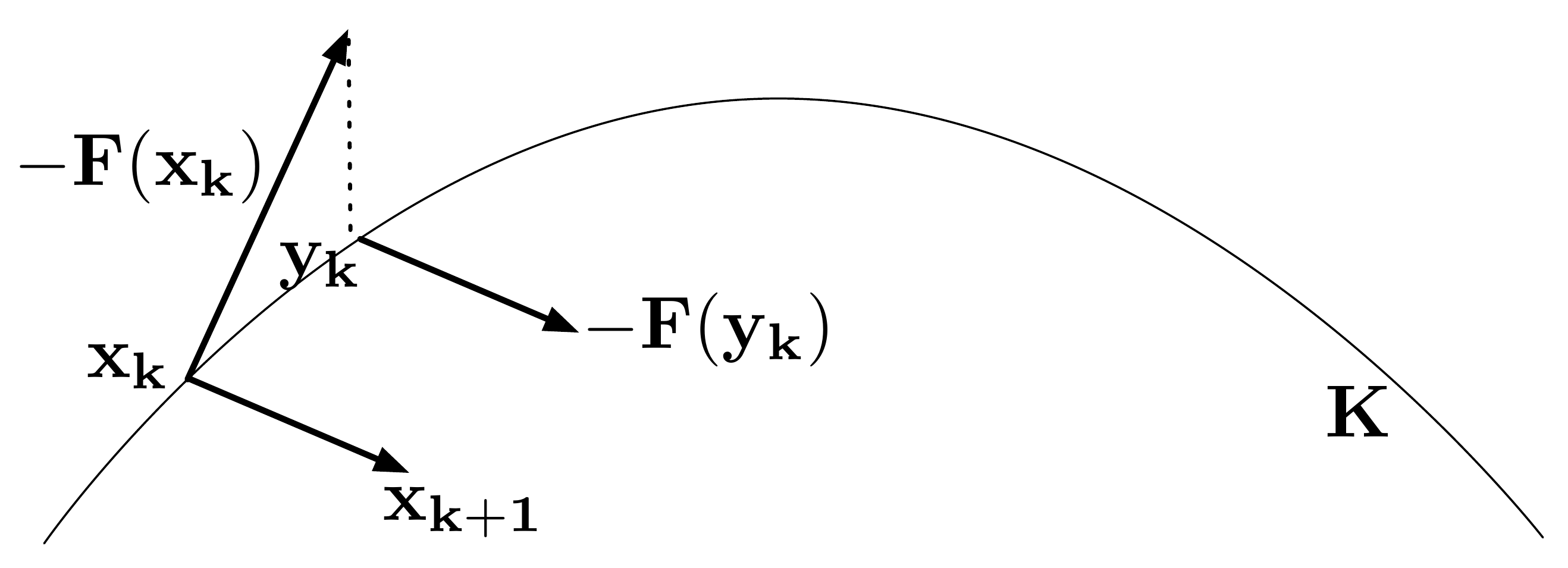}
\end{minipage}
\caption{The extragradient method.}
\label{egfig}
\end{figure}

The mirror-prox method generalizes the extragradient method to non-Euclidean geometries, analogous to the way mirror descent generalizes the regular gradient method. The mirror-prox algorithm (MP) \cite{nemirovski2005prox} is a first-order approach that is able to solve saddle-point problems at a convergence rate of $O(1/t)$.  The MP method plays a key role in our framework as our approach extensively uses the saddle point reformulation of reinforcement learning developed by us \cite{ROTD:NIPS2012}. Figure~\ref{mirrorprox} illustrates the mirror-prox method.

\begin{figure}[ht]
\centering
\begin{minipage}[t]{0.7\textwidth}
\includegraphics[height=0.35\textheight,width=4in]{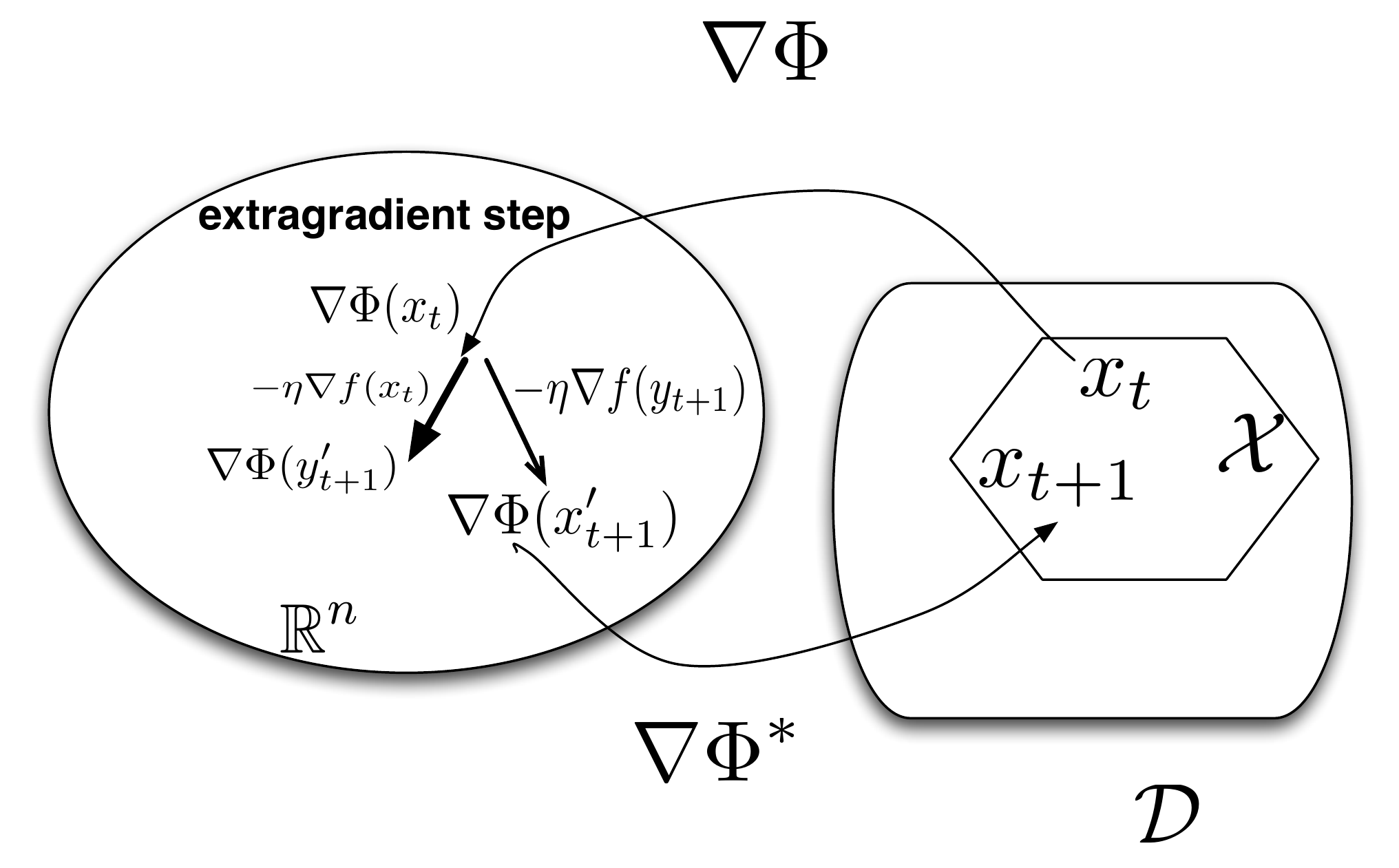}
\end{minipage}
\caption{The mirror prox method. This figure is adapted from \cite{bubeck}.}
\label{mirrorprox}
\end{figure}

\subsection{Proximal Operators}

We now review the concept of proximal mappings, and then describe its relation to the mirror descent framework. The proximal mapping associated with a convex function $h$ is defined as:

\[ \mbox{prox}_h(x) = \mbox{argmin}_{u \in X} \left( h(u) +\frac{1}{2} \| u - x \|^2 \right) \]

If $h(x) = 0$, then $\mbox{prox}_h(x) = x$, the identity function. If $h(x) = I_C(x)$, the indicator function for a convex set $C$, then $\mbox{prox}_{I_C}(x) = \Pi_C(x)$, the projector onto set $C$. For learning sparse representations, the case when $h(w) = \lambda \| w\|_1$ (the $L_1$ norm of $w$) is particularly important. In this case:
\begin{equation}
\label{l1prox}
 \mbox{prox}_h(w)_i = \begin{cases} w_i  - \lambda, & \mbox{if }  w_i > \lambda  \\ 0, & \mbox{if } |w_i| \leq \lambda  \\ w_i + \lambda, & \mbox{otherwise}  \end{cases}
 \end{equation}
An interesting observation follows from noting that the projected subgradient method  can be written equivalently using the proximal mapping as:
\[ w_{k+1} = \mbox{argmin}_{w \in X} \left( \langle w, \partial f(w_k) \rangle + \frac{1}{2 \alpha_k} \| w - w_k \|^2 \right)  \]
where $X$ is a closed convex set. An intuitive way to understand this equation is to view the first term as requiring the next iterate $w_{k+1}$ to move in the direction of the (sub) gradient of $f$ at $w_k$, whereas the second term requires that the next iterate $w_{k+1}$  not move too far away from the current iterate $w_k$.

With this introduction, we can now introduce the main concept of {\em mirror descent}, which was originally proposed by Nemirovksi and Yudin \cite{nemirovski-yudin:book}. We follow the treatment in  \cite{Beck:2003p2359} in presenting the mirror descent algorithm as a nonlinear proximal method based on a distance generator function that is a Bregman divergence \cite{bregman}. The general mirror descent procedure can thus be defined as:
\begin{equation}
\label{md}
 w_{k+1} = \mbox{argmin}_{w \in X} \left( \langle w, \partial f(w_k) \rangle + \frac{1}{\alpha_k} D_\psi(w,w_k) \right)
 \end{equation}
 The solution to this optimization problem can be stated succinctly as the following generalized gradient descent algorithm, which forms the core procedure in mirror descent:
 \begin{equation}
 \label{md-update}
 w_{k+1} = \nabla \psi^* \left( \nabla \psi(w_k) - \alpha_k \partial f(w_k) \right)
 \end{equation}
 %
 An intuitive way to understand the mirror descent procedure specified in Equation~\ref{md-update} is to view the gradient update in two stages: in the first step, the gradient is computed in the dual space using a set of auxiliary weights $\theta$, and subsequently the updated auxilary weights are mapped back into the primal space $w$. Mirror descent is a powerful first-order optimization method that is in some cases ``optimal" in that it leads to low regret. One of the earliest and most successful applications of mirror descent is Positron Emission Tomography (PET) imaging, which involves minimizing a convex function over the unit simplex $X$. It is shown in  \cite{BenTal:2001p2356} that the mirror descent procedure specified in Equation~\ref{md-update}  with the Bregman divergence defined by the {\em p-norm} function \cite{gentile:mlj} can outperform regular projected subgradient method by a factor $\frac{n}{\log n}$ where $n$ is the dimensionality of the space. For high-dimensional spaces, this ratio can be quite large. We will discuss below specific choices of Bregman divergences in the target application of this framework to reinforcement learning.

\subsection{Operator Splitting Strategies}

In our framework, a key insight used to derive a true stochastic gradient method for reinforcement learning is based on the powerful concept of {\em operator splitting} \cite{operator-splitting,douglas-rachford}. Figure~\ref{operator-splitting} illustrates this concept for the {\em convex feasibility} problem, where we are given a collection of convex  sets, and have to find a point in their intersection. This problem originally motivated the development of Bregman divergences \cite{bregman}. The convex feasibility problem is an example of many real-world problems, such as 3D voxel reconstruction in brain imaging \cite{BenTal:2001p2356}, a high-dimensional problem that mirror descent was originally developed for. To find an element in the common intersection of two sets $A$ and $B$ in Figure~\ref{operator-splitting}, a standard method called {\em alternating projections} works as follows. Given an initial point $x_0$, the first step projects it to one of the two convex sets, say $A$, giving the point $\Pi_A(x_0)$. Since $A$ is convex, this is a uniquely defined point. The next step is to project the new point on the second set $B$, giving the next point $\Pi_B(\Pi_A(x_0))$. The process continues, ultimately leading to the desired point common to the two sets. Operator splitting studies a generalized version of this problem, where the projection problem is replaced by the proximal operator problem, as described above. Many different operator splitting strategies have been developed, such as Douglas Rachford splitting \cite{douglas-rachford}, which is a generalization of widely used distributed optimization methods like Alternating Direction Method of Multipliers \cite{boyd:admm}. We will see later that using a sophisticated type of operator splitting strategy, we can address the problem of off-policy temporal difference learning.

\begin{figure}[ht]
\centering
\begin{minipage}[t]{0.7\textwidth}
\includegraphics[height=0.35\textheight,width=4in]{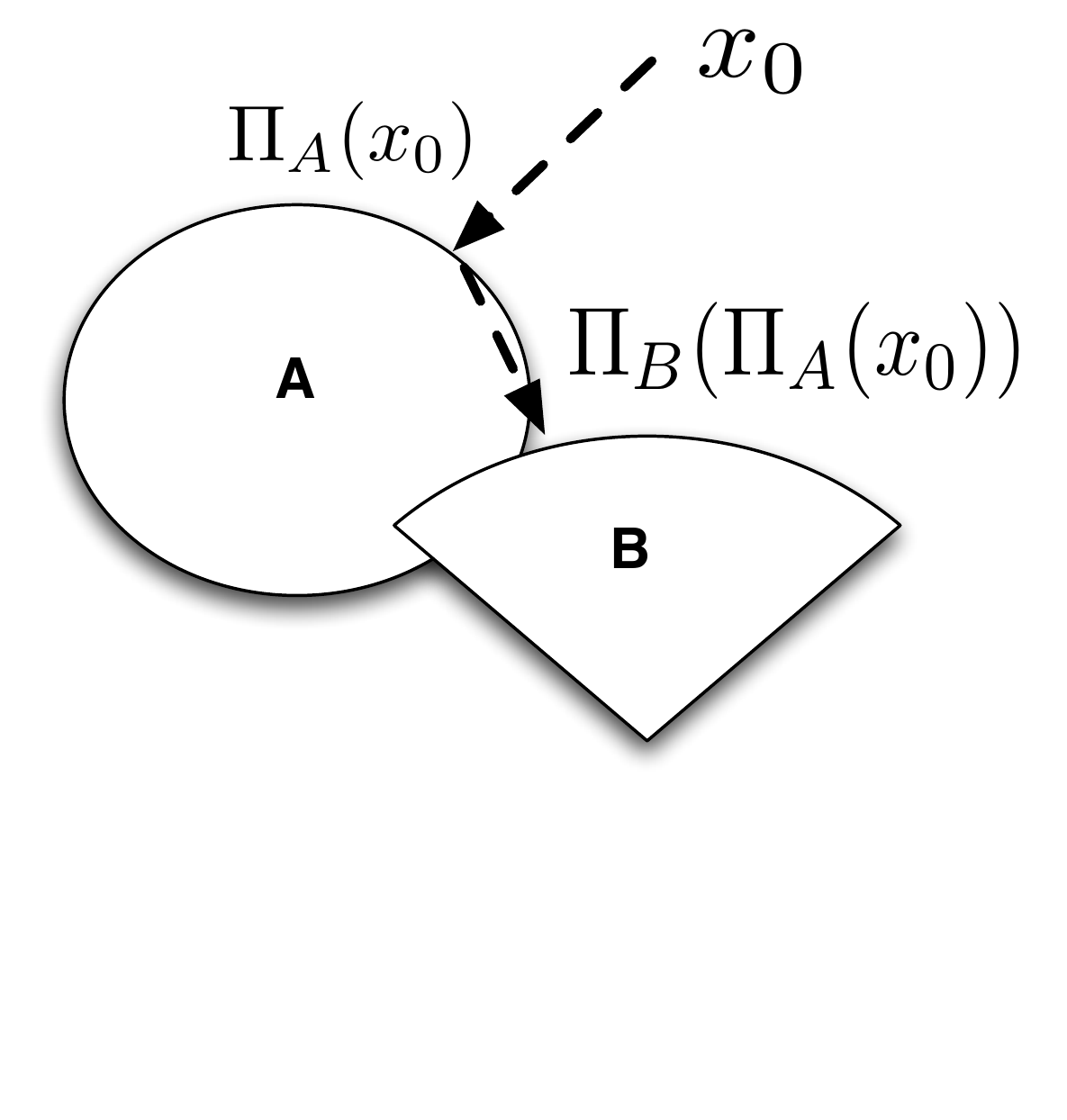}
\end{minipage}
\caption{Operator splitting strategy  for the convex feasibility problem.}
\label{operator-splitting}
\end{figure}

\section{Illustrating the Solution}

Now that we have described the broad elements of our framework, we give a few select examples of the tangible solutions that emerge to the problem of designing safe, reliable, and stable reinforcement learning algorithms. We pick three cases: how to design a ``safe" reinforcement learning method; how to design a ``true" stochastic gradient reinforcement learning method; and finally, how to design a ``robust" reinforcement learning method that does not overfit its training experience.

\section{Safe Reinforcement Learning}

\begin{figure}
\centering
\includegraphics[width=.5\columnwidth]{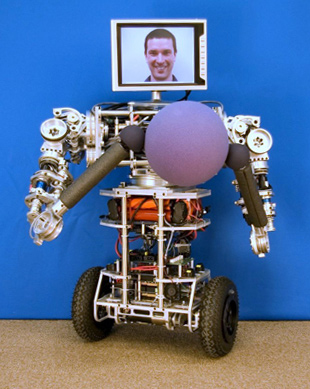}
\caption{The uBot-5 is a 11 degree of freedom mobile manipulator developed at the Laboratory of Perceptual Robotics (LPR) at the University of Massachusetts, Amherst \cite{Deegan2010,Kuindersma2009}. How can we design a ``safe" reinforcement learning algorithm which is guaranteed to ensure that policy learning will not violate pre-defined constraints such that such robots will operate in dangerous regions of the control parameter space? Our framework provides a key solution, based on showing an equivalence between mirror descent and a previously well-studied but unrelated  algorithm called {\em natural gradient} \cite{Amari1998b}.}
\label{ch1fig:uBot}
\end{figure}

Figure~\ref{ch1fig:uBot} shows a complex high-degree of freedom humanoid robot. Teaching robots complex skills is a challenging problem, particularly since reinforcement learning not only may take a long time, but also because it may cause such robots to operate in dangerous regions of the parameter space. Our proposed framework solves this problem by establishing a key technical result, stated below, between mirror descent and the well-known, but previously unrelated, class of algorithms called natural gradient \cite{Amari1998b}. We develop the projected natural actor critic (PNAC) algorithm, a policy gradient method that exploits this equivalence to yield a safe  method for training complex robots using reinforcement learning. We explain the significance of the below result connecting mirror descent and natural gradient methods later in this paper when we describe a novel class of methods called projected natural actor critic (PNAC).

\begin{thm}
The natural gradient descent update at step $k$ with metric tensor $G_k \triangleq G(x_k)$:
\begin{equation}
x_{k+1} = x_k - \alpha_k G_k^{-1}\nabla f(x_k),
\label{eq:NG2}
\end{equation}
is equivalent to the mirror descent update at step $k$, with $\psi_k(x)=(\sfrac{1}{2})x^\intercal G_k x$.
\end{thm}

\section{True Stochastic Gradient Reinforcement Learning}

First-order temporal difference  learning is a widely used class of techniques
in reinforcement learning. Although many more sophisticated methods have been developed over the past three decades,
such as least-squares based
temporal difference approaches, including  LSTD \cite{bradtke-barto:LSTD},
LSPE \cite{bertsekas-nedic} and LSPI \cite{lagoudakis:jmlr}, first-order temporal difference
learning algorithms may scale more gracefully to high dimensional problems.
Unfortunately, the initial class of TD methods was known to converge only when samples are drawn
``on-policy". This motivated the development of the gradient TD (GTD)  family of methods \cite{FastGradient:2009}.  A crucial step in the development of our framework was the development of a novel saddle-point framework for sparse regularized GTD \cite{ROTD:NIPS2012}. However, there have been several unresolved questions
regarding the current off-policy TD algorithms.
(1) The first is the convergence rate of these algorithms. Although
these algorithms are motivated from the gradient of an objective function
such as  mean-squared projected Bellman error (MSPBE) and NEU  \cite{FastGradient:2009}, they are not true stochastic gradient methods
with respect to these objective functions, as pointed out in \cite{szepesvari2010algorithms}, which make the convergence rate and error bound analysis difficult, although asymptotic analysis
has been carried out using the ODE approach. (2) The second concern is
regarding acceleration. It is believed that TDC performs the best so far of
the GTD family of algorithms. One may intuitively ask if there are any gradient TD
algorithms that can outperform TDC. (3) The third concern is regarding compactness
of the feasible set $\theta$. The GTD family of algorithms all assume
that the feasible set $\theta$ is unbounded, and if the feasible
set $\theta$ is compact, there is no theoretical analysis and convergence
guarantee. (4) The fourth question is on regularization: although
the saddle point framework proposed in \cite{ROTD:NIPS2012} provides an online regularization framework
for the GTD family of algorithms, termed as RO-TD, it is based on the
inverse problem formulation and is thus not quite explicit. One further
question is whether there is a more straightforward algorithm, e.g,
the regularization is directly based on the MSPBE and NEU objective
functions.

Biased sampling is a well-known problem in reinforcement learning.
Biased sampling is caused by the stochasticity of the policy wherein
there are multiple possible  successor states from the current state
where the agent is. If it is a deterministic policy, then there will
be no biased sampling problem. Biased sampling is often caused
by the product of the TD errors, or the product of TD error and the
gradient of TD error w.r.t the model parameter $\theta$. There are
two ways to avoid the biased sampling problem, which can be categorized
into double sampling methods and two-time-scale stochastic approximation
methods.

In this paper, we  propose a novel approach to TD algorithm design in reinforcement learning, based on
introducing the  {\em proximal splitting}  framework \cite{PROXSPLITTING2011}. We show that the GTD family of algorithms are true stochastic gradient descent (SGD) methods, thus making their convergence
rate analysis available. New accelerated off-policy algorithms are
proposed and their comparative study with RO-TD is carried out to show
the effectiveness of the proposed algorithms. We also show that primal-dual
splitting is a unified first-order optimization framework to solve the
biased sampling problem. Figure~\ref{ch1fig:star} compares the performance of our newly designed off-policy methods compared to previous methods, like TDC and GTD2 on the
classic 5-state Baird counterexample. Note the significant improvement of TDC-MP over TDC: the latter converges much more slowly, and has much higher variance. This result is validated not only by experiments, but also by a detailed theoretical analysis of sample convergence, which goes beyond the previous asymptotic convergence analysis of off-policy methods.

\begin{figure}
\centering
\begin{minipage}{1.15\textwidth}
\includegraphics[width= .45\textwidth, height=2.25in]{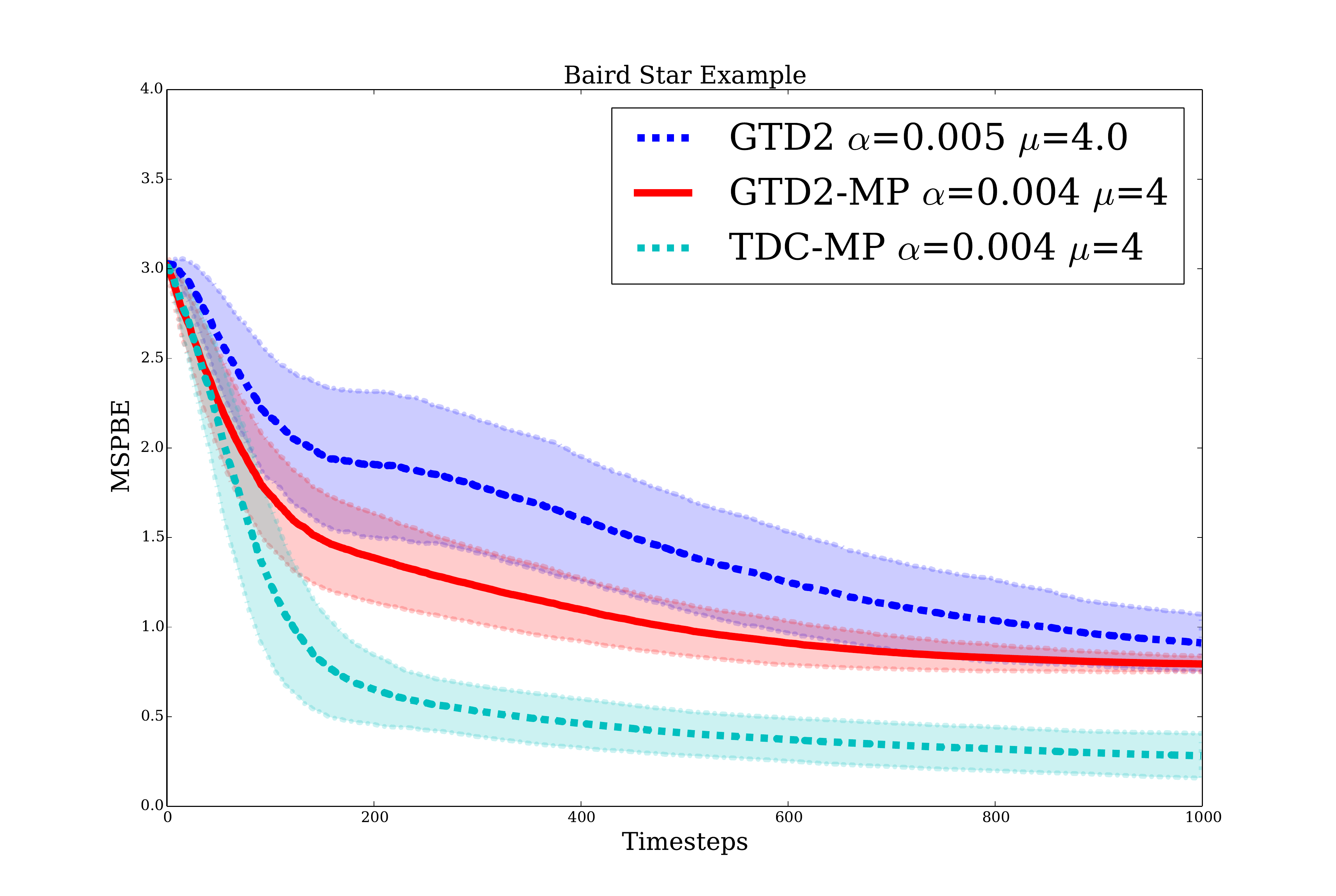}
\includegraphics[width= .45\textwidth, height=2.25in]{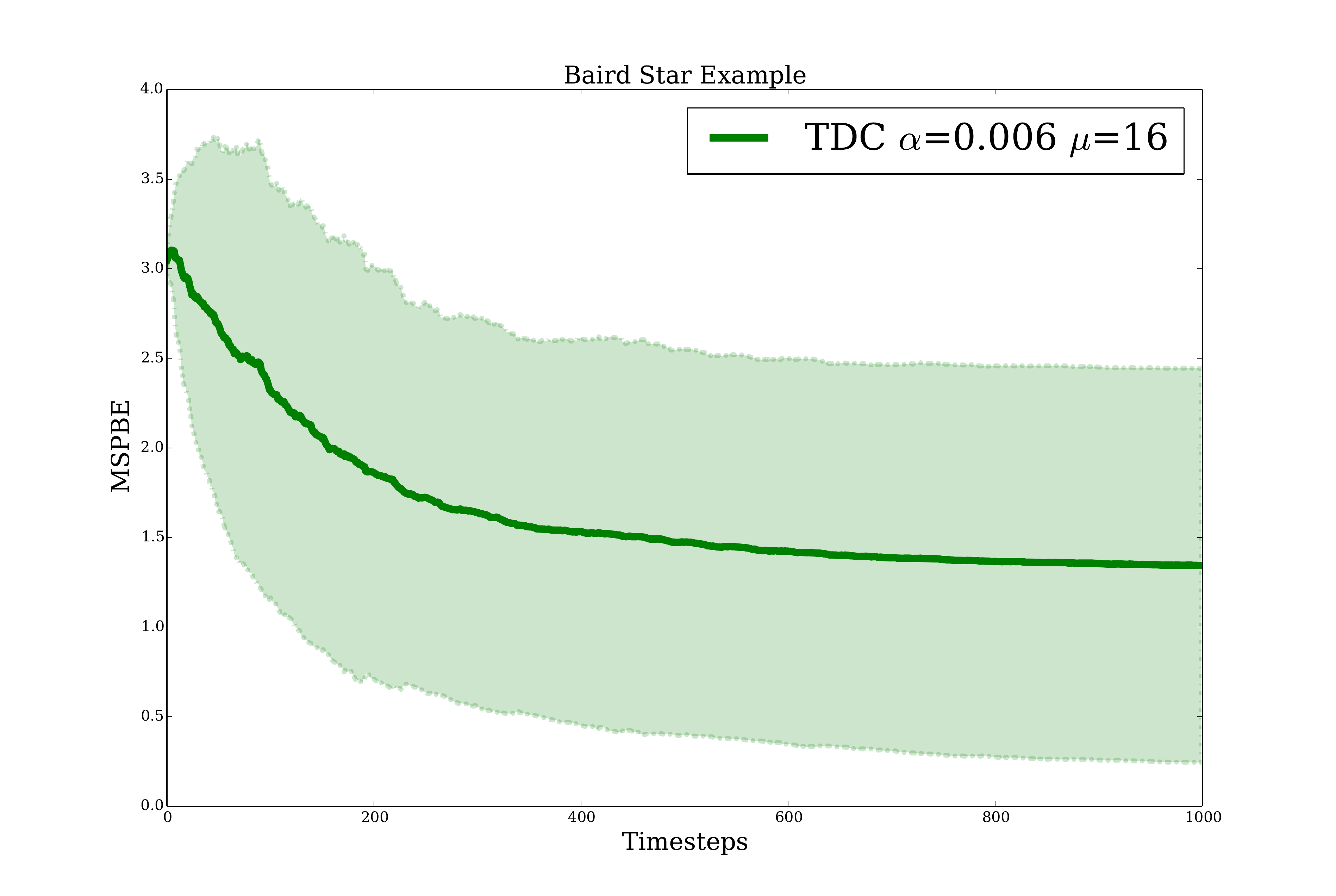}
\end{minipage}
\caption{Off-Policy Convergence Comparison. Our proposed methods, TDC-MP and GTD2-MP, appear to significantly outperform previous methods, like TDC and GTD2 on a simple
benchmark MDP.}
\label{ch1fig:star}
\end{figure}

\section{Sparse Reinforcement Learning using Mirror Descent}

How can we design reinforcement learning algorithms that are robust to overfitting? In this paper we explore a new framework for (on-policy convergent)
TD learning algorithms based on mirror descent and related
algorithms. Mirror descent can be viewed as an enhanced gradient method,
particularly suited to minimization of convex functions in high-dimensional
spaces. Unlike traditional temporal difference learning methods, mirror descent temporal difference learning undertakes
updates of weights in both the dual space and primal space,
which are linked together using a Legendre transform. Mirror descent
can be viewed as a proximal algorithm where the distance-generating
function used is a Bregman divergence. We will present a new class of {\em proximal-gradient}
based temporal-difference (TD) methods based on different
Bregman divergences, which are more powerful than regular TD learning.
Examples of Bregman divergences that are studied include $p$-norm
functions, and Mahalanobis distance based on the covariance of sample
gradients. A new family of sparse mirror-descent reinforcement learning
methods are proposed, which are able to find sparse fixed-point of
an $l_{1}$-regularized Bellman equation at significantly less computational
cost than previous methods based on second-order matrix methods. Figure~\ref{ch1mcar-experiment} illustrates a sample result, showing how the mirror descent variant of
temporal difference learning results in faster convergence, and much lower variance (not shown) on the classic mountain car task \cite{sutton-barto:book}.

\begin{figure}[tbh]
\centering
\begin{minipage}[t]{0.8\textwidth}
 \includegraphics[height=2.5in, width=4in]{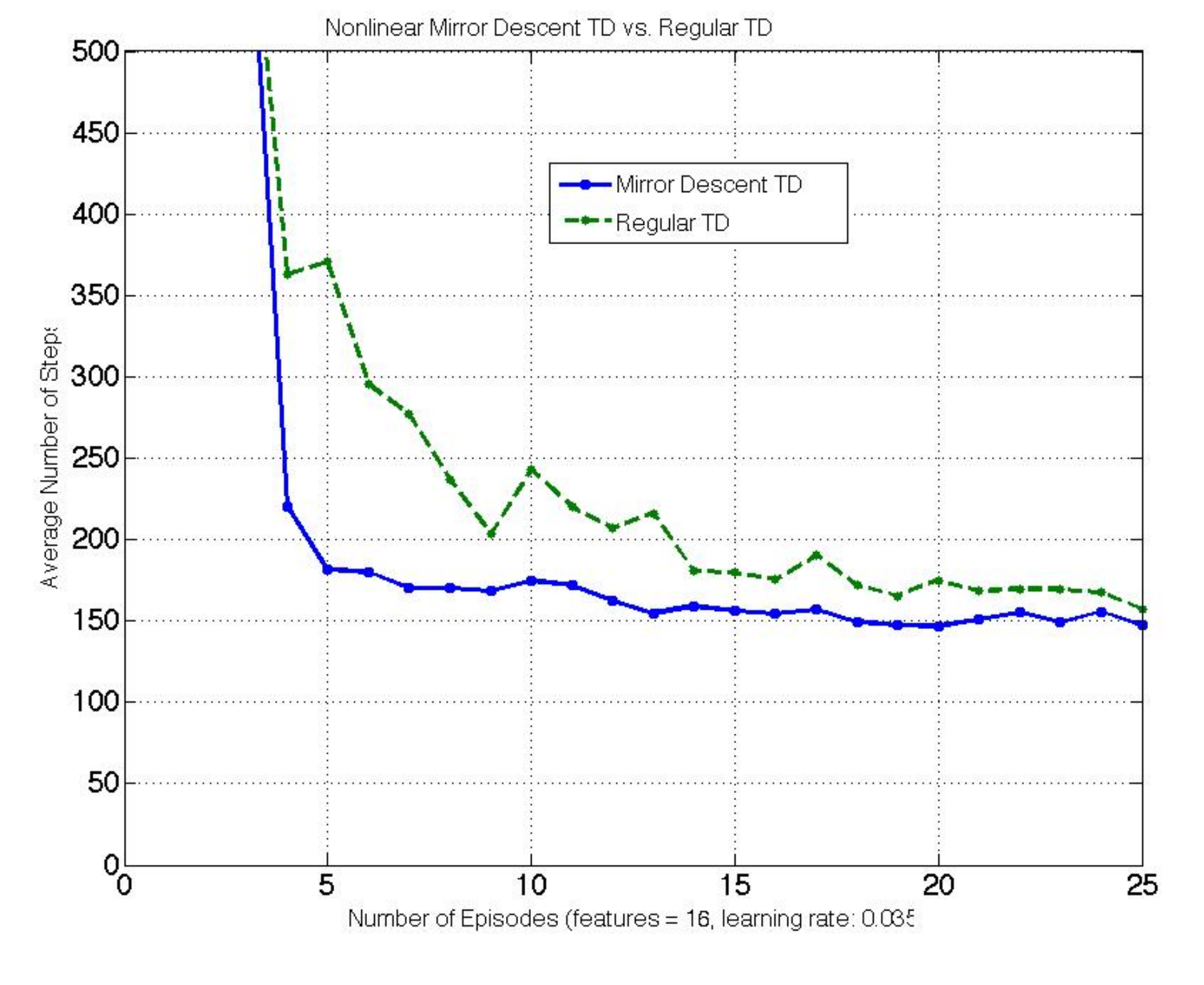}
\end{minipage}
\caption{Comparing mirror-descent TD using the p-norm link function with $16$ tunable Fourier bases with regular TD for the mountain car task. }
\label{ch1mcar-experiment}
\end{figure}


\section{Summary}

We provided a brief overview of our proposed primal-dual framework for reinforcement learning.  The fundamentally new idea underlying the approach is the systematic use of mirror maps to
carry out temporal difference updates, not in the original primal space, but rather in a {\em dual} space. This technical device, as we will show in subsequent chapters, provides for a number
of significant advantages. By choosing the mirror map carefully, we can generalize popular methods like natural gradient based actor-critic methods, and provide safety guarantees. We
can design more robust temporal difference learning methods that are less prone to overfitting the experience of an agent. Finally, we can exploit proximal mappings to design a rich variety
of true stochastic gradient methods. These advantages, when combined, provide a compelling case for the fundamental correctness of our approach. However, much remains to be done in more fully
validating the proposed framework on large complex real-world applications, as well as doing a deeper theoretical analysis of our proposed approach. These extensions will be the subject of ongoing
research by us in the years ahead.

\chapter{Background}\label{chapter:Background}

In this chapter we introduce relevant background material that form the two cornerstones of this paper: reinforcement learning and first-order stochastic composite optimization. The Markov decision process (MDP) model, value function approximation and some basics of reinforcement learning are also introduced. For stochastic composite optimization, we first introduce the problem formulation, and then introduce some tools such as proximal gradient method, mirror descent, etc.

\section{Reinforcement Learning}

\subsection{MDP}

The learning environment for decision-making is generally modeled
by the well-known \textbf{\emph{Markov Decision Process}}\cite{puterman} $M=(S,A,P,R,\gamma)$,
which is derived from a Markov chain.
\begin{defn}
(Markov Chain)\label{def:mc}: A \emph{Markov Chain} is a stochastic process defined as $M=(S,P)$. At each time step $t=1,2,3,\cdots$,
the agent is in a state ${s_{t}}\in S$, and the state transition
probability is given by the state transition kernel $P:S\times S\to\mathbb{{R}}$
satisfying $||P|{|_{\infty}}=1$, where $P({s_{t}}|{s_{t-1}})$ is
the state-transition probability from state $s_{t-1}$ at time step
$t-1$ to the state $s_{t}$ at time step $s_{t}$.
\end{defn}
A Markov decision process (MDPs) is comprised of a set of states $S$,
a set of (possibly state-dependent) actions $A$ ($A_{s}$), a dynamical
system model comprised of the transition probabilities $P_{ss'}^{a}$
specifying the probability of transition to state $s'$ from state
$s$ under action $a$, and a reward model $R$.
\begin{defn}
(Markov Decision Process)\label{def:mdp}\cite{puterman}: A Markov
Decision Process is a tuple $(S,A,P,R,\gamma)$ where $S$ is a finite
set of states, $A$ is a finite set of actions, $P:S\times A\times S\to[0,1]$
is the transition kernel, where $P(s,a,s')$ is the probability of
transmission from state $s$ to state $s'$ given action $a$, and
reward $r:S\times A\to{\mathbb{R}{}^{+}}$ is a reward function, $0\leq\gamma<1$
is a discount factor.
\end{defn}

\subsection{Basics of Reinforcement Learning}

A policy $\pi:S\rightarrow A$ is a deterministic (stochastic) mapping
from states to actions.
\begin{defn}
\noindent\label{def:pol}(Policy): A deterministic stationary policy
$\pi:S\to A$ assigns an action to each state of the Markov decision
process. A stochastic policy $\pi:S\times A\to[0,1]$.
\end{defn}
Value functions are used to compare and evaluate the performance of
policies.
\begin{defn}
\noindent\label{def:value}(Value Function): A value function w.r.t
a policy $\pi$ termed as ${V^{\pi}}:S\to\mathbb{R}$ assigns each state the
\emph{expected sum of discounted} rewards
\end{defn}
\begin{equation}
V^{\pi}=\mathbb{E}\left[{\sum\limits _{i=1}^{t}{{\gamma^{i-1}}{r_{i}}}}\right]
\end{equation}

The goal of reinforcement learning is to find a (near-optimal) policy
that maximizes the value function. $V^{\pi}$ is a fixed-point of
the Bellman equation
\begin{equation}
{V^{\pi}}({s_{t}})=\mathbb{E}\left[{r({s_{t}},{\pi(s_{t})})+\gamma{V^{\pi}}({s_{t+1}})}\right]\label{eq:bellman_eq}
\end{equation}

Equation (\ref{eq:bellman_eq}) can be written in a concise form by
introducing the Bellman operator $T^{\pi}$ w.r.t a policy $\pi$
and denoting the reward vector as $R^{\pi}\in{\mathbb{R}^{n}}$ where ${R^{\pi}_{i}}=\mathbb{E}[r({s_{i}},\pi(s_{i}))]$.

\begin{equation}
V^{\pi}=T^{\pi}(V^{\pi})=R^{\pi}+\gamma P^{\pi}V^{\pi}\label{eq:bellman_op}
\end{equation}

Any optimal policy $\pi^{*}$ defines the unique optimal value function
$V^{*}$ that satisfies the nonlinear system of equations:
\begin{equation}
V^{^{*}}(s)=\max_{a}\sum_{s'}P_{ss'}^{a}\left(R_{ss'}^{a}+\gamma V^{*}(s')\right)\label{eq:Vopt}
\end{equation}

\subsection{Value Function Approximation}

The most popular and widely used RL method is temporal difference
(TD) learning \cite{sutton:td}. TD learning is a stochastic approximation
approach to solving Equation (\ref{eq:Vopt}). The {\em state-action
value} $Q^{*}(s,a)$ represents a convenient reformulation of the
value function, defined as the long-term value of performing $a$
first, and then acting optimally according to $V^{*}$:
\begin{equation}
Q^{*}(s,a)=\mathbb{E}\left(r_{t+1}+\gamma\max_{a'}Q^{*}(s_{t+1},a')|s_{t}=s,a_{t}=a\right)\label{qvalues}
\end{equation}
where $r_{t+1}$ is the actual reward received at the next time step,
and $s_{t+1}$ is the state resulting from executing action $a$ in
state $s_{t}$. The (optimal) action value formulation is convenient
because it can be approximately solved by a temporal-difference (TD)
learning technique called Q-learning \cite{watkins:phd}. The simplest
TD method, called TD($0$), estimates the value function associated
with the fixed policy using a normal stochastic gradient iteration,
where $\delta_{t}$ is called temporal difference error:
\begin{equation}
\begin{array}{l}
{V_{t+1}}({s_{t}})={V_{t}}({s_{t}})+{\alpha_{t}}{\delta_{t}}\\
{\delta_{t}}={r_{t}}+\gamma{V_{t}}({s_{t+1}})-{V_{t}}({s_{t}})
\end{array}
\end{equation}

TD($0$) converges to the optimal value function $V^{\pi}$ for policy
$\pi$ as long as the samples are ``on-policy'', namely following
the stochastic Markov chain associated with the policy; and the learning
rate $\alpha_{t}$ is decayed according to the Robbins-Monro conditions
in stochastic approximation theory: $\sum_{t}\alpha_{t}=\infty,\sum_{t}\alpha_{t}^{2}<\infty$
\cite{ndp:book}. When the set of states $S$ is large, it is often
necessary to approximate the value function $V$ using a set of handcrafted
basis functions (e.g., polynomials, radial basis functions, wavelets
etc.) or automatically generated basis functions \cite{mahadevan:mlfnt}.
In linear value function approximation, the value function is assumed
to lie in the linear spanning space of the basis function matrix $\Phi$
of dimension $|S|\times d$, where it is assumed that $d\ll|S|$.
Hence,
\begin{equation}
V^{\pi}\approx V_{\theta}=\Phi\theta
\end{equation}
The equivalent TD($0$) algorithm for linear function approximated
value functions is given as:
\begin{equation}
\begin{array}{l}
{\theta _{t + 1}} = {\theta _t} + {\alpha _t}{\delta _t}\phi ({s_t})\\
{\delta _t} = {r_t} + \gamma \phi {({s_{t + 1}})^T}{\theta _t} - \phi {({s_t})^T}{\theta _t}
\end{array}
\label{td0}
\end{equation}

\section{Stochastic Composite Optimization}

\subsection{Stochastic Composite Optimization Formulation}

Stochastic optimization explores the use of first-order gradient methods
for solving convex optimization problems.
%
%
We first give some definitions before moving on to introduce
stochastic composite optimization.
\begin{defn}
\noindent (Lipschitz-continuous Gradient)\label{def:lipschitz}: The
gradient of a closed convex function $f(x)$ is $L$-Lipschitz continuous
if $\exists{L},||\nabla f(x)-\nabla f(y)||\le{L}||x-y||,\forall x,y\in X$.
\end{defn}

\begin{defn}
\noindent (Strong Convexity): A convex function is $\mu-$strongly
convex if $\exists{\mu}$, $\frac{\mu}{2}||x-y|{|^{2}}\le f(y)-f(x)-\left\langle {\nabla f(x),y-x}\right\rangle ,\forall x,y\in X$.
\end{defn}
\noindent \textbf{Remark}: If $f(x)$ is both with $L$-Lipschitz
continuous gradient and $\mu$-strongly convex, then we have $\forall x,y\in X$,
\[
\frac{\mu}{2}||x-y|{|^{2}}\le f(y)-f(x)-\left\langle {\nabla f(x),y-x}\right\rangle \le\frac{L}{2}||x-y|{|^{2}}
\]

\begin{defn}
(Stochastic Subgradient) \label{def:sg}: The stochastic subgradient
for closed convex function $f(x)$ at $x$ is defined as $g(x,\xi_{t})$
satisfying ${\mathbb{E}}[g(x,{\xi_{t}})]=\nabla f(x)\in\partial f(x)$.
Further, we assume that the variance is bounded $\exists\sigma>0$
such that
\begin{equation}
\forall x\in X,{\mathbb{E}}[||g(x,{\xi_{t}})-\nabla f(x)||_{*}^{2}]\le{\sigma^{2}}\label{eq:sgvariance}
\end{equation}

\end{defn}
Here we define the problem of Stochastic Composite Optimization (SCO)\cite{lan2012optimal}:
\begin{defn}
(Stochastic Composite Optimization): A stochastic composite optimization
problem $\mathcal{F}(L,M,\mu,\sigma):\Psi(x)$ on a closed convex
set $X$ is defined as

\begin{equation}
{\min_{x\in X}}\Psi(x)\mathop=\limits ^{def}f(x)+h(x)
\label{sumfg}
\end{equation}
$f(x)$ is a convex function with $L$-Lipschitz continuous gradient and $h(x)$ is a convex Lipschitz continuous function such that
\begin{equation}
|h(x) - h(y)| \le M||x - y||,\forall x,y \in X
\end{equation}
$g(x,\xi_{t})$ is the stochastic subgradient of $\Psi(x)$ defined above with variance bound $\sigma$.
Such $\Psi(x)$ is termed as a $\mathcal{F}(L,M,\mu,\sigma)$ problem.
\end{defn}
%

\subsection{Proximal Gradient Method and Mirror Descent}

Before we move on to introduce mirror descent, we first introduce
some definitions and notations.
%

\begin{defn}
\noindent (Distance-generating Function)\cite{MID:2003}: A distance-generating function
$\psi(x)$ is defined as a continuously differentiable $\mu$-strongly
convex function.
$\psi^{*}$ is the Legendre transform of $\psi$,
which is defined as ${\psi^{*}}(y)=\mathop{\sup}\limits _{x\in X}\left({\left\langle {x,y}\right\rangle -\psi(x)}\right)$.
\end{defn}

\begin{defn}
\noindent (Bregman Divergence)\cite{MID:2003}: Given distance-generating function
$\psi$, the Bregman divergence induced by $\psi$ is defined as:
\begin{equation}
D_{\psi}(x,y)=\psi(x)-\psi(y)-\langle\nabla\psi(y),x-y\rangle
\label{bregmand}
\end{equation}

Legendre transform and Bregman divergence have the following properties
\begin{itemize}
\item $\nabla{\psi^{*}}={(\nabla\psi)^{-1}}$
\item ${D_{\psi}}(u,v)={D_{{\psi^{*}}}}(\nabla\psi(u),\nabla\psi(v))$
\item $\nabla{D_{\psi}}(u,v)=\nabla\psi(u)-\nabla\psi(v)$
\end{itemize}

An interesting choice of the link function $\psi(\cdot)$ is the $(q-1)$-strongly convex function $\psi(\theta)=\frac{1}{2}\|\theta\|_{q}^{2}$,
and ${\psi^{*}}(\tilde{\theta})=\frac{1}{2}||\tilde{\theta}||_{p}^{2}$.
Here, $\|\theta\|_{q}=\left(\sum_{j}|\theta_{j}|^{q}\right)^{\frac{1}{q}}$,
and $p$ and $q$ are conjugate numbers such that $\frac{1}{p}+\frac{1}{q}=1$
\cite{gentile2003robustness}. $\theta$ and $\tilde{\theta}$ are
conjugate variables in primal space and dual space, respectively .
\begin{eqnarray}
\mathop{\nabla\psi}\limits _{\theta\to\tilde{\theta}}{(\theta)_{j}} & = & \frac{{{\rm {sign}}({\theta_{j}})|{\theta_{j}}{|^{q-1}}}}{{||\theta||_{q}^{q-2}}}\nonumber \\
\;{\mathop{\;\nabla\psi}\limits _{\tilde{\theta}\to\theta}}^{*}{(\tilde{\theta})_{j}} & = & \frac{{{\rm {sign}}({{\tilde{\theta}}_{j}})|{{\tilde{\theta}}_{j}}{|^{p-1}}}}{{||\tilde{\theta}||_{p}^{p-2}}}
\end{eqnarray}
Also it is worth noting that when $p=q=2$, the Legendre transform
is the identity mapping.
\end{defn}
We now introduce the concept of \emph{proximal mapping}, and then
describe the mirror descent framework. The proximal mapping associated
with a convex function $h(x)$ is defined as:
\begin{equation}
prox_{h}(x)=\arg\mathop{\min}\limits _{u\in X}(h(u)+\frac{1}{2}{\left\Vert {u-x}\right\Vert ^{2}})\label{eq:0}
\end{equation}
In the case of $h(x)=\rho{\left\Vert x\right\Vert _{1}}(\rho>0)$,
which is particularly important for sparse feature selection, the
proximal operator turns out to be the soft-thresholding operator ${S_{\rho}}(\cdot)$,
which is an \emph{entry-wise} shrinkage operator that moves a point
towards zero, i.e.,
\begin{equation}
pro{x_{h}}{(x)_{i}}={S_{\rho}}{(x)_{i}}={\rm {sign}}({x_{i}})\max({|x_{i}-\rho|},0)\label{softthresholding}
\end{equation}
where $i$ is the index, and $\rho$ is a threshold. With this background,
we now introduce the proximal gradient method. At each iteration, the
optimization sub-problem of Equation (\ref{sumfg}) can be rewritten
as
\begin{equation}
{x_{t+1}}=\arg\mathop{\min}\limits _{u\in X}(h(u)+\langle\nabla{f_{t}},u\rangle+\frac{1}{{2{\alpha_{t}}}}{\left\Vert {u-x_{t}}\right\Vert ^{2}})\label{eq:proxgradient}
\end{equation}
If computing $prox_{h}$ is not expensive, then computation of Equation
(\ref{sumfg}) is of the following formulation, which is called the \emph{proximal
gradient method}
\begin{equation}
{x_{t+1}}=pro{x_{{\alpha_{t}}h}}\left({{x_{t}}-{\alpha_{t}}\nabla f({x_{t}})}\right)\label{eq:linearproximal}
\end{equation}
where ${\alpha_{t}}>0$ is stepsize, constant or determined by line
search. The mirror descent \cite{MID:2003} algorithm is a generalization
of classic gradient descent, which has led to developments of new
more powerful machine learning methods for classification and regression.
Mirror descent can be viewed as an enhanced gradient method, particularly
suited to minimization of convex functions in high-dimensional spaces.
Unlike traditional gradient methods, mirror descent undertakes gradient
updates of weights in the dual space, which is linked together with
the primal space using a Legendre transform. Mirror descent can be
viewed as a proximal algorithm where the distance-generating function
used is a Bregman divergence w.r.t the distance-generating function
$\psi$, and thus the optimization problem is
\begin{equation}
prox_{h}(x)=\arg\mathop{\min}\limits _{u\in X}(h(u)+D_{\psi}(u,x))\label{eq:0breg}
\end{equation}
The solution to this optimization problem of Equation (\ref{eq:0breg})
forms the core procedure of \emph{mirror descent} as a generalization
of Equation (\ref{eq:proxgradient})
\begin{equation}
{x_{t+1}}=\arg\mathop{\min}\limits _{u\in X}(h(u)+\langle\nabla{f_{t}},u\rangle+\frac{1}{{\alpha_{t}}}{D_{\psi}}(u,{x_{t}}))\label{eq:proxgradientmid}
\end{equation}
which is a nonlinear extension of Equation(\ref{eq:linearproximal})
\begin{equation}
{x_{t+1}}=\nabla{\psi^{*}}\left({pro{x_{{\alpha_{t}}h}}\left({\nabla\psi({x_{t}})-{\alpha_{t}}\nabla f({x_{t}})}\right)}\right)\label{eq:MIDSA}
\end{equation}
Mirror descent has become the cornerstone of many online $l_{1}$
regularization approaches such as in \cite{ShalevShwartz:jmlr}, \cite{RDA:2009}
and \cite{COMID:2010}.

\subsection{Dual Averaging}

Regularized dual averaging (RDA) \cite{RDA:2009} is a variant of
Dual averaging (DA) with ``simple'' regularizers, such as $l_{1}$
regularization. DA method is strongly related to cutting-plane methods.
Cutting-plane methods formulate a polyhedral lower bound model of
the objective function where each gradient from past iterations contributes
a supporting hyperplane w.r.t its corresponding previous iteration,
which is often expensive to compute. The DA method approximates this
lower bound model with an approximate (possibly not supporting) lower
bound hyperplane with the averaging of all the past gradients \cite{gaojianfeng:2013}.

We now explain RDA from the proximal gradient perspective. Thus far,
the proximal gradient methods we have described in Equation (\ref{eq:MIDSA})
adjust the weights to lie in the direction of the current gradient
$\nabla f_{t}$. Regularized dual averaging methods (RDA) uses a (weighted)
averaging of gradients, which explain their name. Compared with Equation
(\ref{eq:MIDSA}), the main difference is the average (sub)gradient
$\nabla{{\bar{f}}_{t}}$ is used, where $\nabla{\bar{f}_{t}}=\frac{1}{t}\sum\limits _{i=1}^{t}{\nabla{f_{i}}}$.
The equivalent space-efficient recursive representation is
\begin{equation}
\nabla{{\bar{f}}_{t}}=\frac{{t-1}}{t}\nabla{{\bar{f}}_{t-1}}+\frac{1}{t}\nabla{f_{t}}\label{barft-1}
\end{equation}
The generalized mirror-descent proximal gradient formulation of RDA
iteratively solves the following optimization problem at each step:

\begin{equation}
{x_{t+1}}=\arg\mathop{\min}\limits _{x\in X}\left\{ {\left\langle {x,\nabla{{\bar{f}}_{t}}}\right\rangle +h(x)+\frac{1}{{\alpha_{t}}}{D_{\psi}}(x)}\right\} \label{eq:RDAproximal-1}
\end{equation}

Note that different from Equation (\ref{eq:proxgradientmid}), besides
the averaging gradient ${\nabla{{\bar{f}}_{t}}}$ is used instead
of ${\nabla{{f}_{t}}}$, a global origin-centered stabilizer ${{D_{\psi}}(x)}$
is used. RDA with local stabilizer can be seen in \cite{mcmahan2011follow}.
There are several advantages of RDA over other competing methods in
regression and classification problems. The first is the sparsity
of solution when the penalty term is $h(x)=\rho||x||_{1}$. Compared
with other first-order $l_{1}$ regularization algorithms of the mirror-descent
type, including truncated gradient method \cite{Lanford:2008} and
SMIDAS \cite{SMIDAS:2009}, RDA tends to produce sparser solutions
in that the RDA method is more aggressive on sparsity than many other
competing approaches. Moreover, many optimization problems can be
formulated as composite optimization, e.g., a smooth objective component
in conjunction with a global non-smooth regularization function. It
is worth noting that problems with non-smooth regularization functions
often lead to solutions that lie on a low-dimensional supporting data
manifold, and regularized dual averaging is capable of identifying
this manifold, and thus bringing the potential benefit of accelerating
convergence rate by searching on the low-dimensional manifold after
it is identified, as suggested in \cite{lee2012manifold}. Moreover, the finite iteration behavior
of RDA is much better than SGD in practice.

\subsection{Extragradient}

The extragradient method was first proposed by Korpelevich\cite{extragradient:1976} as a relaxation
of ordinary gradient descent to solve variational
inequality (VI) problems. Conventional ordinary gradient descent can
be used to solve VI problems only if some strict restrictions
such as strong monotonicity of the operator or compactness of the
feasible set are satisfied. The extragradient method was proposed to solve
VIs to relax the aforementioned strict restrictions. The
essence of extragradient methods is that instead of moving along the
steepest gradient descent direction w.r.t the initial point in each iteration, two steps,
i.e., a extrapolation step and a gradient descent step, are taken. In
the extrapolation step, a step is made along the steepest gradient descent
direction of the initial point, resulting in an intermediate point
which is used to compute the gradient. Then the gradient descent step
is made \emph{from} the initial point in the direction of the gradient
w.r.t the intermediate point. The extragradient take steps as follows
\begin{equation}
\begin{array}{l}
{x_{t+\frac{1}{2}}}={\Pi_{X}}\left({{x_{t}}-\alpha_{t}\nabla f({x_{t}})}\right)\\
{x_{t+1}}={\Pi_{X}}\left({{x_{t}}-\alpha_{t}\nabla f({x_{t+\frac{1}{2}}})}\right)
\end{array}\label{eq:eg}
\end{equation}

$\Pi_{X}(x)=\mbox{argmin}_{y\in X}\|x-y\|^{2}$ is the
projection onto the convex set $X$, and $\alpha_{t}$ is a stepsize.
Convergence of the iterations of Equation (\ref{eq:eg}) is guaranteed
under the constraints $0<{\alpha_{t}}<\frac{1}{{\sqrt{2}L}}$\cite{nemirovski2005prox}, where
$L$ is the Lipschitz constant for $\nabla f({x})$.

\subsection{Accelerated Gradient}

Nesterov's seminal work on accelerated gradient (AC) enables deterministic
smooth convex optimization to reach its optimal convergence rate $O(\frac{L}{{N^{2}}})$.
The AC method consists of three major steps: an interpolation step,
a proximal gradient step and a weighted averaging step. During each
iteration,

\begin{eqnarray}
{y_{t}} & = & {\alpha_{t}}{x_{t-1}}+(1-{\alpha_{t}}){z_{t-1}}\nonumber \\
{x_{t}} & = & \arg\mathop{\min}\limits _{x}\left\{ {\left\langle {x,\nabla f({y_{t}})}\right\rangle +h(x)+\frac{1}{{\beta_{t}}}{D_{\psi}}(x,{x_{t-1}})}\right\} \nonumber \\
{z_{t}} & = & {\alpha_{t}}{x_{t}}+(1-{\alpha_{t}}){z_{t-1}}
\end{eqnarray}

It is worth noting that in the proximal gradient step, the stabilizer
makes $x_{t}$ start from $x_{t-1}$, and go along the gradient
descent direction of $\nabla f({y_{t})}$, which is quite similar
to extragradient. The essence of Nesterov's accelerated gradient
method is to carefully select the prox-center for proximal gradient
step, and the selection of two stepsize sequences $\{{\alpha_{t}},{\beta_{t}}\}$
where $\alpha_{t}$ is for interpolation and averaging, $\beta_{t}$
is for proximal gradient. Later work and variants of Nesterov's method
utilizing the strong convexity of the loss function with Bregman divergence
are summarized in \cite{tseng2008accelerated}. Recently, the extension
of accelerated gradient method from deterministic smooth convex optimization
to stochastic composite optimization, termed as AC-SA, is studied
in \cite{lan2012optimal}.

\section{Subdifferentials and Monotone Operators}

We introduce the important concept of a subdifferential.

\begin{definition}
The {\em subdifferential} of a convex function $f$ is defined as the set-valued mapping $\partial f$:
\[ \partial f(x) = \{v \in \mathbb{R}^n : f(z) \geq f(x) + v^T (z - x), \forall z \in \mbox{dom}(f) \]
\end{definition}
A simple example of a subdifferential is the {\em normal cone}, which is the subdifferential of the indicator function $I_K$ of a convex set $K$ (defined as $0$ within the set and $+\infty$ outside).  More formally, the normal cone  $N_K(x^*)$ at the vector $x^*$ of a convex set $K$ is defined as $N_K(x^*) = \{y \in \mathbb{R}^n | y^T(x - x^*) \leq 0, \forall x \in K \}$. Each vector $v \in \partial f(x)$ is referred to as the {\em subgradient} of $f$ at $x$.

An important property of closed proper convex functions is that their subdifferentials induce a relation on $\mathbb{R}^n$ called a {\em maximal monotone operator} \cite{operator-splitting,rockafellar}.
\begin{definition}
A relation $F$ on $\mathbb{R}^n$ is monotone if
\[ (u - v)^T (x - y) \geq 0 \ \mbox{for all} \ (x, u), (y, v) \in F \]
F is maximal monotone is there is no monotone operator that properly contains it.
\end{definition}
The subdifferential $\partial f$ of a convex function $f$ is a canonical example of a maximal monotone operator. A very general way to formulate optimization problems is {\em monotone inclusion}:
\begin{definition}
Given a monotone operator $F$, the monotone inclusion problem is to find a vector $x$ such that $0 \in F(x)$. For example, given a (subdifferentiable) convex function $f$, finding a vector $x^*$ that minimizes $f$ is equivalent to solving the monotone inclusion problem $0 \in \partial f(x^*)$.
\end{definition}

\section{Convex-concave Saddle-Point First Order Algorithms}
\label{ch2sec:saddle}

A key  novel contribution of our paper is a convex-concave saddle-point formulation for reinforcement learning.
A convex-concave saddle-point problem is formulated as follows. Let
$x\in X,y\in Y$,  where $X,Y$ are both nonempty closed convex sets,  and $f(x):X\to\mathbb{R}$
be a convex function. If there exists a function $\varphi(\cdot,\cdot)$
such that $f(x)$ can be represented as $f(x): = {\sup _{y \in Y}}\varphi (x,y)$,
then the pair $(\varphi,Y)$ is referred as the saddle-point representation
of $f$. The optimization problem
of minimizing $f$ over $X$ is converted into an equivalent convex-concave
saddle-point problem
$SadVal = {\inf _{x \in X}}{\sup _{y \in Y}}\varphi (x,y)$
of $\varphi$ on $X\times Y$. If $f$ is non-smooth
yet convex and well structured, which is not suitable for many existing
optimization approaches requiring smoothness, its saddle-point representation
$\varphi$ is often smooth and convex. The convex-concave saddle-point
problems are, therefore, usually better suited for first-order methods \cite{sra2011optimization}. A comprehensive overview on extending convex minimization to convex-concave saddle-point problems with unified variational inequalities is presented in \cite{ben2005non}.
As an example, consider $f(x) = ||Ax - b|{|_m}$ which admits a bilinear minimax representation
\begin{equation}
f(x): = {\left\| {Ax - b} \right\|_m} = {\max _{{{\left\| y \right\|}_n} < 1}}\left( {\left\langle {y,Ax - b} \right\rangle } \right)
\label{eq:minimax1}
\end{equation}
where $m,n$ are conjugate numbers. Using the approach in \cite{RobustSA:2009}, Equation (\ref{eq:minimax1}) can be solved as
\begin{equation}
{x_{t + 1}} = {x_t} - {\alpha _t}\left\langle {{y_t},A} \right\rangle ,{y_{t + 1}} = {\Pi _{{{\left\| {{y_t}} \right\|}_n} \le 1}}({y_t} + {\alpha _t}(A{x_t} - b))
\label{eq:minimaxFOM}
\end{equation}
where ${\Pi _{{{\left\| {{y_t}} \right\|}_n} \le 1}}$ is the projection operator of $y_t$ onto the unit-$l_n$ ball  ${\left\| y \right\|_n} \le {\rm{1}}$,which is defined as
\begin{equation}\label{ch1_l2proj}
{\Pi _{{{\left\| y \right\|}_n} \le 1}}y = \min (1,1/{\left\| y \right\|_n})y,n = 2,{\left( {{\Pi _{{{\left\| y \right\|}_n} \le 1}}y} \right)_i} = \min (1,\frac{1}{{|{y_i}|}}){y_i},n = \infty
\end{equation}
and ${{\Pi _{{{\left\| y \right\|}_\infty} \le 1}}y}$ is an entrywise operator.

\section{Abstraction through Proximal Operators}
A general procedure for solving the monotone inclusion problem, the {\em proximal point algorithm} \cite{rockafellar:prox}, uses the following identities:
\[ 0 \in \partial f(x) \leftrightarrow 0 \in \alpha \partial f(x) \leftrightarrow x \in (I  + \alpha \partial (x)) \leftrightarrow x  = (I + \alpha \partial f)^{-1}(x)  \]
Here, $\alpha > 0$ is any real number. The proximal point algorithm is based on the last fixed point identity, and consists of the following iteration:
\[ x_{k+1} \leftarrow (I + \alpha_k \partial f )^{-1} (x_k) \]
Interestingly, the proximal point method involves the computation of the so-called {\em resolvent} of a relation, defined as follows:
\begin{definition}
The resolvent of a relation F is given as the relation $R_F = (I + \lambda F)^{-1}$, where $\lambda > 0$.
\end{definition}
In the case where the relation $R = \partial f$ of some convex function $f$,  the resolvent can be shown to be the {\em proximal mapping} \cite{moreau:prox}, a crucially important abstraction of the concept of projection, a cornerstone of constrained optimization.
\begin{definition}
The proximal mapping of a vector $v$ with respect to a convex function $f$ is defined as the minimization problem:
\[ \mbox{prox}_f(v) = \mbox{argmin}_{x \in K} (f(x) + \| v - x\|^2_2 ) \]
\end{definition}
In the case where $f(x) = I_K(x)$, the indicator function for a convex set $K$, the proximal mapping reduces to the projection $\Pi_K$. While the proximal point algorithm is general, it is not very effective for problems in high-dimensional machine learning that involve minimizing a {\em sum} of two or more functions, one or more of which may not be differentiable.  A key extension of the proximal point algorithm is through a general decomposition principle called operator splitting, reviewed below.
%
%
\section{Decomposition through Operator Splitting}
Operator splitting \cite{operator-splitting,douglas-rachford} is a generic approach to decomposing complex optimization and variational inequality problems into simpler ones that involve computing the resolvents of individual relations, rather than sums or other compositions of relations. For example, given a monotone inclusion problem of the form:
\[ 0 \in A(x) + B(x) \]
for two relations $A$ and $B$, how can we find the solution $x^*$ without computing the resolvent $(I + \lambda (A + B))^{-1}$, which may be complicated, but rather only compute the resolvents of $A$ and $B$ individually?  There are several classes of operator splitting schemes. We will primarily focus on the {\em Douglas Rachford} algorithm \cite{douglas-rachford} specified in Figure~\ref{splitting}, because it leads to a widely used distributed optimization method called Alternating Direction Method of Multipliers (ADMM) \cite{boyd:admm}. The Douglas Rachford method is based on the ``damped iteration" given by:
\[ z_{k+1} = \frac{1}{2} (I + C_A C_B) (z_k) \]
where $C_A = 2 R_A + I$ and $C_B = 2 R_B + I$ are the ``reflection" or {\em Cayley} operators associated with the relations $A$ and $B$. Note that the Cayley operator is defined in terms of the resolvent, so this achieves the necessary decomposition.
\begin{figure} [th]
\centering
\fbox{
\begin{minipage}[t]{0.4\textwidth}
\begin{algorithm}[H]
\caption{Douglas Rachford method.}

{\bf INPUT:} Given $A(x), B(X)$ and a scalar $\lambda > 0$.

\begin{algorithmic}[1]

\STATE Set $k=0$ and initial vector $z_k = 0$.

\REPEAT

\STATE  Set $x_{k + \frac{1}{2}} \leftarrow R_B(z_k)$

\STATE Set $z_{k+\frac{1}{2}} \leftarrow 2 x_{k + \frac{1}{2}} - z_k$

\STATE Set $x_{k+1} \leftarrow R_A(z_{k + \frac{1}{2}})$

\STATE Set $z_{k+1} \leftarrow z_k + x_{k+1} - x_{k+ \frac{1}{2}}$

\UNTIL{$z_{k  + 1} < \epsilon$ }

\STATE Return $x_{k+1}$

\end{algorithmic}
\end{algorithm}
\end{minipage}
\begin{minipage}[t]{0.55\textwidth}
\begin{algorithm}[H]
\caption{Alternating Direction Method of Multipliers.}

{\bf INPUT:} Given sub-differentiable convex functions $f(x), g(x)$ and a scalar $\lambda > 0$.

\begin{algorithmic}[1]

\STATE Set $k=0$ and initial vector $z_k = 0$.

\REPEAT

\STATE  Set $x_{k + \frac{1}{2}} \leftarrow \mbox{argmin}_x (f(x) + \frac{1}{2 \lambda} \|x - z_k \|^2_2)$

\STATE Set $z_{k+\frac{1}{2}} \leftarrow 2 x_{k + \frac{1}{2}} - z_k$

\STATE Set $x_{k+1} \leftarrow \mbox{argmin}_x (g(x) + \frac{1}{2 \lambda} \|x - x_{k + \frac{1}{2}} \|^2_2)$

\STATE Set $z_{k+1} \leftarrow z_k + x_{k+1} - x_{k+ \frac{1}{2}}$

\UNTIL{$z_{k  + 1} < \epsilon$ }

\STATE Return $x_{k+1}$
\end{algorithmic}
\end{algorithm}
\end{minipage}}
\caption{Operator splitting is a generic framework for decomposing a composite objective function into simpler components.}
\label{splitting}
\end{figure}
When $A = \partial f$ and $B = \partial g$, two convex functions, the Douglas Rachford algorithm becomes the well-known Alternating Direction Method of Multipliers (ADMM) method, as described in Figure~\ref{splitting}, where the resolvent of $A$ and $B$ turn into proximal minimization steps. The ADMM algorithm has been extensively studied in optimization; a detailed review is available in the tutorial paper by Boyd and colleagues \cite{boyd:admm}, covering both its theoretical properties, operator splitting origins, and applications to high-dimensional data mining. ADMMs have also recently been studied for spectroscopic data, in particular {\em hyperspectral unmixing} \cite{admm:hyper}.

\subsection{Forward Backwards Splitting}

In this section we will give a brief overview of proximal splitting
algorithms \cite{PROXSPLITTING2011}. The two key ingredients of proximal
splitting are proximal operators and operator splitting. Proximal
methods \cite{rockafellar1976PROXIMALmonotone,parikh2013proximal},
which are widely used in machine learning, signal processing, and stochastic
optimization, provide a general framework for large-scale optimization.
The {\em proximal mapping} associated with a convex function $h$ is defined
as:
\begin{equation}
{\rm{prox}}_{h}(x)=\arg\mathop{\min}\limits _{u}(h(u)+\frac{1}{2}{\left\Vert {u-x}\right\Vert ^{2}})\label{eq:prox}
\end{equation}
Operator splitting is widely used to reduce the computational complexity
of many optimization problems, resulting in algorithms such as
sequential non-iterative approach (SNIA), Strang splitting, and sequential
iterative approach (SIA). Proximal splitting is a technique that combines
proximal operators and operator splitting, and deals with
problems where the proximal operator is difficult to compute at first,
yet is easier to compute after decomposition. The very basic scenario
is Forward-Backward Splitting (FOBOS) \cite{FOBOS:2009}
\begin{equation}
\mathop{\min}\limits _{\theta}\left({\Psi(\theta)=f(\theta)+h(\theta)}\right)
\label{ch1eq:fobos}
\end{equation}
where $f(x)$ is a convex, continuously differentiable function with
$L$-Lipschitz-continuous bounded gradients, i.e. $\forall x,y,||\nabla f(x)-\nabla f(y)||\le L||x-y||$,
and $h(\theta)$ is a convex (possibly not smooth) function. FOBOS
solves this problem via the following proximal gradient method
\begin{equation}
{\theta _{t + 1}} = {\rm{pro}}{{\rm{x}}_{{\alpha _t}h}}({\theta _t} - {\alpha _t}\nabla f({\theta _t}))
\label{eq:fobos-prox}
\end{equation}
An extension of FOBOS is when the objective function is separable,
i.e.,
\begin{equation}
\mathop{\min}\limits _{\theta}\sum\limits _{i=1}^{m}{{f_{i}}(\theta)}\label{eq:admm-probform}
\end{equation}
where computing ${\rm{pro}}{{\rm{x}}_{\sum\limits_{i = 1}^m {{f_i}} }}(\cdot)$
is difficult, yet for each $i$, ${\rm{pro}}{{\rm{x}}_{{f_i}}}(\cdot)$ is easy to
compute. To solve this problem, Douglas-Rachford splitting \cite{PROXSPLITTING2011}
and Alternating Direction of Multiple Multipliers (ADMM) can be used.
Recently, ADMM has  been used proposed for sparse RL \cite{ZHIWEI2014}.

\subsection{Nonlinear Primal Problem Formulation}

In this paper we will investigate a scenario of proximal splitting that is
different from the problem formulation in Section (\ref{eq:admm-probform}),
namely the nonlinear primal form
\begin{equation}
\mathop{\min}\limits _{\theta}\left(\Psi(\theta)={F(K(\theta))+h(\theta)}\right)\label{eq:primal}
\end{equation}
where $F(\cdot)$ is a lower-semicontinuous (l.s.c) nonlinear convex
function, $K$ is a linear operator, the induced
norm is $||K||$. In the following, we will denote $F(K(\theta))$
as $F\circ K(\theta)$. The proximal operator of this problem is
\begin{equation}
{\theta _{t + 1}} = \arg \mathop {\min }\limits_\theta  \{ \Psi (\theta ) + \frac{1}{{2{\alpha _t}}}||\theta  - {\theta _t}||_2^2\}  = {\rm{pro}}{{\rm{x}}_{{\alpha _t}}}_{\left( {F \circ K + h} \right)}({\theta _t})
\end{equation}
In many cases, although ${{\rm{pro}}{{\rm{x}}_{{\alpha _t}F}}}$ and ${{\rm{pro}}{{\rm{x}}_{{\alpha _t}K}}}$
are easy to compute, ${{\rm{pro}}{{\rm{x}}_{{\alpha _t}F\circ K}}}$ is often difficult
to compute. For the NEU case, we have
\begin{equation}
\begin{array}{l}
K(\theta)=\mathbb{E}[{\phi_{t}}{\delta_{t}}]={\Phi^{T}}\Xi(T{V_{\theta}}-{V_{\theta}})={\Phi^{T}}\Xi(R+\gamma{\Phi^{'}}\theta-\Phi\theta),{\rm {}}F(\cdot)=\frac{1}{2}||\cdot||_{2}^{2}\end{array}\label{eq:neu-fk}
\end{equation}
%
It is straightforward to verify that ${{\rm{pro}}{{\rm{x}}_{{\alpha _t}F}}},{{\rm{pro}}{{\rm{x}}_{{\alpha _t}K}}}$
are easy to compute, but ${{\rm{pro}}{{\rm{x}}_{{\alpha _t}F\circ K}}}$
is not easy to compute since it involves the biased sampling problem
as indicated in Equation (\ref{eq:neu-grad}). To solve this problem,
we transform the problems formulation to facilitate operator splitting,
i.e., which only uses ${{\rm{pro}}{{\rm{x}}_{{\alpha _t}F}}},{{\rm{pro}}{{\rm{x}}_{{\alpha _t}K},}}{{\rm{pro}}{{\rm{x}}_{{\alpha _t}h}}}$
and avoids computing ${{\rm{pro}}{{\rm{x}}_{{\alpha _t}F\circ K}}}$ directly. We
will use the primal-dual splitting framework to this end.

\subsection{Primal-Dual Splitting}
\label{pd-splitting}

The corresponding primal-dual formulation \cite{BOOK2011PROXSPLIT,PROXSPLITTING2011,POCK2011SADDLE}
of Section (\ref{eq:primal}) is
\begin{equation}
\mathop{\min}\limits _{\theta\in X}\mathop{\max}\limits _{y\in Y}\left({L(\theta,y)=\left\langle {K(\theta),y}\right\rangle -{F^{*}}(y)}+h(\theta)\right)
\label{eq:primal-dual}
\end{equation}
where $F^{*}(\cdot)$ is the Legendre transform of the convex nonlinear
function $F(\cdot)$, which is defined as $\ensuremath{{F^{*}}(y)={\sup_{x\in X}}(\langle x,y \rangle -F(x))}$.
The proximal splitting update per iteration is written as
\begin{equation}
\begin{array}{l}
{y_{t + 1}} = \arg \mathop {\min }\limits_{y \in Y} \left\langle { - {K_t}({\theta _t}),y} \right\rangle  + {F^*}(y) + \frac{1}{{2{\alpha _t}}}||y - {y_t}|{|^2}\\
{{\theta_{t+1}}=\arg\mathop{\min}\limits _{\theta\in X}\left\langle {{K_{t}}({\theta}),y_{t}}\right\rangle +h(\theta)+\frac{1}{{2{\alpha_{t}}}}||\theta-{\theta_{t}}|{|^{2}}}
\end{array}
\end{equation}
Thus we have the general update rule as
\begin{equation}
\begin{array}{*{20}{l}}
{{y_{t + 1}} = {y_t} + {\alpha _t}{K_t}({\theta _t}) - {\alpha _t}\nabla F_t^*(y)
{\;,\;}
{\theta _{t + 1}} = {\rm{pro}}{{\rm{x}}_{{\alpha _t}h}}({\theta _t} - {\alpha _t}\nabla {K_t}({\theta _t}){y_t})}
\end{array}
\end{equation}
However, in stochastic learning setting, we do not have knowledge of the exact ${K_t}({\theta _t})$, $\nabla F_t^*(y)$ and $\nabla {K_t}({\theta _t}){y_t}$, whereas a stochastic oracle $\mathcal{SO}$ is able to provide unbiased estimation of them.

\section{Natural Gradient Methods}
\label{sec:NG}

Consider the problem of minimizing a differentiable function $f:\mathbb R^n \to \mathbb R$. The standard gradient descent approach is to select an initial $x_0 \in \mathbb R^n$, compute the direction of steepest descent, $-\nabla f(x_0)$, and then move some amount in that direction (scaled by a stepsize parameter, $\alpha_0$). This process is then repeated indefinitely: $x_{k+1} = x_k - \alpha_k \nabla f(x_k)$, where $\{\alpha_k\}$ is a stepsize schedule and $k \in \{1,\hdots\}$. Gradient descent has been criticized for its low asymptotic rate of convergence. Natural gradients are a quasi-Newton approach to improving the convergence rate of gradient descent.

When computing the direction of steepest descent, gradient descent assumes that the vector $x_k$ resides in Euclidean space. However, in several settings it is more appropriate to assume that $x_k$ resides in a Riemannian space with metric tensor $G(x_k)$, which is an $n \times n$ positive definite matrix that may vary with $x_k$ \cite{Amari1998b}. In this case, the direction of steepest descent is called the \emph{natural gradient} and is given by $-G(x_k)^{-1}\nabla f(x_k)$ \cite{Amari1998}. In certain cases, (which include our policy search application), following the natural gradient is asymptotically Fisher-efficient \cite{Amari1998b}.

\section{Summary}
We provided a brief overview of some background material in reinforcement learning and optimization in this chapter. The subsequent chapters contain further elaboration of this material as it is required.
The overall goal of our work is to bring reinforcement learning into the main fabric of modern stochastic optimization theory. As we show in subsequent chapters, accomplishing this goal gives us access
to many advanced algorithms and analytical tools. It is worth noting that we make little use of classical stochastic approximation theory, which has traditionally been used to analyze reinforcement learning
methods (as discussed in detail in books such as \cite{ndp:book}). Classical stochastic approximation theory provides only asymptotic convergence bounds, for the most part. We are interested, however, in
getting tighter sample complexity bounds, which stochastic optimization provides.

\chapter{Sparse Temporal Difference Learning in Primal Dual Spaces}
\label{md-td-chapter}

In this chapter we explore a new framework for (on-policy convergent)
TD learning algorithm based on \textit{mirror descent} and related
algorithms.\footnote{This chapter is based on the paper ``Sparse Q-learning with Mirror Descent" published in UAI 2012.}
 Mirror descent can be viewed as an enhanced gradient method,
particularly suited to minimization of convex functions in high-dimensional
spaces. Unlike traditional gradient methods, mirror descent undertakes
gradient updates of weights in both the dual space and primal space,
which are linked together using a Legendre transform. Mirror descent
can be viewed as a proximal algorithm where the distance-generating
function used is a Bregman divergence. A new class of {\em proximal-gradient}
based temporal-difference (TD) methods are presented based on different
Bregman divergences, which are more powerful than regular TD learning.
Examples of Bregman divergences that are studied include $p$-norm
functions, and Mahalanobis distance based on the covariance of sample
gradients. A new family of sparse mirror-descent reinforcement learning
methods are proposed, which are able to find sparse fixed-point of
an $l_{1}$-regularized Bellman equation at significantly less computational
cost than previous methods based on second-order matrix methods.

\section{Problem Formulation}

The problem formulation in this chapter is based on the Lasso-TD objective
defined as follows, 
which is used in LARS-TD and LCP-TD.  We first define $l_{1}$-regularized Projection, and then give the definition of Lasso-TD objective function.

\begin{defn}
\cite{LASSOTD:2011} ($l_{1}$-regularized Projection): ${{\Pi}_{{l_{1}}}}$ is the $l_{1}$-regularized
projection defined as:
\[{\Pi_{{l_{1}}}}y=\Phi\theta, \theta=\arg{\min_{w}}{\left\Vert {y-\Phi w}\right\Vert ^{2}}+\rho{\left\Vert w\right\Vert _{1}}\]
which is a non-expansive mapping w.r.t weighted $l_{2}$ norm, as proven in \cite{LASSOTD:2011}.
\end{defn}

\begin{lemma}
\cite{LASSOTD:2011}: ${\Pi_{\rho}}$ is a non-expansive mapping such
that
\begin{equation}
\forall x,y\in{R^{d}},||{\Pi_{\rho}}x-{\Pi_{\rho}}y|{|^{2}}\le||x-y|{|^{2}}-||x-y-({\Pi_{\rho}}x-{\Pi_{\rho}}y)|{|^{2}}\label{eq:proj_rho}
\end{equation}
\end{lemma}

\begin{defn}
\cite{LASSOTD:2011} (Lasso-TD) Lasso-TD is a fixed-point
equation w.r.t $l_{1}$ regularization with
parameter $\rho$, which is defined as
\begin{equation}
\begin{array}{l}
\theta=f(\theta)={\rm {argmi}}{{\rm {n}}_{u\in{R^{d}}}}\left({||T\Phi\theta-\Phi u{||{}^{2}}+\rho||u{||_{1}}}\right)\\
{\rm {=argmi}}{{\rm {n}}_{u\in{R^{d}}}}\left({||{R^{\pi}}+\gamma{P^{\pi}}\Phi\theta-\Phi u{||{}^{2}}+\rho||u{||{}_{1}}}\right)
\end{array}\label{eq:lassotd}
\end{equation}
\end{defn}

The properties of Lasso-TD is discussed in detail in \cite{LASSOTD:2011}.
Note that the above $l_{1}$ regularized fixed-point is not a convex
optimization problem but a fixed-point problem. Several prevailing
sparse RL methods use Lasso-TD as the objective function, such as SparseTD\cite{mahadevan:uai2005},
LARS-TD\cite{Kolter09LARSTD} and LCP-TD\cite{JeffLcp:nips2010}.
The advantage of LARS-TD comes from LARS in that it computes a homotopy path of solutions with
different regularization parameters, and thus offers a rich solution
family. The major drawback comes from LARS, too. To maintain the LARS
criteria wherein each active variable has the same correlation with
the residual, variables may be added and dropped several times, which
is computationally expensive. In fact, the computational complexity
per iteration is $O(Ndk^{2})$ where $k$ is the cardinality of the active
feature set. Secondly, LARS-TD requires the $A$ matrix to be a $P$-matrix(a
square matrix which does not necessarily to be symmetric, but all
the principal minors are positive), which poses extra limitation on
applications. The author of LARS-TD claims that this seems never to
be a problem in practice, and given on-policy sampling condition or
given large enough ridge regression term, $P$-matrix condition can
be guaranteed. LCP-TD {[}12{]} formulates LASSO-TD as a linear complementarity
problem (LCP), which can be solved by a variety of available LCP solvers.

We then derive the major step by formulating the
problem as a forward-backward splitting problem (FOBOS) as in \cite{FOBOS:2009},
\begin{equation}
\begin{array}{l}
{\theta{}_{t+\frac{1}{2}}}={\theta_{t}}-{\alpha_{t}}{g_{t}}\\
{\theta_{t+1}}=\arg\mathop{\min}\limits _{\theta}\left\{ {\frac{1}{2}||{\theta}-{\theta_{t+\frac{1}{2}}}||_{2}^{2}+\alpha_{t}h(\theta)}\right\}
\end{array}\label{eq:fobos}
\end{equation}

This is equivalent to the formulation of proximal gradient method
\begin{equation}
{\theta_{t+1}}=\arg\mathop{\min}\limits _{\theta}\left\{ {\left\langle {{g_{t}},\theta}\right\rangle +h(\theta)+\frac{1}{2\alpha_{t}}||\theta-{\theta_{t}}||_{2}^{2}}\right\} \label{eq:fobos2}
\end{equation}

Likewise, we could formulate the sparse TD algorithm as
\begin{equation}
\begin{array}{l}
{\theta_{t+\frac{1}{2}}}={\theta_{t}}-\frac{{\alpha_{t}}}{2}\nabla{\rm {MSE(}}\theta{\rm {)}}\\
{\theta_{t+1}}=\arg\mathop{\min}\limits _{\theta}\left\{ {\frac{1}{2}||{\theta_{t}}-{\theta_{t+\frac{1}{2}}}||_{2}^{2}+\alpha_{t}h(\theta)}\right\}
\end{array}
\end{equation}

And this can be formulated as
\begin{equation}
{\theta_{t+1}}=\arg\mathop{\min}\limits _{\theta}\left\{ {\left\langle {\frac{1}{2}\nabla{\rm {MSE(}}\theta{\rm {)}},\theta}\right\rangle +h(\theta)+\frac{1}{2\alpha_{t}}||\theta-{\theta_{t}}||_{2}^{2}}\right\}
\end{equation}

\section{Mirror Descent RL}
%

\begin{algorithm}
\label{mdtd-algm1}
\caption{Adaptive Mirror Descent TD($\lambda$)}

Let $\pi$  be some fixed policy for an MDP M, and $s_0$ be the initial state. Let $\Phi$ be some fixed or automatically generated basis.

%
%
%
%
%
%
%
%
%
\begin{algorithmic}[1]
\REPEAT

\STATE  Do action $\pi(s_t)$ and observe next state $s_{t+1}$ and reward $r_t$.

\STATE Update the eligibility trace $e_t \leftarrow e_t + \lambda \gamma \phi(s_t)$

\STATE Update the dual weights $\theta_t$ for a linear function approximator:
\[ \theta_{t+1} =  \nabla \psi_t(w_t) + \alpha_t (r_t + \gamma \phi(s_{t+1})^T w_t - \phi(s_t)^T w_t) e_t  \]
where $\psi$ is a distance generating function.

\STATE Set $w_{t+1} = \nabla \psi_t^*(\theta_{t+1})$ where  $\psi^*$ is the Legendre transform of $\psi$.

\STATE Set $t \leftarrow t+1$.

\UNTIL{{\bf done}}.

Return $\hat{V}^\pi \approx \Phi w_t $ as the value function associated with policy $\pi$ for MDP $M$.

\end{algorithmic}

\end{algorithm}

 Algorithm 1 describes the proposed mirror-descent TD($\lambda$) method.\footnote{All the algorithms described extend to the action-value case where $\phi(s)$ is replaced by $\phi(s,a)$. }  Unlike regular TD, the weights are updated using the TD error  in the dual space by mapping the primal weights $w$ using a gradient of a strongly convex function $\psi$. Subsequently, the updated dual weights are converted back into the primal space using the gradient of the Legendre transform of $\psi$, namely $\nabla \psi^*$. Algorithm 1 specifies the mirror descent  TD($\lambda$) algorithm wherein each weight $w_i$ is associated with an eligibility trace $e(i)$. For $\lambda = 0$, this is just the features of the current state $\phi(s_t)$, but for nonzero $\lambda$, this corresponds to a decayed set of features proportional to the recency of state visitations. Note that the distance generating function $\psi_t$ is a function of time.

\subsection{Choice of Bregman Divergence}\label{bd}
We now discuss various choices for the distance generating function in Algorithm 1. In the simplest case, suppose $\psi(w)  =  \frac{1}{2} \| w \|^2_2$, the Euclidean length of $w$.  In this case, it is easy to see that mirror descent TD($\lambda$) corresponds to regular TD($\lambda$), since the gradients $\nabla \psi$ and $\nabla \psi^*$ correspond to the identity function.
A much more interesting choice of $\psi$ is $\psi(w) = \frac{1}{2} \| w \|^2_q$,  and its conjugate Legendre transform  $\psi^*(w) = \frac{1}{2} \| w \|^2_p$. Here, $\|w\|_q = \left( \sum_j |w_j|^q \right)^{\frac{1}{q}}$, and $p$ and $q$ are conjugate numbers such that $\frac{1}{p} + \frac{1}{q} = 1$.
This $\psi(w)$ leads to the p-norm link function $\theta = f(w)$ where $f: \mathbb{R}^d \rightarrow \mathbb{R}^d$  \cite{gentile:mlj}:
\begin{equation}
\label{pnorm}
f_j(w) = \frac{\mbox{sign}(w_j) |w_j|^{q-1}}{\| w \|^{q-2}_q}, \ \ f^{-1}_j(\theta) = \frac{\mbox{sign}(\theta_j) |\theta_j|^{p-1}}{\| \theta \|^{p-2}_p}
\end{equation}
The p-norm function has been extensively studied in the literature on online learning \cite{gentile:mlj}, and it is well-known that for large $p$, the corresponding classification or regression method behaves like a multiplicative method (e.g., the p-norm regression method for large $p$ behaves like an exponentiated gradient method (EG) \cite{kivinen:95,littlestone:mlj}).

Another distance generating function is the negative entropy function $\psi(w) = \sum_i w_i \log w_i$, which leads to the entropic mirror descent algorithm \cite{Beck:2003p2359}. Interestingly, this special case has been previously explored  \cite{sutton-precup:egtd} as the exponentiated-gradient TD method, although the connection to mirror descent and Bregman divergences were not made in this previous study, and EG does not generate sparse solutions \cite{ShalevShwartz:jmlr}. We discuss EG methods vs. p-norm methods in Section~\ref{eg-vs-pnorm}.

\subsection{Sparse Learning with Mirror Descent TD}\label{l1mdtd}
%

\begin{algorithm}
\caption{Sparse Mirror Descent TD($\lambda$)}


\begin{algorithmic}[1]
\REPEAT

\STATE  Do action $\pi(s_t)$ and observe next state $s_{t+1}$ and reward $r_t$.

\STATE Update the eligibility trace $e_t \leftarrow e_t + \lambda \gamma \phi(s_t)$

\STATE Update the dual weights $\theta_t$:
\[ \tilde{\theta}_{t+1} = \nabla \psi_t(w_t) + \alpha_t \left( r_t + \gamma \phi(s_{t+1})^T w_t - \phi(s_t)^T w_t \right) e_t  \]
(e.g.,  $\psi(w) = \frac{1}{2} \| w \|^2_q$ is the p-norm link function).

\STATE Truncate weights:
\[ \forall j, \ \ \theta^{t+1}_j = \mbox{sign}(\tilde{\theta}^{t+1}_j) \max(0, |\tilde{\theta}^{t+1}_j| - \alpha_t \beta) \]

\STATE $w_{t+1} = \nabla \psi_t^*(\theta_{t+1})$ (e.g.,  $\psi^*(\theta) = \frac{1}{2} \|\theta\|^2_p$  and $p$ and $q$ are dual norms such that $\frac{1}{p} + \frac{1}{q} = 1$).

\STATE Set $t \leftarrow t+1$.

\UNTIL{{\bf done}}.

Return $\hat{V}^\pi \approx \Phi w_t $ as the $l_1$ penalized sparse value function associated with policy $\pi$ for MDP $M$.

\end{algorithmic}
\end{algorithm}

Algorithm 2 describes a modification to obtain sparse value functions resulting in a sparse mirror-descent TD($\lambda)$ algorithm. The main difference is that the dual weights $\theta$ are truncated according to Equation~\ref{l1prox} to satisfy the $l_1$ penalty on the weights. Here, $\beta$ is a sparsity  parameter. An analogous approach was suggested in \cite{ShalevShwartz:jmlr} for $l_1$ penalized classification and regression.

\subsection{Composite Mirror Descent TD}\label{comdtd}
%
%
Another possible mirror-descent TD algorithm uses as the distance-generating function a Mahalanobis distance derived from the subgradients generated during actual trials. We base our derivation on the composite mirror-descent approach proposed in \cite{duchi:jmlr} for classification and regression. The composite mirror-descent solves the following optimization problem at each step:
\begin{equation}
\label{adaptsubgmd}
w_{t+1} = \mbox{argmin}_{x \in X} \left( \alpha_t \langle x, \partial f_t \rangle + \alpha_t \mu(x)  + D_{\psi_t}(x, w_t) \right)
\end{equation}
Here, $\mu$ serves as a fixed regularization function, such as the $l_1$ penalty, and $\psi_t$ is the time-dependent distance generating function as in mirror descent. We now describe a different Bregman divergence to be used as the distance generating function in this method.  Given a positive definite matrix $A$, the Mahalanobis norm of a vector $x$ is defined as $\|x\|_A = \sqrt{\langle x, A x \rangle}$. Let $g_t = \partial f(s_t)$ be the subgradient of the function being minimized at time $t$, and $G_t = \sum_t g_t g_t^T$ be the covariance matrix of outer products of the subgradients. It is computationally more efficient to use the diagonal matrix $H_t = \sqrt{\mbox{diag}(G_t)}$ instead of the full covariance matrix, which can be expensive to estimate. Algorithm 3 describes the adaptive subgradient mirror descent TD method.

\begin{algorithm}
\caption{Composite Mirror Descent TD($\lambda$)}


\begin{algorithmic}[1]
\REPEAT

\STATE  Do action $\pi(s_t)$ and observe next state $s_{t+1}$ and reward $r_t$.

\STATE Set TD error $\delta_t = r_t + \gamma \phi(s_{t+1})^T w_t - \phi(s_t)^T w_t$

\STATE Update the eligibility trace $e_t \leftarrow e_t + \lambda \gamma \phi(s_t)$

\STATE Compute TD update $\xi_t = \delta_t e_t$.

\STATE Update feature covariance
\[G_{t} = G_{t-1} + \phi(s_t) \phi(s_t)^T\]
\STATE Compute Mahalanobis matrix  $H_{t} = \sqrt{\mbox{diag}(G_{t})}$.

\STATE Update the  weights $w$:
\[ w_{{t+1},i} = \mbox{sign}( w_{t,i} - \frac{\alpha_t \xi_{t,i}}{H_{{t},ii}})(|w_{t,i} - \frac{\alpha_t \xi_{t,i} }{H_{{t},ii}}| - \frac{ \alpha_t \beta}{H_{{t},ii}} )\]
\STATE Set $t \leftarrow t+1$.

\UNTIL{{\bf done}}.

Return $\hat{V}^\pi \approx \Phi w_t $ as the $l_1$ penalized sparse value function associated with policy $\pi$ for MDP $M$.

\end{algorithmic}
\end{algorithm}
\section{Convergence Analysis}

\textbf{Definition 2} \cite{LASSOTD:2011}: ${{\Pi}_{{l_{1}}}}$
is the $l_{1}$-regularized projection defined as:  ${\Pi_{{l_{1}}}}y=\Phi\alpha$ such that $\alpha=\arg{\min_{w}}{\left\Vert {y-\Phi w}\right\Vert ^{2}}+\beta{\left\Vert w\right\Vert _{1}}$, which is a non-expansive mapping w.r.t weighted $l_2$ norm induced by the on-policy sample distribution setting, as proven in \cite{LASSOTD:2011}. Let the approximation error $f(y,\beta)={\left\Vert {y-{\Pi_{{l_{1}}}}y}\right\Vert ^{2}}$.

\noindent \textbf{Definition 3} (Empirical $l_{1}$-regularized projection): ${{\hat{\Pi}}_{{l_{1}}}}$
is the empirical $l_{1}$-regularized projection with a specific $l_{1}$
regularization solver, and satisfies the non-expansive mapping property.
It can be shown using a direct derivation that ${{{\hat \Pi }_{{l_1}}}\Pi T}$ is a $\gamma$-contraction mapping. Any unbiased $l_1$ solver which generates intermediate sparse solution before convergence, e.g., SMIDAS solver after $t$-th iteration, comprises an  empirical $l_{1}$-regularized projection.

\noindent \textbf{Theorem 1} The approximation error $||V - \hat V|| $ of Algorithm 2
is bounded by (ignoring  dependence on $\pi$ for simplicity):
\begin{equation}
\begin{array}{*{20}{l}}
{|| {V - \hat V} || \le \frac{1}{{1 - \gamma }} \times }\\
{\left( {\left\| {V - \Pi V} \right\| + f(\Pi V,\beta ) + (M - 1)P(0) + \left\| {{w^*}} \right\|_1^2\frac{M}{{\alpha_t N}}} \right)}
\end{array}
\label{eq:theorem}
\end{equation}
where $\hat{V}$ is the approximated value function after $N$-th iteration, i.e., $\hat{V}=\Phi{w_{N}}$,
 $M=\frac{2}{{2-4\alpha_t(p-1)e}}$, $\alpha_t$ is the stepsize, $P(0) = \frac{1}{N}\sum\limits_{i = 1}^N {\left\| {\Pi V({s_i})} \right\|_2^2} $, $s_i$ is the state of $i$-th sample,  $e = {d^{\frac{p}{2}}}$, $d$ is the number of features, and finally, $w^*$ is $l_1$-regularized projection of $\Pi V$ such that $\Phi {w^*} = {\Pi _{{l_1}}}\Pi V$.

\noindent \textbf{Proof:} In the on-policy setting, the solution given by Algorithm 2 is the fixed point
of $\hat V = {\hat \Pi _{{l_1}}}\Pi T\hat V$
 and the error decomposition is illustrated in Figure \ref{Fio:Error2}.
\begin{figure}
\centering
\includegraphics[width=2in,height=1.5in]{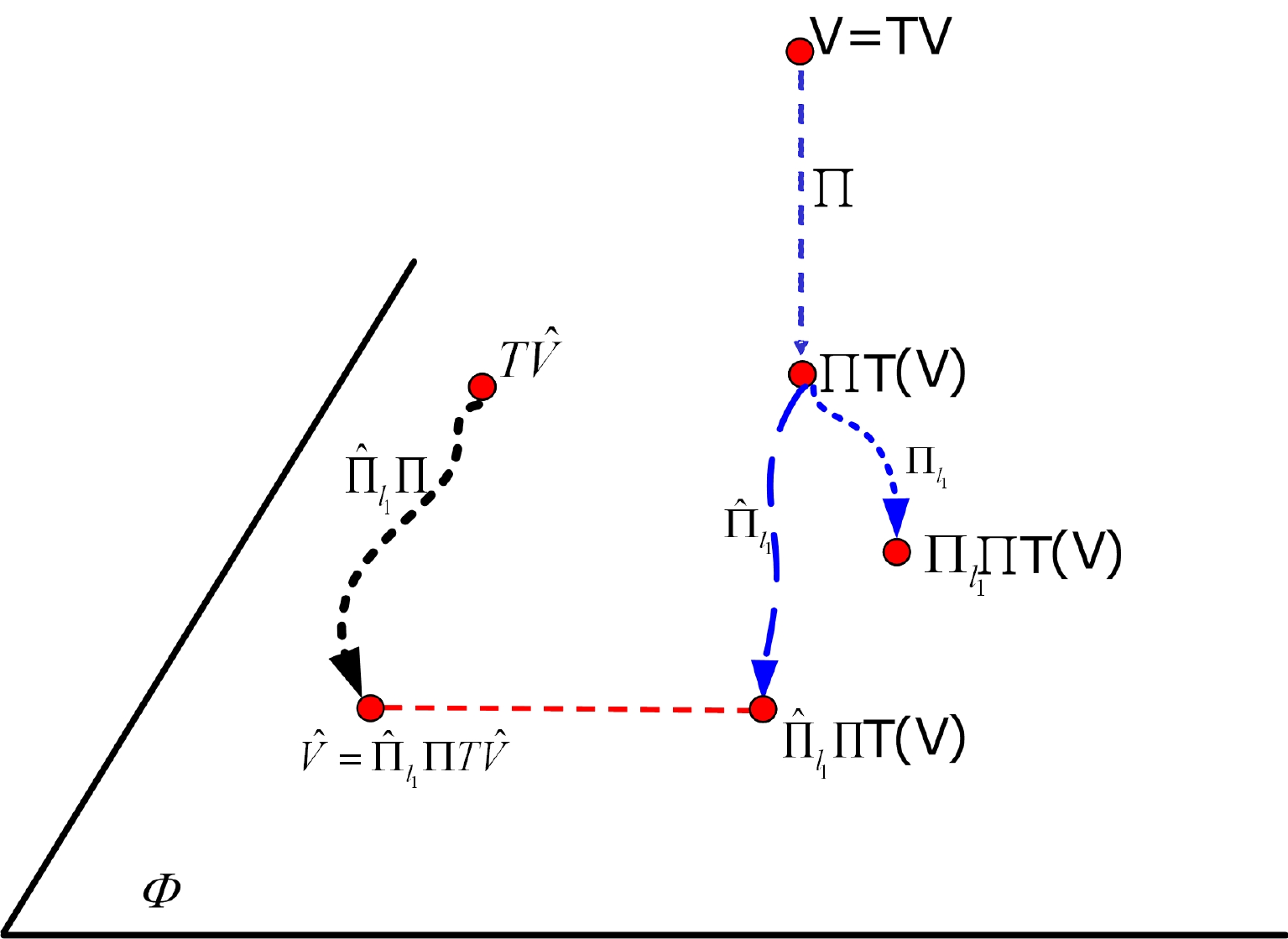}
\caption{Error Bound and Decomposition}
\label{Fio:Error2}
\end{figure}
The error can be bounded by the triangle inequality
\begin{small}
\begin{equation}
|| {V - \hat V} || = || {V - \Pi TV} || + ||\Pi TV - {\hat \Pi _{{l_1}}}\Pi TV|| +  ||{\hat \Pi _{{l_1}}}\Pi TV - \hat V||
 \label{eq:triangular}
 \end{equation}
 \end{small}
Since ${{{\hat \Pi }_{{l_1}}}\Pi T}$ is a $\gamma$-contraction mapping,
and $\hat V = {\hat \Pi _{{l_1}}}\Pi T\hat V$, we have
\begin{equation}
||{\hat \Pi _{{l_1}}}\Pi TV - \hat V|| = ||{\hat \Pi _{{l_1}}}\Pi TV - {\hat \Pi _{{l_1}}}\Pi T\hat V|| \le \gamma ||V - \hat V||
\label{eq:gamma}\end{equation}
So we have
\[(1 - \gamma )|| {V - \hat V} || \le || {V - \Pi TV} || + ||\Pi TV - {\hat \Pi _{{l_1}}}\Pi TV||\]
$\left\Vert {V-\Pi TV}\right\Vert $ depends on the expressiveness
of the basis $\Phi$, where if $V$ lies in $span(\Phi)$, this error term
is zero. $||\Pi TV - {\Pi _{{l_1}}}\hat \Pi TV||$
is further bounded by the triangle inequality
\[\begin{array}{*{20}{l}}
{||\Pi TV - {{\hat \Pi }_{{l_1}}}\Pi TV|| \le }\\
{||\Pi TV - {\Pi _{{l_1}}}\Pi TV|| + ||{\Pi _{{l_1}}}\Pi TV - {{\hat \Pi }_{{l_1}}}\Pi TV||}
\end{array}\]
where $\left\Vert {\Pi TV-{\Pi_{{l_{1}}}}\Pi TV}\right\Vert $ is controlled
by the sparsity  parameter $\beta$, i.e.,
$f(\Pi TV,\beta ) = || {\Pi TV - {\Pi _{{l_1}}}\Pi TV} ||$,
where \noindent $\varepsilon  = || {{{\hat \Pi }_{{l_1}}}\Pi TV - {\Pi _{{l_1}}}\Pi TV} ||$
is the approximation error depending on the quality of the $l_{1}$ solver employed. In Algorithm 2, the $l_{1}$
solver is related to the SMIDAS $l_1$ regularized mirror-descent method for regression and classification \cite{ShalevShwartz:jmlr}. Note that for a squared loss function
$L(\left\langle {w,{x_{i}}}\right\rangle ,{y_{i}})=||\left\langle {w,{x_{i}}}\right\rangle -{y_{i}}||_{2}^{2}$, we have ${\left|{L'}\right|^{2}}\le 4L$.
Employing the result of Theorem 3 in \cite{ShalevShwartz:jmlr}, after the $N$-th iteration, the $l_{1}$ approximation error is bounded by
\[\varepsilon  \le (M - 1)P(0) + || {{w^*}} ||_1^2\frac{M}{{\alpha_t N}},M = \frac{2}{{2 - 4\alpha_t (p - 1)e}}\]
By rearranging the terms and applying $V=TV$, Equation (\ref{eq:theorem}) can be deduced.

\section{Experimental Results: Discrete  MDPs}
Figure~\ref{ogrid-results} shows that mirror-descent TD converges more quickly with far smaller Bellman errors than LARS-TD \cite{lars-td} on a discrete ``two-room" MDP \cite{pvf-jmlr1}. The basis matrix $\Phi$ was automatically generated as $50$ proto-value functions  by diagonalizing the graph Laplacian of the discrete state space connectivity graph\cite{pvf-jmlr1}. The figure also shows that Algorithm 2 (sparse mirror-descent TD) scales more gracefully than LARS-TD. Note LARS-TD is unstable for $\gamma = 0.9$.
It should be noted that the computation cost of LARS-TD is $O(Ndm^3)$, whereas that for Algorithm 2 is $O(Nd)$, where $N$ is the number of samples, $d$ is the number of basis functions, and $m$ is the number of active basis functions. If $p$ is linear or sublinear w.r.t $d$, Algorithm 2 has a significant advantage over LARS-TD.
\begin{figure}[ht]
\begin{center}
\begin{minipage}[t]{0.4\textwidth}
\includegraphics[height=0.15\textheight,width=2.8in]{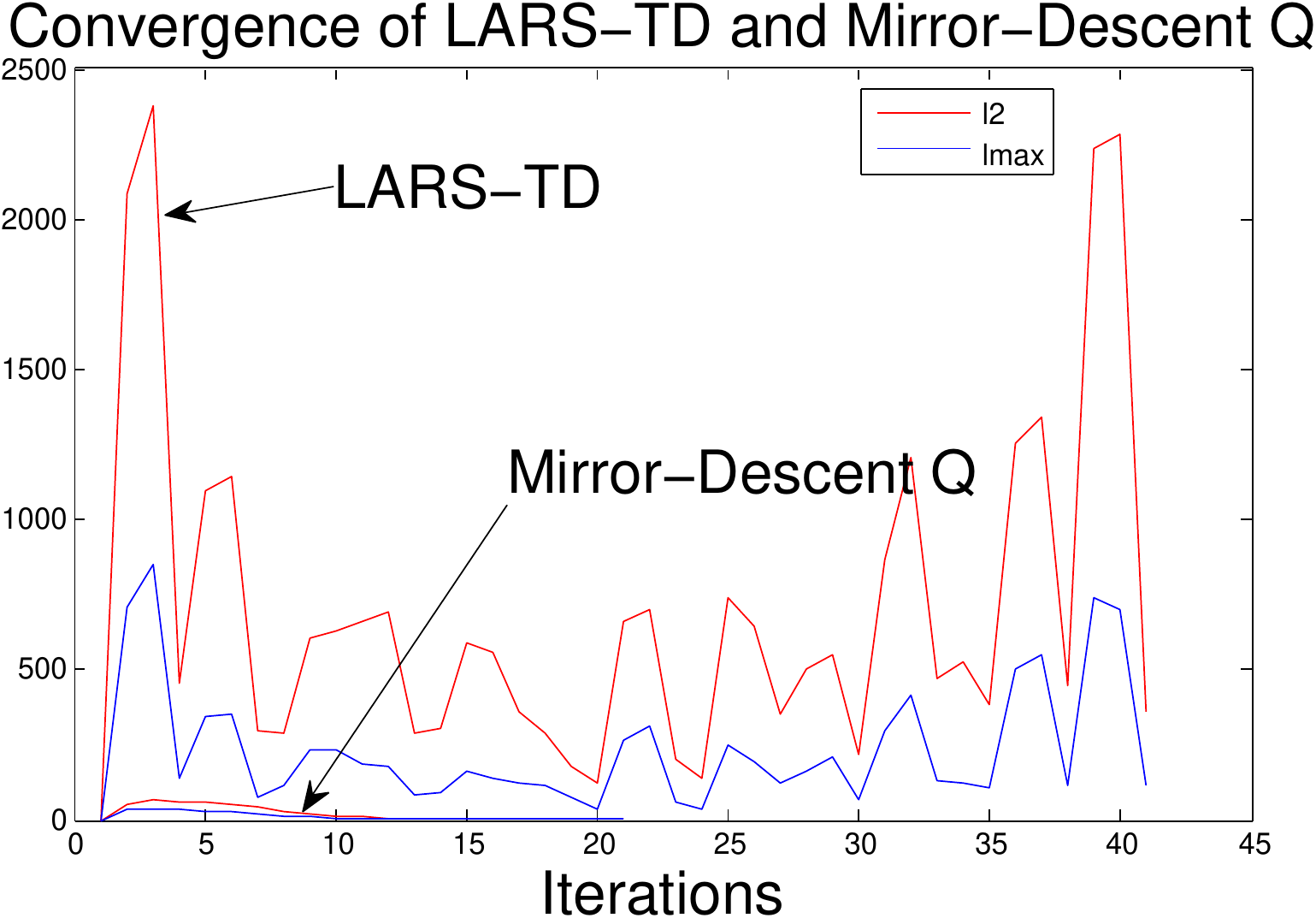}\\
\includegraphics[height=0.15\textheight,width=2.8in]{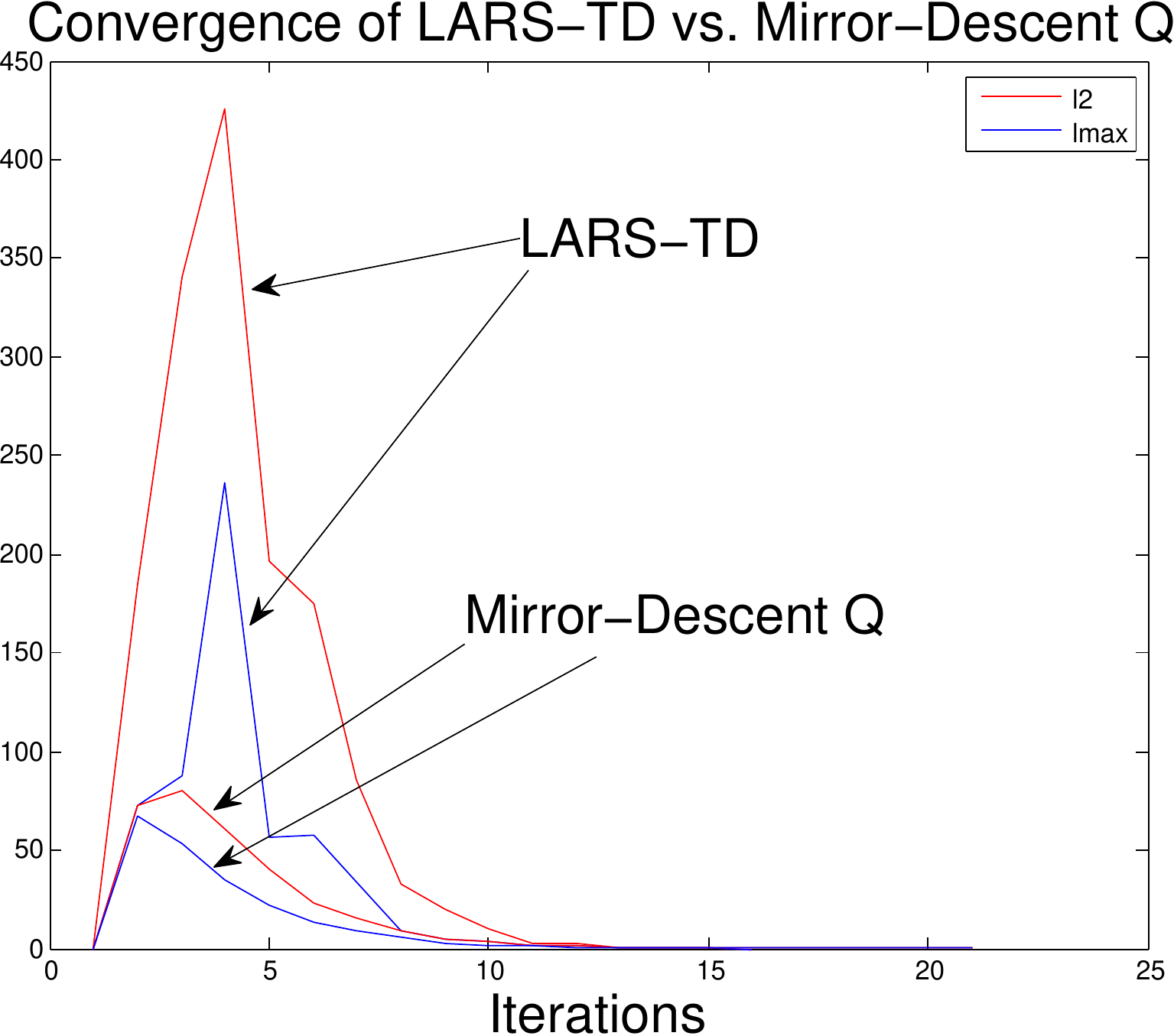}\\
\includegraphics[height=0.15\textheight,width=2.8in]{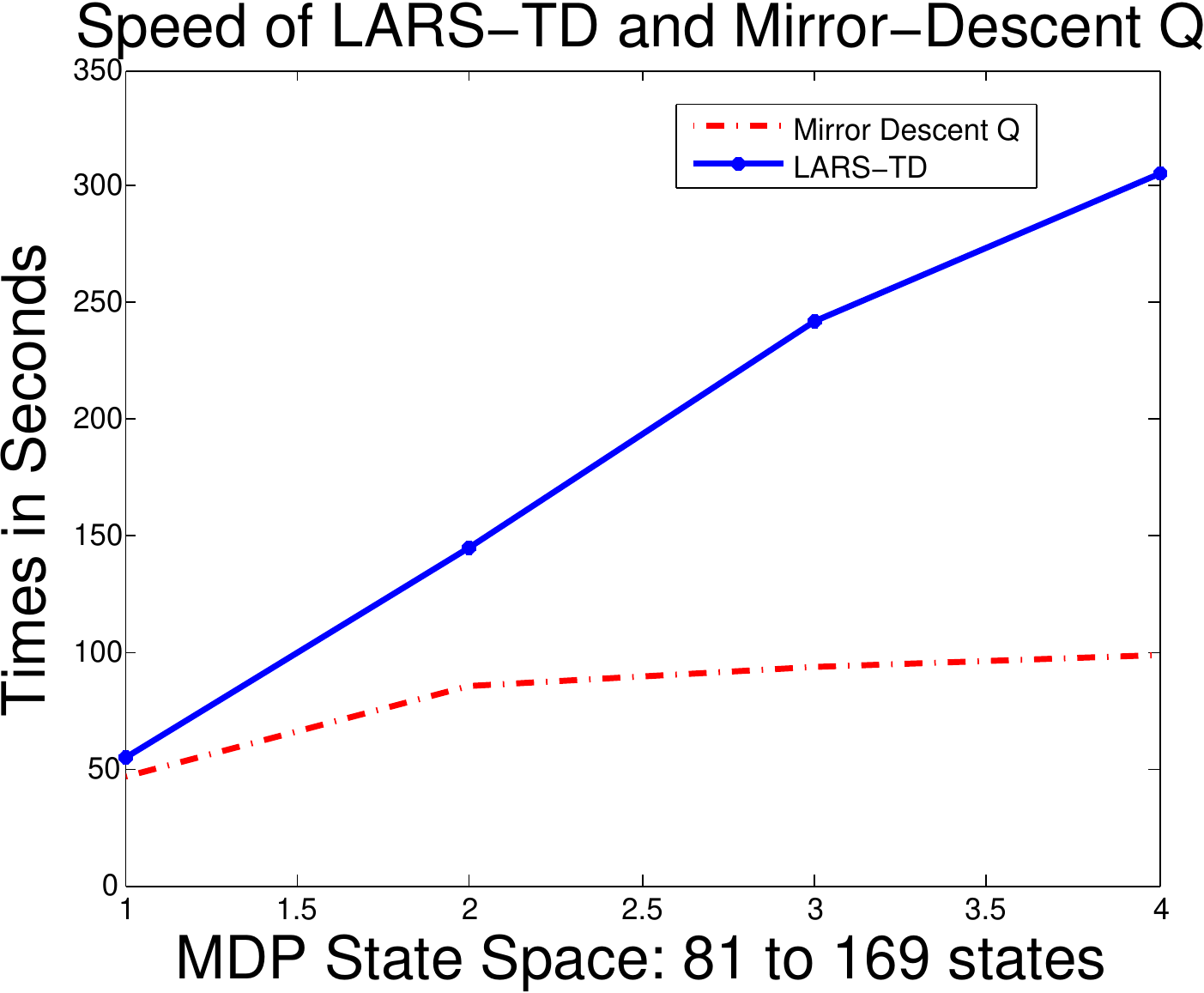}
\end{minipage}
\end{center}
\caption{Mirror-descent Q-learning converges significantly faster than LARS-TD on a ``two-room" grid world MDP for $\gamma = 0.9$ (top left) and $\gamma = 0.8$ (top right).  The y-axis measures the $l_2$ (red curve) and $l_\infty$ (blue curve) norm difference between successive weights during policy iteration. Bottom: running times for LARS-TD (blue solid) and mirror-descent Q (red dashed).  Regularization  $\beta  = 0.01$. }
\label{ogrid-results}
\end{figure}

\begin{figure}[ht]
\centering
\begin{minipage}[t]{0.6\textwidth}
 \includegraphics[width=3in]{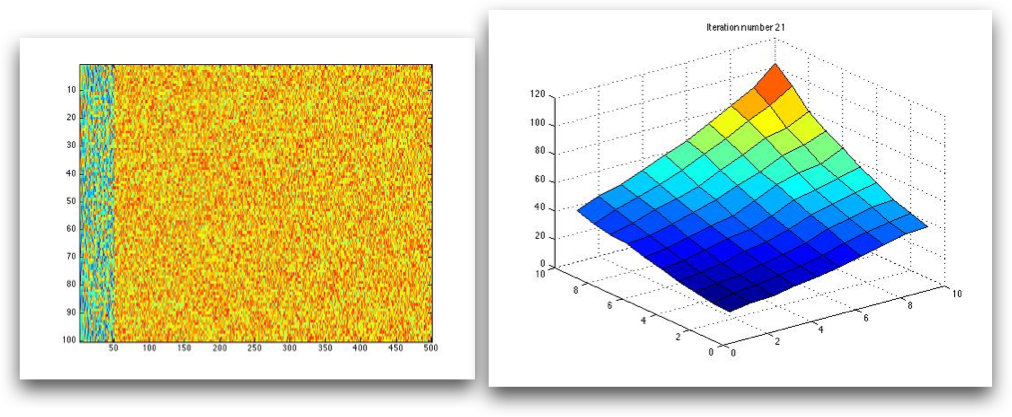}
\end{minipage}  \hspace{0.2in}
\caption{Sensitivity of sparse mirror-descent TD to noisy features in a grid-world domain. Left: basis matrix with the first 50 columns representing proto-value function bases and the remainder 450 bases representing mean-0 Gaussian noise. Right: Approximated value function using sparse mirror-descent TD.  }
\label{grid-noisypvfs}
\end{figure}

Figure~\ref{grid-noisypvfs} shows the result of another experiment conducted to test the noise immunity of Algorithm 2 using a discrete $10 \times 10$ grid world domain with the goal set at the upper left hand corner. For this problem, $50$ proto-value basis functions were automatically generated, and $450$ random Gaussian mean $0$ noise features were added. The sparse mirror descent TD algorithm was able to generate a very good approximation to the optimal value function despite the large number of irrelevant noisy features, and took a fraction of the time required by LARS-TD.

Figure~\ref{pnorm-decay} compares the performance of mirror-descent Q-learning with a fixed p-norm link function vs. a decaying p-norm link function for a $10 \times 10$ discrete grid world domain with the goal state in the upper left-hand corner. Initially, $p = O(\log d)$ where $d$ is the number of features, and subsequently $p$ is decayed to a minimum of $p=2$. Varying $p$-norm interpolates between additive and multiplicative updates. Different values of $p$ yield an interpolation between the truncated gradient method \cite{Lanford:2008} and SMIDAS \cite{SMIDAS:2009}. 

\begin{figure}[ht]
\begin{center}
\begin{minipage}[t]{0.4\textwidth}
\includegraphics[totalheight=0.13\textheight,width=2.8in]{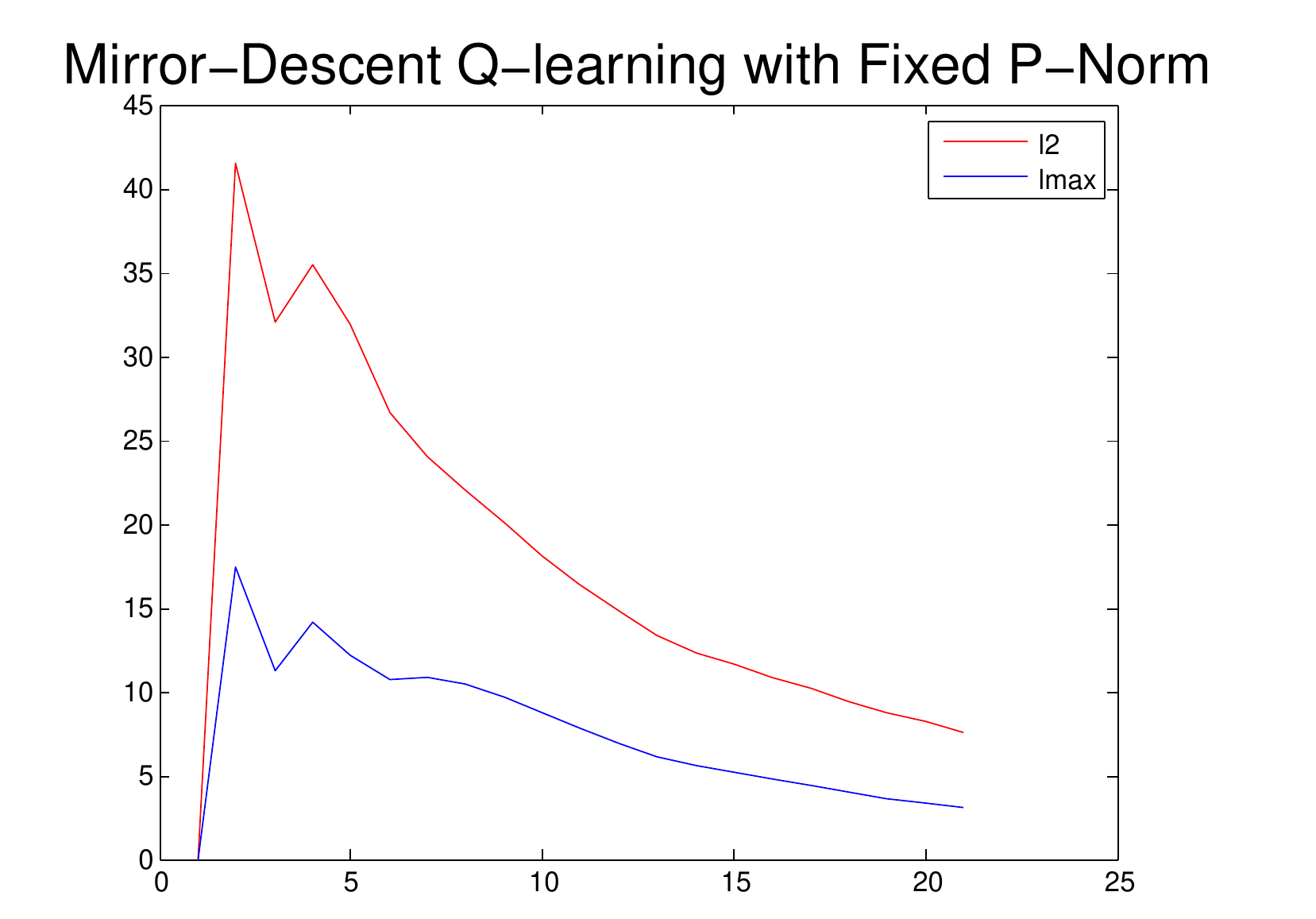}\\
\includegraphics[totalheight=0.13\textheight,width=2.8in]{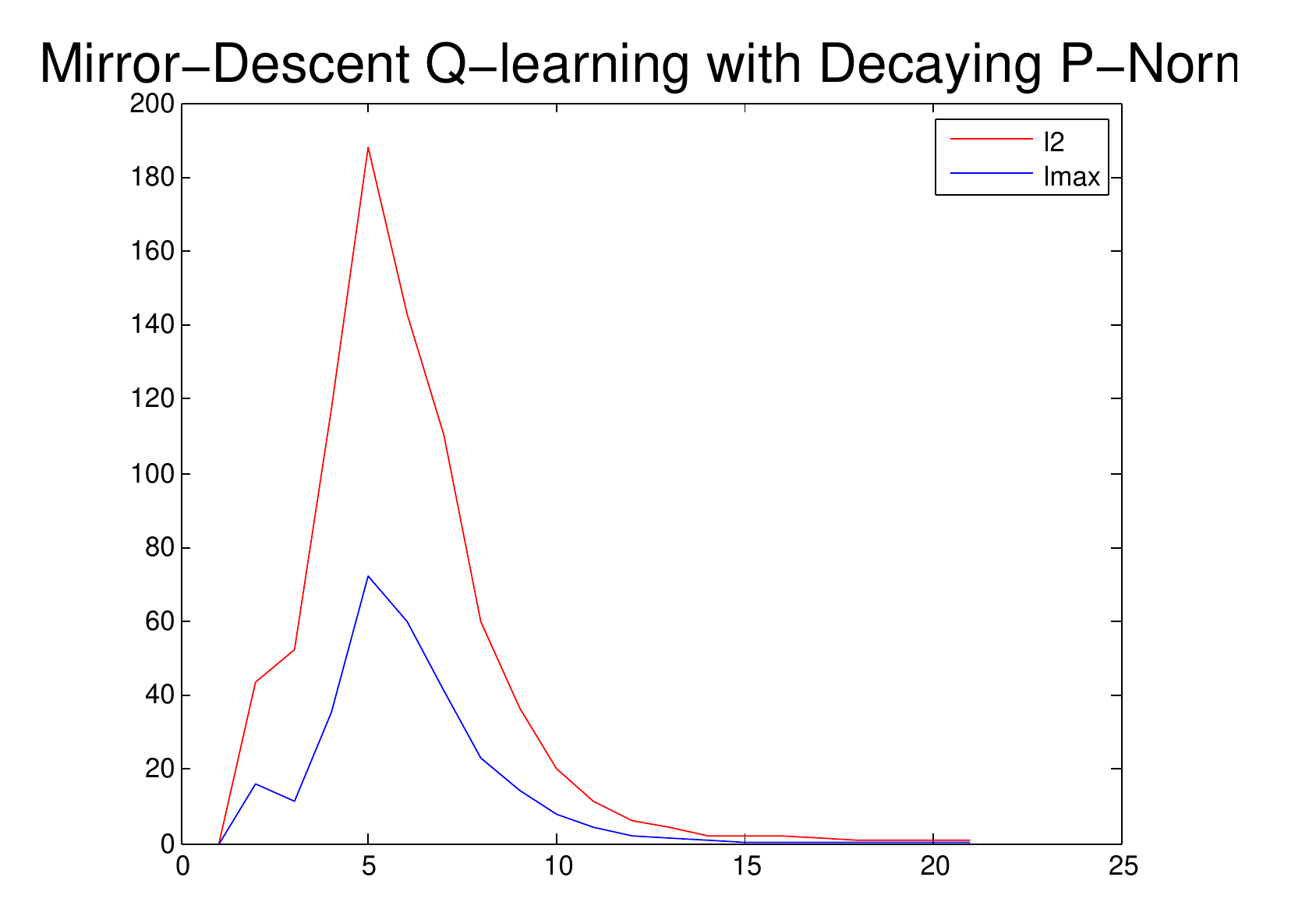}
\end{minipage}
\end{center}
\caption{Left: convergence of mirror-descent Q-learning with a fixed p-norm link function. Right: decaying p-norm link function. }
\label{pnorm-decay}
\end{figure}
\begin{figure}[t]
\begin{minipage}[t]{0.2\textwidth}
\includegraphics[totalheight=0.13\textheight, width=2in]{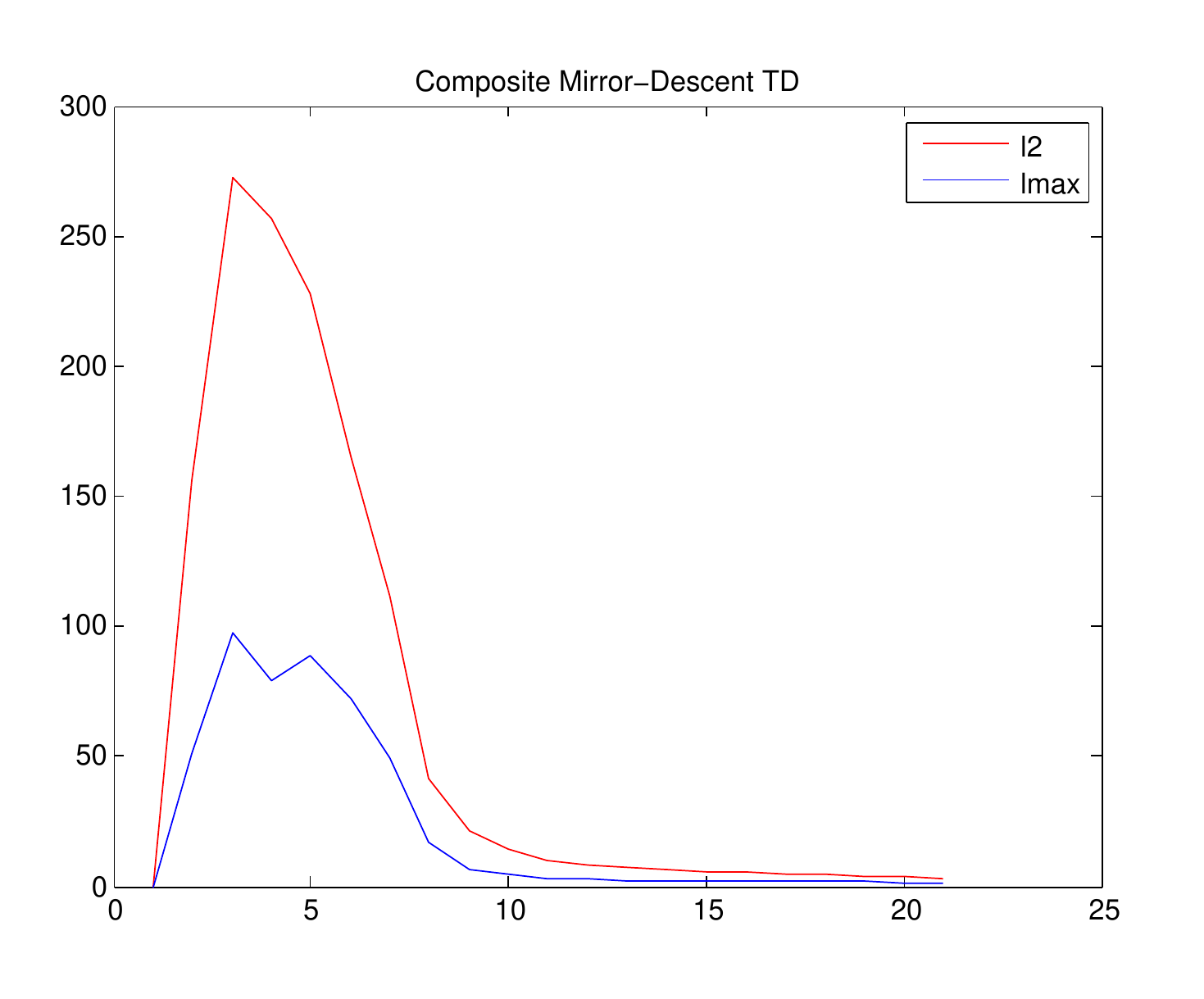}
\end{minipage}
\begin{minipage}[t]{0.2\textwidth}
 \includegraphics[totalheight=0.13\textheight, width=2in]{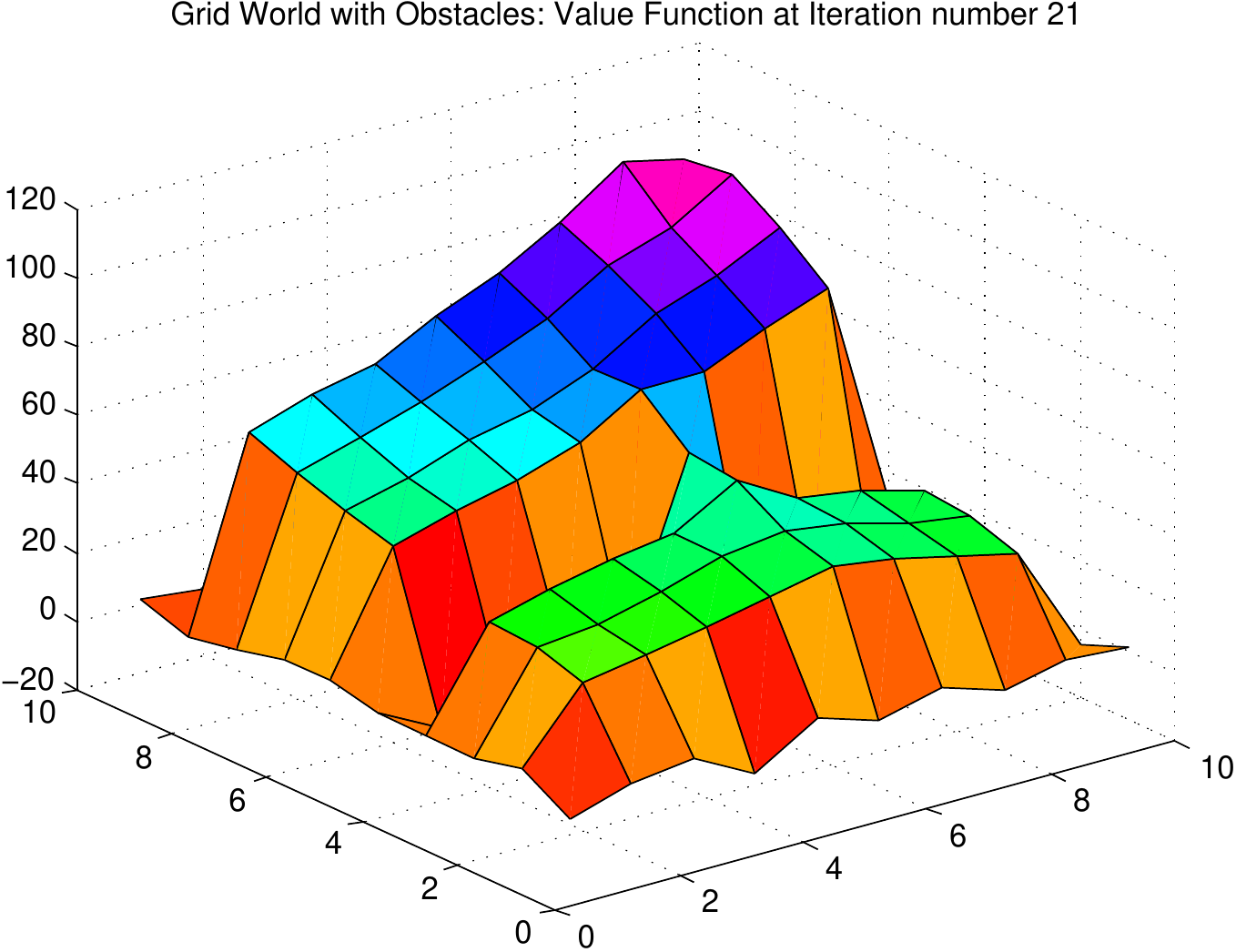}
\end{minipage}
\caption{Left: Convergence of composite mirror-descent Q-learning on  two-room gridworld domain. Right: Approximated value function, using $50$ proto-value function bases.}
\label{comd-ogrid-results}
\end{figure}

Figure~\ref{comd-ogrid-results} illustrates the performance of Algorithm 3 on the two-room discrete grid world navigation task.

\section{Experimental Results: Continuous  MDPs}\label{contmdp}
Figure~\ref{mcar-experiment} compares the performance of Q-learning vs.  mirror-descent Q-learning for the mountain car task, which converges more quickly to a better solution with much lower variance. Figure~\ref{acrobot-experiment}  shows that mirror-descent Q-learning with learned diffusion wavelet bases converges quickly on the $4$-dimensional Acrobot task. We found in our experiments that LARS-TD did not converge within $20$ episodes (its curve, not shown in Figure~\ref{mcar-experiment},  would be flat on the vertical axis at $1000$ steps).
\begin{figure}[tbh]
\begin{center}
\begin{minipage}[t]{0.4\textwidth}
 \includegraphics[height=1.5in, width=2in]{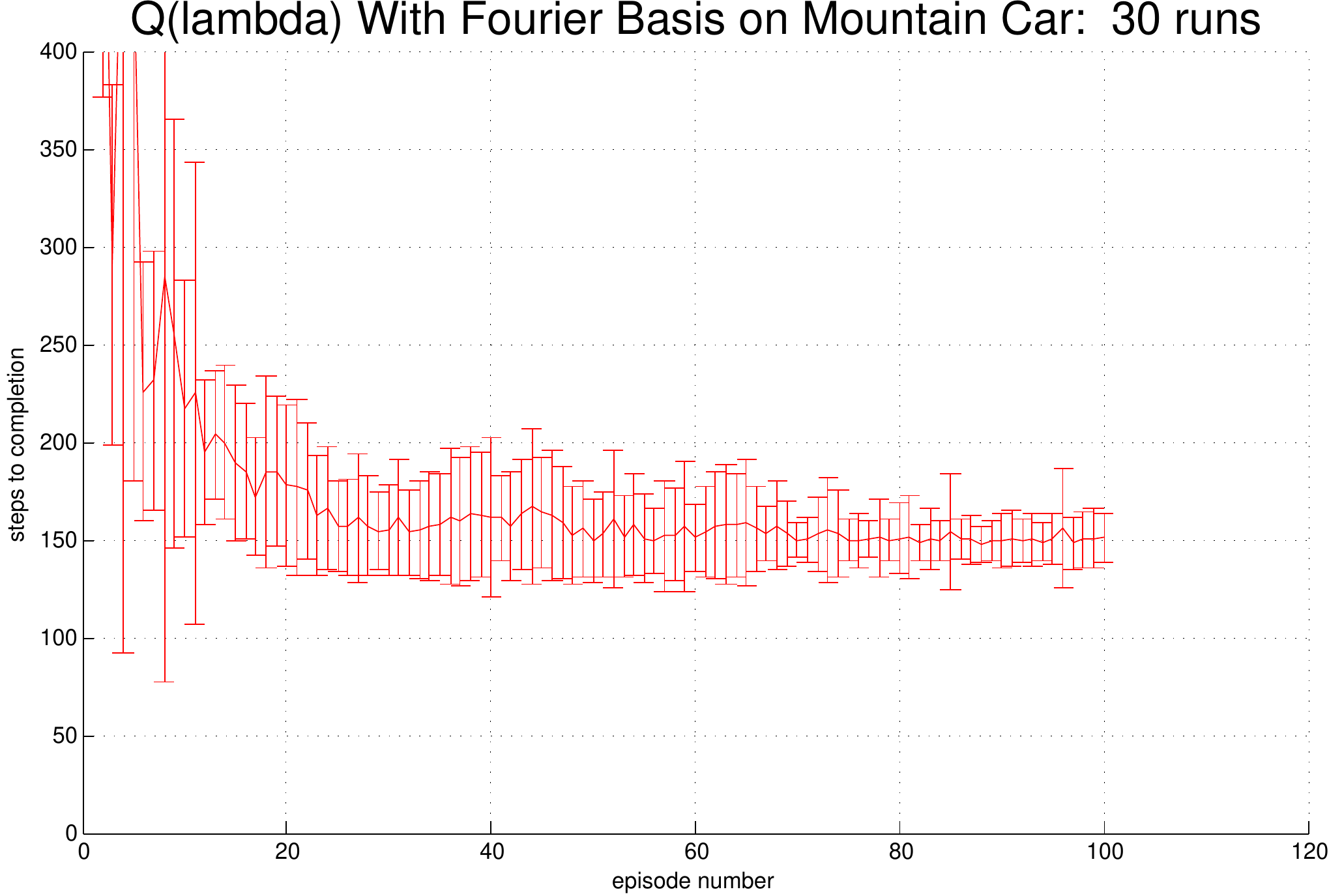}
\end{minipage} \hspace{.1in}
\begin{minipage}[t]{0.4\textwidth}
 \includegraphics[height=1.5in, width=2.5in]{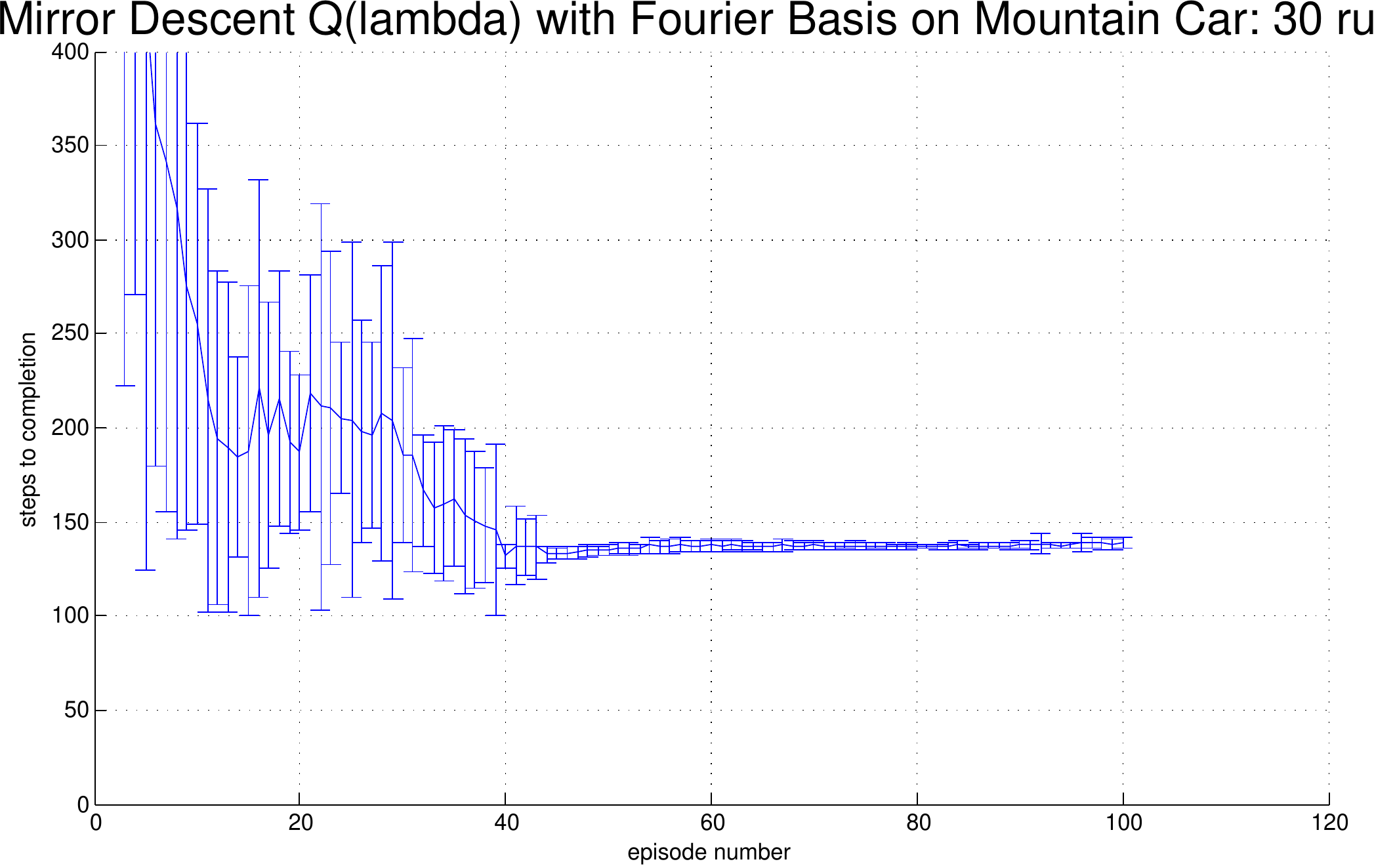}
\end{minipage}
\end{center}
\caption{Top: Q-learning; Bottom: mirror-descent Q-learning with p-norm link function, both with $25$ fixed Fourier bases \cite{konidaris2008value} for the mountain car task. }
\label{mcar-experiment}
\end{figure}
\begin{figure}[tbh]
\begin{center}
\begin{minipage}[t]{0.45\textwidth}
\includegraphics[height=1.35in, width=2in]{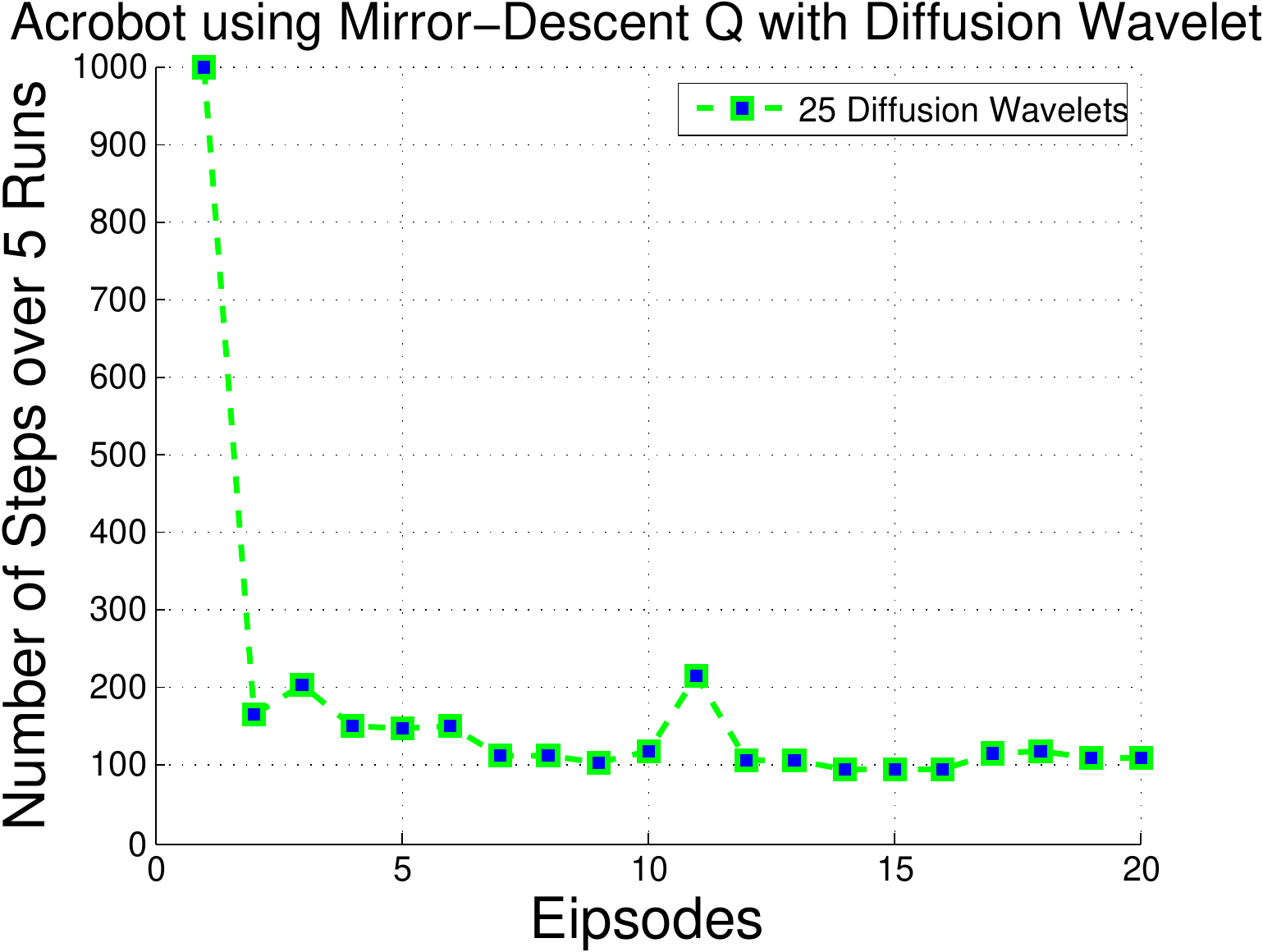}
\end{minipage}
\end{center}
\caption{Mirror-descent Q-learning on the Acrobot task using automatically generated diffusion wavelet bases  averaged over $5$ trials. }
\label{acrobot-experiment}
\end{figure}
Finally, we tested the mirror-descent approach on a more complex $8$-dimensional continuous MDP. The triple-link inverted pendulum \cite{si2001online}  is a highly nonlinear time-variant under-actuated system, which is a standard benchmark testbed in the control community. We base our simulation using the  system parameters described in  \cite{si2001online}, except that the action space is discretized because the algorithms described here are restricted to policies with discrete actions. There are three actions, namely $\{ 0,5{\rm{Newton}}, - 5{\rm{Newton}}\} $. The state space is $8$-dimensional, consisting of the angles made to the horizontal of the three links in the arm as well as their angular velocities, the position and velocity of the cart used to balance the pendulum.  The goal is to learn a policy that can balance the system with the  minimum number of episodes. A run is successful if it balances the inverted pendulum for the specified number of steps within  $300$ episodes, resulting in a reward of $0$. Otherwise, this run is considered as a failure and yields  a negative reward $-1$. The first action is chosen randomly to push the pendulum away from initial state. Two experiments were conducted on the triple-link pendulum domain with $20$ runs
for each experiment. As Table 1 shows, Mirror Descent Q-learning is able to learn the policy  with fewer episodes and  usually with reduced variance compared with  regular Q-learning.

The experiment settings are Experiment 1: Zero initial state and the system receives a reward $1$ if it is able to balance 10,000 steps.
Experiment 2: Zero initial state and the system receives a reward $1$ if it is able
to balance 100,000 steps.
Table 1 shows the comparison result between regular Q-learning and Mirror Descent
Q-learning.
\begin{table}
\centering
\begin{tabular}{|c|c|c|}
\hline
\# of Episodes\textbackslash{}Experiment & 1 & 2\tabularnewline
\hline
Q-learning                 & $6.1 \pm 5.67$    & $15.4 \pm 11.33$  \tabularnewline
\hline
Mirror Descent Q-learning & $5.7 \pm 9.70$     & $11.8 \pm 6.86$    \tabularnewline
\hline
\end{tabular}
\caption{Results on Triple-Link Inverted Pendulum Task.}
\label{table1}
\end{table}

\section{Comparison of Link Functions}\label{eg-vs-pnorm}
The  two most widely used link functions in mirror descent are the $p$-norm link function \cite{Beck:2003p2359} and the relative entropy function for exponentiated gradient (EG) \cite{kivinen:95}. Both of these link functions offer a multiplicative update rule compared with regular additive gradient methods. The differences between these two are discussed here. Firstly, the loss function for  EG is the relative entropy whereas that of the $p$-norm link function is the square $l_2$-norm function. Second and more importantly,  EG does not produce sparse solutions since it must maintain the weights away from zero, or else its potential (the relative entropy)  becomes unbounded at the boundary.

Another advantage of $p$-norm link functions over EG is that the $p$-norm link function offers a flexible interpolation between additive and multiplicative gradient updates. It has been shown that when the features are dense and the optimal coefficients $\theta^*$ are sparse, EG converges faster than the regular additive gradient methods \cite{kivinen:95}. However, according to our experience, a significant drawback of EG is the overflow of the coefficients due to the exponential operator. To prevent overflow, the most commonly used technique is rescaling: the weights are re-normalized to sum to a constant. However, it seems that this approach does not always work.  It has been pointed out  \cite{sutton-precup:egtd} that in the EG-Sarsa algorithm,  rescaling can fail, and replacing eligible traces instead of regular additive eligible traces is used to prevent overflow. EG-Sarsa usually poses restrictions on the basis as well. Thanks to the flexible interpolation capability between multiplicative and additive gradient updates, the $p$-norm link function is more robust and applicable to various basis functions, such as polynomial, radial basis function (RBF), Fourier basis \cite{konidaris2008value}, proto-value functions (PVFs), etc.

\section{Summary}

We proposed a novel framework for reinforcement learning using mirror-descent online convex optimization. Mirror Descent Q-learning demonstrates the following advantage over regular Q learning: faster convergence rate and reduced variance due to larger stepsizes with theoretical convergence guarantees \cite{nemirovski:siam}. Compared with existing sparse reinforcement learning algorithms such as LARS-TD, Algorithm 2 has lower sample complexity and lower computation cost, advantages accrued from the first-order mirror descent framework combined with proximal mapping \cite{ShalevShwartz:jmlr}.
There are many promising future research topics along this direction.
We are currently exploring a mirror-descent fast-gradient RL method, which is both convergent off-policy and quicker than fast gradient TD methods such as GTD and TDC  \cite{Sutton09fastgradient-descent}. To scale to large MDPs, we are investigating hierarchical mirror-descent RL methods, in particular extending SMDP Q-learning. We are also undertaking a  more detailed theoretical analysis of the mirror-descent RL framework, building on existing analysis of mirror-descent methods \cite{duchi:jmlr,ShalevShwartz:jmlr}. Two types of theoretical investigations are being explored: regret bounds of mirror-descent TD methods, extending previous results \cite{warmuth:mlj} and  convergence analysis combining robust stochastic approximation \cite{nemirovski:siam} and RL theory \cite{ndp:book,borkar:book}.

\chapter{Regularized Off-Policy Temporal Difference Learning}\label{chapter:Regularized-Off-Policy-TD}

In the last chapter we proposed an on-policy convergent sparse TD learning algorithm.
Although TD converges when samples are drawn ``on-policy''
by sampling from the Markov chain underlying a policy in a Markov
decision process (MDP), it can be shown to be divergent when samples
are drawn ``off-policy''.

In this chapter, the off-policy TD learning problem is formulated
from the stochastic optimization perspective. \footnote{This chapter is based on the paper "Regularized Off-Policy TD-Learning" published in NIPS 2012.} A novel objective function
is proposed based on the linear equation formulation of the TDC algorithm.
The optimization problem underlying off-policy TD methods, such as
TDC, is reformulated as a convex-concave saddle-point stochastic approximation
problem, which is both convex and incrementally solvable. A detailed
theoretical and experimental study of the RO-TD algorithm is presented.

\section{Introduction}
\subsection{Off-Policy Reinforcement Learning}

Off-policy learning refers to learning about one way of behaving,
called the \emph{target policy}, from sample sets that are generated
by another policy of choosing actions, which is called the \emph{behavior
policy}, or exploratory policy. As pointed out
in \cite{maei2011gradient}, the target policy is often a deterministic
policy that approximates the optimal policy, and the behavior policy
is often stochastic, exploring all possible actions in each state
as part of finding the optimal policy. Learning the target policy
from the samples generated by the behavior policy allows a greater
variety of exploration strategies to be used. It also enables learning
from training data generated by unrelated controllers, including manual
human control, and from previously collected data. Another reason
for interest in off-policy learning is that it enables learning about
multiple target policies (e.g., optimal policies for multiple sub-goals)
from a single exploratory policy generated by a single behavior policy,
which triggered an interesting research area termed as ``parallel
reinforcement learning''.
Besides, off-policy methods are of
wider applications since they are able to learn while executing an
exploratory policy, learn from demonstrations, and learn multiple
tasks in parallel \cite{OFFACTOR:2012}. Sutton et al. \cite{FastGradient:2009}
introduced convergent off-policy temporal difference learning algorithms,
such as TDC, whose computation time scales linearly with the number
of samples and the number of features. Recently, a linear off-policy
actor-critic algorithm based on the same framework was proposed in
\cite{OFFACTOR:2012}.

\subsection{Convex-concave Saddle-point First-order Algorithms \label{sec:saddle}}

The key novel contribution of this chapter is a convex-concave saddle-point
formulation for regularized off-policy TD learning. A convex-concave
saddle-point problem is formulated as follows. Let $x\in X,y\in Y$,
where $X,Y$ are both nonempty bounded closed convex sets, and $f(x):X\to\mathbb{R}$
be a convex function. If there exists a function $\varphi(\cdot,\cdot)$
such that $f(x)$ can be represented as $f(x):={\sup_{y\in Y}}\varphi(x,y)$,
then the pair $(\varphi,Y)$ is referred as the saddle-point representation
of $f$. The optimization problem of minimizing $f$ over $X$ is
converted into an equivalent convex-concave saddle-point problem $SadVal={\inf_{x\in X}}{\sup_{y\in Y}}\varphi(x,y)$
of $\varphi$ on $X\times Y$. If $f$ is non-smooth yet convex and
well structured, which is not suitable for many existing optimization
approaches requiring smoothness, its saddle-point representation $\varphi$
is often smooth and convex. Thus, convex-concave saddle-point problems
are, therefore, usually better suited for first-order methods \cite{sra2011optimization}.
A comprehensive overview on extending convex minimization to convex-concave
saddle-point problems with unified variational inequalities is presented
in \cite{ben2005non}. 
As an example, consider $f(x)=||Ax-b|{|_{m}}$ which admits a bilinear
minimax representation
\begin{equation}
f(x):={\left\Vert {Ax-b}\right\Vert _{m}}={\max_{{{\left\Vert y\right\Vert }_{n}}\le1}}{y^{T}}(Ax-b)\label{eq:minimax}
\end{equation}
where $m,n$ are conjugate numbers. Using the approach in \cite{RobustSA:2009},
Equation (\ref{eq:minimax}) can be solved as
\begin{equation}
{x_{t+1}}={x_{t}}-{\alpha_{t}}{A^{T}}{y_{t}},{y_{t+1}}={\Pi_{n}}({y_{t}}+{\alpha_{t}}(A{x_{t}}-b))\label{eq:minimaxFOM1}
\end{equation}
where $\Pi_{n}$ is the projection operator of $y$ onto the unit
$l_{n}$-ball ${\left\Vert y\right\Vert _{n}}\le{\rm {1}}$,which
is defined as
\begin{equation}
{\Pi_{n}}(y)=\min(1,1/{\left\Vert y\right\Vert _{n}})y,n=2,3,\cdots,{\Pi_{\infty}}{\rm {}}({y_{i}})=\min(1,1/|{y_{i}}|){y_{i}}\label{l2proj}
\end{equation}
and $\Pi_{\infty}$ is an entrywise operator.

\section{Problem Formulation}

\subsection{Objective Function Formulation}

Now let's review the concept of MSPBE. MSPBE is defined as
\begin{equation}
\begin{array}{l}
{\rm {MSPBE}}(\theta)\\
=\left\Vert {\Phi\theta-\Pi T(\Phi\theta)}\right\Vert _{\Xi}^{2}\\
={({\Phi^{T}}\Xi(T\Phi\theta-\Phi\theta))^{T}}{({\Phi^{T}}\Xi\Phi)^{-1}}{\Phi^{T}}\Xi(T\Phi\theta-\Phi\theta)\\
=\mathbb{E}{[\delta_t(\theta)\phi_t]^{T}}\mathbb{E}{[\phi_t{\phi_t^{T}}]^{-1}}\mathbb{E}[\delta_{t}(\theta)\phi_t]
\end{array}
\label{mspbe}
\end{equation}
To avoid computing the inverse matrix ${({\Phi^{T}}\Xi\Phi)^{-1}}$
and to avoid the double sampling problem \cite{sutton-barto:book}
in (\ref{mspbe}), an auxiliary variable $w$ is defined
\begin{equation}
w=\mathbb{E}{[\phi_t{\phi_t^{T}}]^{-1}}\mathbb{E}[\delta_{t}(\theta)\phi_t]={({\Phi^{T}}\Xi\Phi)^{-1}}{\Phi^{T}}\Xi(T\Phi\theta-\Phi\theta)
\label{eq:w}
\end{equation}

Thus we can have the following linear inverse problem
\begin{equation}
\mathbb{E}[\delta_{t}(\theta)\phi_t]=\mathbb{E}[\phi_t{\phi_t^{T}}]w=({\Phi^{T}}\Xi\Phi)w={\Phi^{T}}\Xi(T\Phi\theta-\Phi\theta)
\label{eq:tdc_lip1}
\end{equation}

By taking gradient w.r.t $\theta$ for optimum condition $\nabla{\rm {MSPBE}}(\theta)=0$
and utilizing Equation (\ref{eq:w}), we have
\begin{equation}
\mathbb{E}[\delta_{t}(\theta)\phi_t]=\gamma\mathbb{E}[\phi'_t{\phi_t^{T}}]w
\label{eq:tdc_lip2}
\end{equation}

Rearranging the two equality of Equation (\ref{eq:tdc_lip1},\ref{eq:tdc_lip2}),
we have the following linear system equation

\begin{equation}
\left[{\begin{array}{cc}
{\eta{\Phi^{T}}\Xi\Phi} & {\eta{\Phi^{T}}\Xi(\Phi-\gamma\Phi')}\\
{\gamma{\Phi^{'}}^{T}\Xi\Phi} & {{\Phi^{T}}\Xi(\Phi-\gamma\Phi')}
\end{array}}\right]\left[{\begin{array}{c}
w\\
\theta
\end{array}}\right]=\left[{\begin{array}{c}
{\eta{\Phi^{T}}\Xi R}\\
{{\Phi^{T}\Xi}R}
\end{array}}\right]\label{eq:eax=00003Db}
\end{equation}

The stochastic gradient version of the above equation is as follows,
where
\begin{equation}
A=\mathbb{E}[{A_{t}}],b=\mathbb{E}[{b_{t}}],x=[w;\theta]\label{eq:UNIFY}
\end{equation}

\begin{equation}
{A_{t}}=\left[{\begin{array}{cc}
{\eta{\phi_{t}}{\phi_{t}}^{T}} & {\eta{\phi_{t}}{{({\phi_{t}}-\gamma{{\phi'}_{t}})}^{T}}}\\
{\gamma{{\phi'}_{t}}{\phi_{t}}^{T}} & {{\phi_{t}}{{({\phi_{t}}-\gamma{{\phi'}_{t}})}^{T}}}
\end{array}}\right],{b_{t}}=\left[{\begin{array}{c}
{\eta{r_{t}}{\phi_{t}}}\\
{{r_{t}}{\phi_{t}}}
\end{array}}\right]\label{eq:abc}
\end{equation}

Following \cite{FastGradient:2009}, the TDC algorithm solution follows
from the linear equation $Ax=b$, where a single iteration gradient
update would be
\[
{x_{t+1}}={x_{t}}-{\alpha_{t}}({A_{t}}{x_{t}}-{b_{t}})
\]
where $x_{t}=[w_{t};\theta_{t}]$. The two time-scale gradient descent
learning method TDC \cite{FastGradient:2009} is
\begin{equation}
{\theta_{t+1}}={\theta_{t}}+{\alpha_{t}}{\delta_{t}}{\phi_{t}}-{\alpha_{t}}\gamma{\phi_{t}}^{\prime}(\phi_{t}^{T}{w_{t}}),{w_{t+1}}={w_{t}}+{\beta_{t}}({\delta_{t}}-\phi_{t}^{T}{w_{t}}){\phi_{t}}\label{eq:TDCupdate}
\end{equation}
where $-{\alpha_{t}}\gamma{\phi_{t}}^{\prime}(\phi_{t}^{T}{w_{t}})$
is the term for correction of gradient descent direction, and ${\beta_{t}}=\eta{\alpha_{t}},\eta>1$.
\begin{figure}[tbh]
\center 
\fbox{ %
\begin{tabular}{p{11.4cm}}
\begin{itemize}
\item $\Xi$ is a diagonal matrix whose entries $\xi(s)$ are given by a
positive probability distribution over states. $\Pi=\Phi{({\Phi^{T}}\Xi\Phi)^{-1}}{\Phi^{T}}\Xi$
is the weighted least-squares projection operator.
\item A square root of $A$ is a matrix $B$ satisfying $B^{2}=A$ and $B$
is denoted as ${A^{\frac{1}{2}}}$. Note that ${A^{\frac{1}{2}}}$
may not be unique.
\item $[\cdot,\cdot]$ is a row vector, and $[\cdot;\cdot]$ is a column
vector.
\item For the $t$-th sample, $\phi_{t}$ (the $t$-th row of $\Phi$),
$\phi'_{t}$ (the $t$-th row of $\Phi'$) are the feature vectors
corresponding to $s_{t},s'_{t}$, respectively. $\theta_{t}$ is the
coefficient vector for $t$-th sample in first-order TD learning methods,
and ${\delta_{t}}=({r_{t}}+\gamma\phi_{t}^{'T}{\theta_{t}})-\phi_{t}^{T}{\theta_{t}}$
is the temporal difference error. Also, ${x_{t}}=[{w_{t}};{\theta_{t}}]$,
$\alpha_{t}$ is a stepsize, ${\beta_{t}}=\eta{\alpha_{t}},\eta>0$.
\item $m,n$ are conjugate numbers if $\frac{1}{m}+\frac{1}{n}=1,m\ge1,n\ge1$.
$||x||{_{m}}={(\sum\nolimits _{j}{|{x_{j}}{|^{m}}})^{\frac{1}{m}}}$
is the $m$-norm of vector $x$.
\item $\rho$ is $l_{1}$ regularization parameter, $\lambda$ is the eligibility
trace factor, $N$ is the sample size, $d$ is the number of basis
functions, $k$ is the number of active basis functions. \end{itemize}
\tabularnewline
\end{tabular}} 
\caption{Notations and Definitions.}
\label{fig:notationmsma}
\end{figure}

There are some issues regarding the objective function, which arise
from the online convex optimization and reinforcement learning perspectives,
respectively. The first concern is that the objective function should
be convex and stochastically solvable. Note that $A,A_{t}$ are neither
PSD nor symmetric, and it is not straightforward to formulate a convex
objective function based on them. The second concern is that since
we do not have knowledge of $A$, the objective function should be
separable so that it is stochastically solvable based on $A_{t},b_{t}$.
The other concern regards the sampling condition in temporal difference
learning: double-sampling. As pointed out in \cite{sutton-barto:book},
double-sampling is a necessary condition to obtain an unbiased estimator
if the objective function is the Bellman residual or its derivatives
(such as projected Bellman residual), wherein the product of Bellman
error or projected Bellman error metrics are involved. To overcome
this sampling condition constraint, the product of TD errors should
be avoided in the computation of gradients. Consequently, based on
the linear equation formulation in (\ref{eq:UNIFY}) and the requirement
on the objective function discussed above, we propose the regularized
loss function as
\begin{equation}
L(x)={\left\Vert {Ax-b}\right\Vert _{m}}+h(x)\label{eq:lossfunc}
\end{equation}

Here we also enumerate some intuitive objective functions and give
a brief analysis on the reasons why they are not suitable for regularized
off-policy first-order TD learning. One intuitive idea is to add a
sparsity penalty on MSPBE, i.e., $L(\theta)={\rm {MSPBE(}}\theta{\rm {)+}}\rho{\left\Vert \theta\right\Vert _{1}}.$
Because of the $l_{1}$ penalty term, the solution to $\nabla L=0$
does not have an analytical form and is thus difficult to compute.
The second intuition is to use the online least squares formulation
of the linear equation $Ax=b$. However, since $A$ is not symmetric
and positive semi-definite (PSD), ${A^{\frac{1}{2}}}$ does not exist
and thus $Ax=b$ cannot be reformulated as ${\min_{x\in X}}||{{A^{\frac{1}{2}}}x-{A^{-\frac{1}{2}}}b}||_{2}^{2}$.
Another possible idea is to attempt to find an objective function
whose gradient is exactly ${A_{t}}{x_{t}}-{b_{t}}$ and thus the regularized
gradient is $pro{x_{{\alpha_{t}}h({x_{t}})}}({A_{t}}{x_{t}}-{b_{t}})$.
However, since $A_{t}$ is not symmetric, this gradient does not explicitly
correspond to any kind of optimization problem, not to mention a convex
one{%
\footnote{Note that the $A$ matrix in GTD2's linear equation representation
is symmetric, yet is not PSD, so it cannot be formulated as a convex
problem.%
}}.

\subsection{Squared Loss Formulation\label{sec:Extension}}

It is also worth noting that there exists another formulation of the
loss function different from Equation (\ref{eq:lossfunc}) with the
following convex-concave formulation as in \cite{nesterov:composite:2007,sra2011optimization},
\begin{eqnarray}
\mathop{\min}\limits _{x}\frac{1}{2}\left\Vert {Ax-b}\right\Vert _{2}^{2}+\rho{\left\Vert x\right\Vert _{1}} & = & \mathop{\max}\limits _{{{\left\Vert {{A^{T}}y}\right\Vert }_{\infty}}\le1}({b^{T}}y-\frac{\rho}{2}{y^{T}}y)\nonumber \\
 & = & \mathop{\min}\limits _{x}\mathop{\max}\limits _{{{\left\Vert u\right\Vert }_{\infty}}\le1,y}\left({{x^{T}}u+{y^{T}}(Ax-b)-\frac{\rho}{2}{y^{T}}y}\right)\label{eq:formula2}
\end{eqnarray}
\label{eq:ROTD2}

Here we give the detailed deduction of formulation in Equation (\ref{eq:formula2}).
First, using the dual norm representation, the standard LASSO problem
formulation is reformulated as

\begin{equation}
f(x)=\frac{1}{2}\left\Vert {Ax-b}\right\Vert _{2}^{2}+\rho{\left\Vert x\right\Vert _{1}}=\mathop{\max}\limits _{y,{{\left\Vert {{A^{T}}y}\right\Vert }_{\infty}}\le1}\left[{\left\langle {b/\rho,y}\right\rangle -\frac{1}{2}{y^{T}}y}\right]\label{eq:formula3}
\end{equation}

Then%
\footnote{Let $w=-y$, then we will have the same formulation as in Nemirovski's
tutorial in COLT2012.

\[
\Phi(x,w)=\left\langle {w,Ax-b}\right\rangle -\frac{1}{2}{w^{T}}w-\left\langle {x,{A^{T}}w}\right\rangle
\]
}

\[
\begin{array}{l}
\left\langle {b,y}\right\rangle -\frac{1}{2}{y^{T}}y=\left\langle {b,y}\right\rangle -\frac{1}{2}{y^{T}}y+\left\langle {x,{A^{T}}y}\right\rangle -\left\langle {y,Ax}\right\rangle \\
=\left\langle {y,b-Ax}\right\rangle -\frac{1}{2}{y^{T}}y+\left\langle {x,{A^{T}}y}\right\rangle
\end{array}
\]

which can be solved iteratively without the proximal gradient step
as follows, which serves as a counterpart of Equation (\ref{eq:RO-TD}),
\begin{eqnarray}
{x_{t+1}}={x_{t}}-{\alpha_{t}}\rho({u_{t}}+{A_{t}}^{T}{y_{t}}) & , & {y_{t+1}}={y_{t}}+\frac{{\alpha_{t}}}{\rho}({A_{t}}{x_{t}}-{b_{t}}-\rho{y_{t}})\nonumber \\
{u_{t+\frac{1}{2}}}={u_{t}}+\frac{{\alpha_{t}}}{\rho}{x_{t}} & , & {u_{t+1}}={\Pi_{\infty}}({u_{t+\frac{1}{2}}})
\end{eqnarray}

\section{Algorithm Design}

\subsection{RO-TD Algorithm Design \label{sec:stochasticsaddlepoint}}

In this section, the problem of (\ref{eq:lossfunc}) is formulated
as a convex-concave saddle-point problem, and the RO-TD algorithm
is proposed. Analogous to (\ref{eq:minimax}), the regularized loss
function can be formulated as
\begin{equation}
{\left\Vert {Ax-b}\right\Vert _{m}}+h(x)={\max_{{{\left\Vert y\right\Vert }_{n}}\le1}}{y^{T}}(Ax-b)+h(x)\label{eq:regminimax}
\end{equation}
Similar to (\ref{eq:minimaxFOM}), Equation (\ref{eq:regminimax})
can be solved via an iteration procedure as follows, where ${{x}_{t}}=\left[{{w_{t}};{\theta_{t}}}\right]$.
\begin{eqnarray}
{{x}_{t+\frac{1}{2}}}={x_{t}}-{\alpha_{t}}{A_{t}^{T}}{y_{t}} & , & {{y}_{t+\frac{1}{2}}}={y_{t}}+\alpha_{t}({A_{t}}{x_{t}}-{b_{t}})\nonumber \\
{x_{t+1}}=pro{x_{{\alpha_{t}}h}}({x_{t+\frac{1}{2}}}) & , & {y_{t+1}}={\Pi_{n}}({y_{t+\frac{1}{2}}})\label{eq:RO-TD}
\end{eqnarray}
The averaging step, which plays a crucial role in stochastic optimization
convergence, generates the \emph{approximate saddle-points} \cite{sra2011optimization,nedic2009subgradient}
\begin{equation}
{\bar{x}_{t}}={\left({\sum\nolimits _{i=0}^{t}{\alpha_{i}}}\right)^{-1}}\sum\nolimits _{i=0}^{t}{{\alpha_{i}}{x_{i}}},{\bar{y}_{t}}={\left({\sum\nolimits _{i=0}^{t}{\alpha_{i}}}\right)^{-1}}\sum\nolimits _{i=0}^{t}{{\alpha_{i}}{y_{i}}}\label{eq:averaging}
\end{equation}
Due to the computation of $A_{t}$ in (\ref{eq:RO-TD}) at each iteration,
the computation cost appears to be $O(Nd^{2})$, where $N,d$ are
defined in Figure \ref{fig:notationmsma}. However, the computation
cost is actually $O(Nd)$ with a linear algebraic trick by computing
not $A_{t}$ but $y_{t}^{T}{A_{t}},{A_{t}}{x_{t}}-{b_{t}}$. Denoting
${y_{t}}=[{y_{1,t}};{y_{2,t}}]$, where ${y_{1,t}};{y_{2,t}}$ are
column vectors of equal length, we have
\begin{equation}
y_{t}^{T}{A_{t}}=\left[{\begin{array}{cc}
{\eta\phi_{t}^{T}(y_{1,t}^{T}{\phi_{t}})+\gamma\phi_{t}^{T}(y_{2,t}^{T}\phi_{t}^{\prime})} & {{{({\phi_{t}}-\gamma\phi_{t}^{\prime})}^{T}}(\eta y_{1,t}^{T}+y_{2,t}^{T}){\phi_{t}}}\end{array}}\right]\label{eq:ytAt}
\end{equation}
${A_{t}}{x_{t}}-{b_{t}}$ can be computed according to Equation (\ref{eq:TDCupdate})
as follows:
\begin{equation}
{A_{t}}{x_{t}}-{b_{t}}=\left[{\begin{array}{c}
{-\eta({\delta_{t}}-\phi_{t}^{T}{w_{t}}){\phi_{t}}};{\gamma(\phi_{t}^{T}{w_{t}}){\phi_{t}}^{\prime}-{\delta_{t}}{\phi_{t}}}\end{array}}\right]\label{eq:Atxt-bt}
\end{equation}
Both (\ref{eq:ytAt}) and (\ref{eq:Atxt-bt}) are of linear computational
complexity. Now we are ready to present the RO-TD algorithm:

There are some design details of the algorithm to be elaborated. First,
the regularization term $h(x)$ can be any kind of convex regularization,
such as ridge regression or sparsity penalty $\rho||x||_{1}$. In
case of $h(x)=\rho||x||_{1}$, $pro{x_{{\alpha_{t}}h}}(\cdot)={S_{{\alpha_{t}}\rho}}(\cdot)$.
In real applications the sparsification requirement on $\theta$ and
auxiliary variable $w$ may be different, i.e., $h(x)={\rho_{1}}{\left\Vert \theta\right\Vert _{1}}+{\rho_{2}}{\left\Vert w\right\Vert _{1}},{\rho_{1}}\ne{\rho_{2}}$,
one can simply replace the uniform soft thresholding $S_{{\alpha_{t}}\rho}$
by two separate soft thresholding operations $S_{{\alpha_{t}}\rho_{1}},S_{{\alpha_{t}}\rho_{2}}$
and thus the third equation in (\ref{eq:RO-TD}) is replaced by the
following,
\begin{equation}
{x_{t+\frac{1}{2}}}=\left[{{w_{t+\frac{1}{2}}};{\theta_{t+\frac{1}{2}}}}\right],{\theta_{t+1}}={S_{{\alpha_{t}}{\rho_{1}}}}({\theta_{t+\frac{1}{2}}}),{w_{t+1}}={S_{{\alpha_{t}}{\rho_{2}}}}({w_{t+\frac{1}{2}}})\label{twothreshold}
\end{equation}
Another concern is the choice of conjugate numbers $(m,n)$. For ease
of computing $\Pi_{n}$, we use $(2,2)$($l_{2}$ fit), $(+\infty,1)$(uniform
fit) or $(1,+\infty)$. $m=n=2$ is used in the experiments below.

\begin{algorithm}[tph]
\caption{RO-TD}
\label{alg: RO-TD}

Let $\pi$ be some fixed policy of an MDP $M$, and let the sample
set $S=\{{s_{i}},{r_{i}},{s_{i}}'\}_{i=1}^{N}$. Let $\Phi$ be some
fixed basis.
\begin{enumerate}
\item \textbf{REPEAT}
\item Compute ${{\phi_{t}},{\phi_{t}}'}$ and TD error ${\delta_{t}}=({r_{t}}+\gamma\phi_{t}^{'T}{\theta_{t}})-\phi_{t}^{T}{\theta_{t}}$
\item Compute $y_{_{t}}^{T}{A_{t}},{A_{t}}{x_{t}}-{b_{t}}$ in Equation
(\ref{eq:ytAt}) and (\ref{eq:Atxt-bt}).
\item Compute $x_{t+1},y_{t+1}$ as in Equation (\ref{eq:RO-TD})
\item Set $t\leftarrow t+1$;
\item \textbf{UNTIL} {$t=N$};
\item Compute $\bar{x}_{N},\bar{y}_{N}$ as in Equation (\ref{eq:averaging})
with $t=N$ .\end{enumerate}
\end{algorithm}

\subsection{RO-GQ($\lambda$) Design}

GQ($\lambda$)\cite{gq:maei2010} is a generalization of the TDC algorithm
with eligibility traces and off-policy learning of temporally abstract
predictions, where the gradient update changes from Equation (\ref{eq:TDCupdate})
to
\begin{equation}
{\theta_{t+1}}={\theta_{t}}+{\alpha_{t}}[{\delta_{t}}{e_{t}}-\gamma(1-\lambda){w_{t}}^{T}{e_{t}}{\bar{\phi}_{t+1}}],{w_{t+1}}={w_{t}}+{\beta_{t}}({\delta_{t}}{e_{t}}-w_{t}^{T}{\phi_{t}}{\phi_{t}})\label{eq:gq}
\end{equation}
The central element is to extend the MSPBE function to the case where
it incorporates eligibility traces. The objective function and corresponding
linear equation component $A_{t},b_{t}$ can be written as follows:
\begin{equation}
L(\theta)=||\Phi\theta-\Pi T^{\pi\lambda}\Phi\theta||_{\Xi}^{2}\label{eq:gqObj}
\end{equation}
\begin{equation}
{A_{t}}=\left[{\begin{array}{cc}
{\eta{\phi_{t}}{\phi_{t}}^{T}} & {\eta{e_{t}}{{({\phi_{t}}-\gamma{\bar{\phi}_{t+1}})}^{T}}}\\
{\gamma(1-\lambda){{\bar{\phi}}_{t+1}}e_{t}^{T}} & {{e_{t}}{{({\phi_{t}}-\gamma{\bar{\phi}_{t+1}})}^{T}}}
\end{array}}\right],{b_{t}}=\left[{\begin{array}{c}
{\eta{r_{t}}{e_{t}}}\\
{{r_{t}}{e_{t}}}
\end{array}}\right]\label{eq:gqAb}
\end{equation}
Similar to Equation (\ref{eq:ytAt}) and (\ref{eq:Atxt-bt}), the
computation of $y_{_{t}}^{T}{A_{t}},{A_{t}}{x_{t}}-{b_{t}}$ is
\begin{eqnarray}
y_{_{t}}^{T}{A_{t}} & = & \left[{\begin{array}{cc}
{\eta\phi_{t}^{T}(y_{1,t}^{T}{\phi_{t}})+\gamma(1-\lambda)e_{t}^{T}(y_{2,t}^{T}{\bar{\phi}_{t+1}})} & {{({\phi_{t}}-\gamma{\bar{\phi}_{t+1}})^{T}}(\eta y_{1,t}^{T}+y_{2,t}^{T}){e_{t}}}\end{array}}\right]\nonumber \\
{A_{t}}{x_{t}}-{b_{t}} & = & \left[{\begin{array}{c}
{-\eta({\delta_{t}}{e_{t}}-\phi_{t}^{T}{w_{t}}{\phi_{t}})};{\gamma(1-\lambda)(e_{t}^{T}{w_{t}}){\bar{\phi}_{t+1}}-{\delta_{t}}{e_{t}}}\end{array}}\right]\label{eq:gqcase}
\end{eqnarray}
where eligibility traces $e_{t}$, and ${\bar{\phi}_{t}},T^{\pi\lambda}$
are defined in \cite{gq:maei2010}. Algorithm \ref{alg: RO-GQ}, RO-GQ($\lambda$),
extends the RO-TD algorithm to include eligibility traces.

\begin{algorithm}[tph]
\caption{RO-GQ($\lambda$)}
\label{alg: RO-GQ}

Let $\pi$ be some fixed policy of an MDP $M$. Let $\Phi$ be some
fixed basis. Starting from $s_{0}$.
\begin{enumerate}
\item \textbf{REPEAT}
\item Compute ${{\phi_{t}},{\bar{\phi}_{t+1}}}$ and TD error ${\delta_{t}}=({r_{t}}+\gamma\bar{\phi}_{t+1}^{T}{\theta_{t}})-\phi_{t}^{T}{\theta_{t}}$
\item Compute $y_{_{t}}^{T}{A_{t}},{A_{t}}{x_{t}}-{b_{t}}$ in Equation
(\ref{eq:gqcase}).
\item Compute $x_{t+1},y_{t+1}$ as in Equation (\ref{eq:RO-TD})
\item Choose action $a_{t}$, and get $s_{t+1}$
\item Set $t\leftarrow t+1$;
\item \textbf{UNTIL} $s_{t}$ is an absorbing state;
\item Compute $\bar{x}_{t},\bar{y}_{t}$ as in Equation (\ref{eq:averaging}) \end{enumerate}
\end{algorithm}

\section{Theoretical Analysis}

The theoretical analysis of RO-TD algorithm can be seen in the Appendix.

\section{Empirical Results\label{sec:Empirical-Experiments-1}}

We now demonstrate the effectiveness of the RO-TD algorithm against
other algorithms across a number of benchmark domains. LARS-TD \cite{Kolter09LARSTD},
which is a popular second-order sparse reinforcement learning algorithm,
is used as the baseline algorithm for feature selection and TDC is
used as the off-policy convergent RL baseline algorithm, respectively.

\subsection{MSPBE Minimization and Off-Policy Convergence}

\begin{figure}[tph]
\centering %
\begin{minipage}[c]{1\textwidth}%
\begin{center}
\includegraphics[width=4in,height=2in]{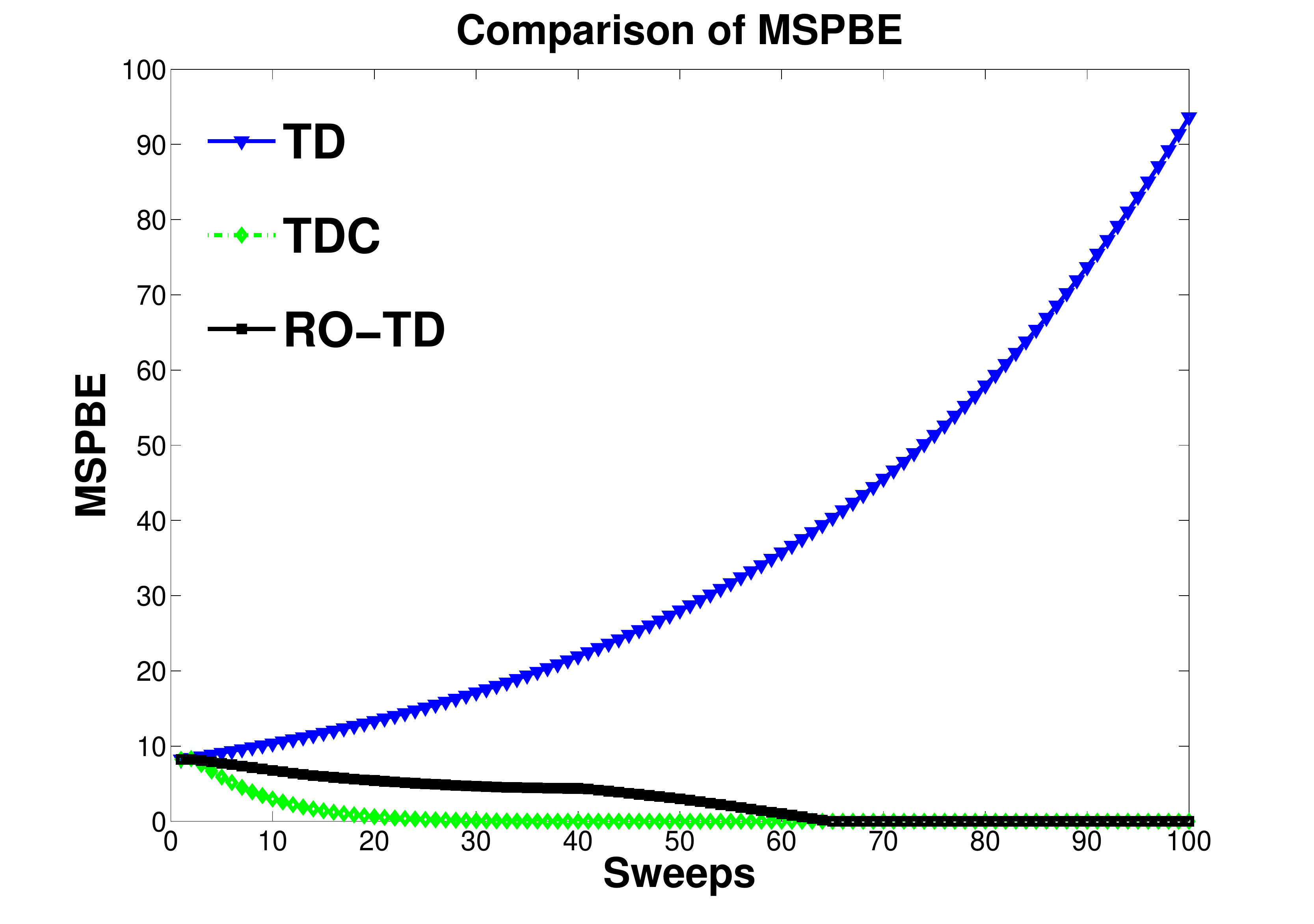}
\par\end{center}

\begin{center}
\includegraphics[width=4in,height=2in]{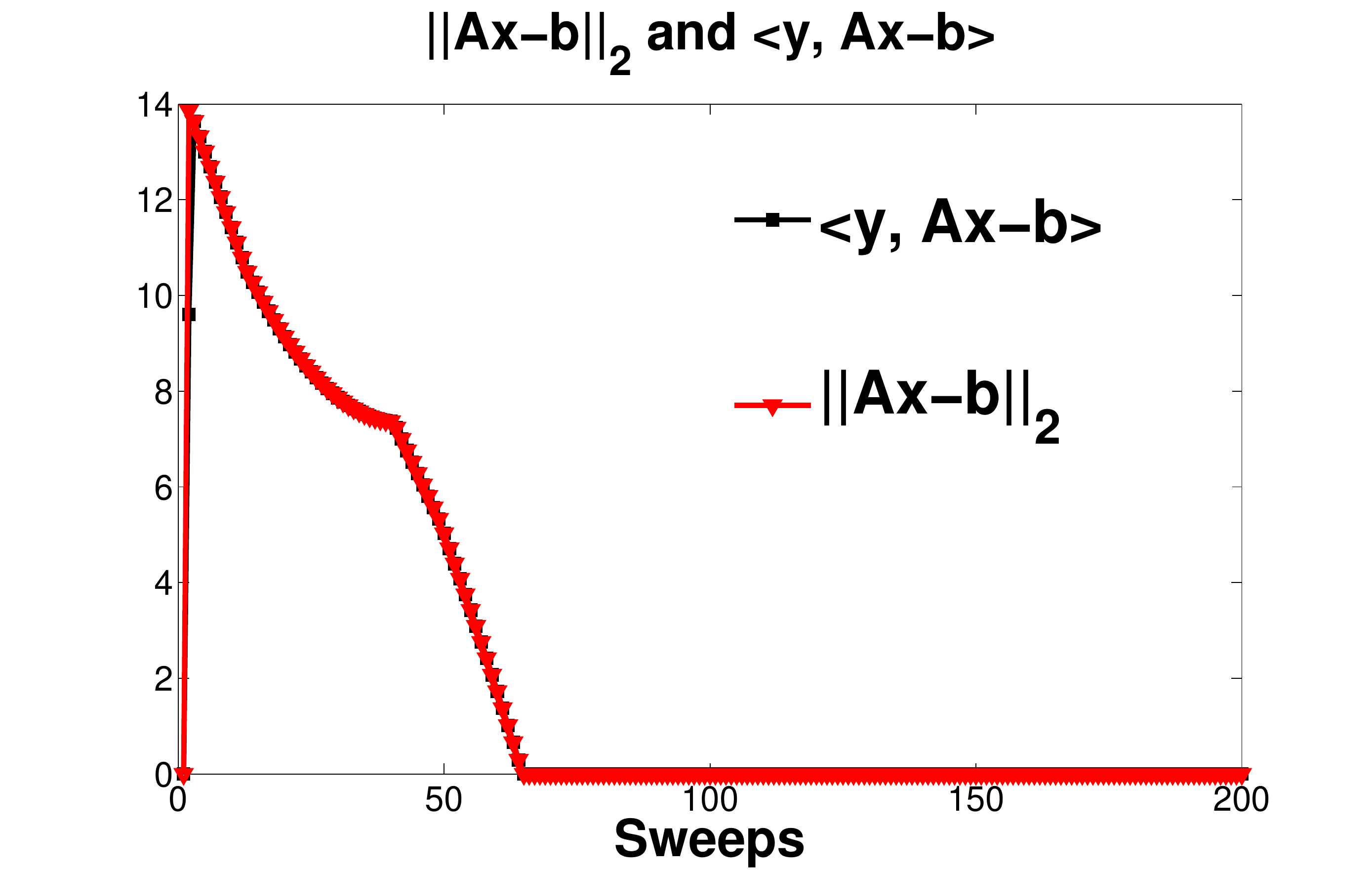}
\par\end{center}

\begin{center}
\includegraphics[width=4in,height=2in]{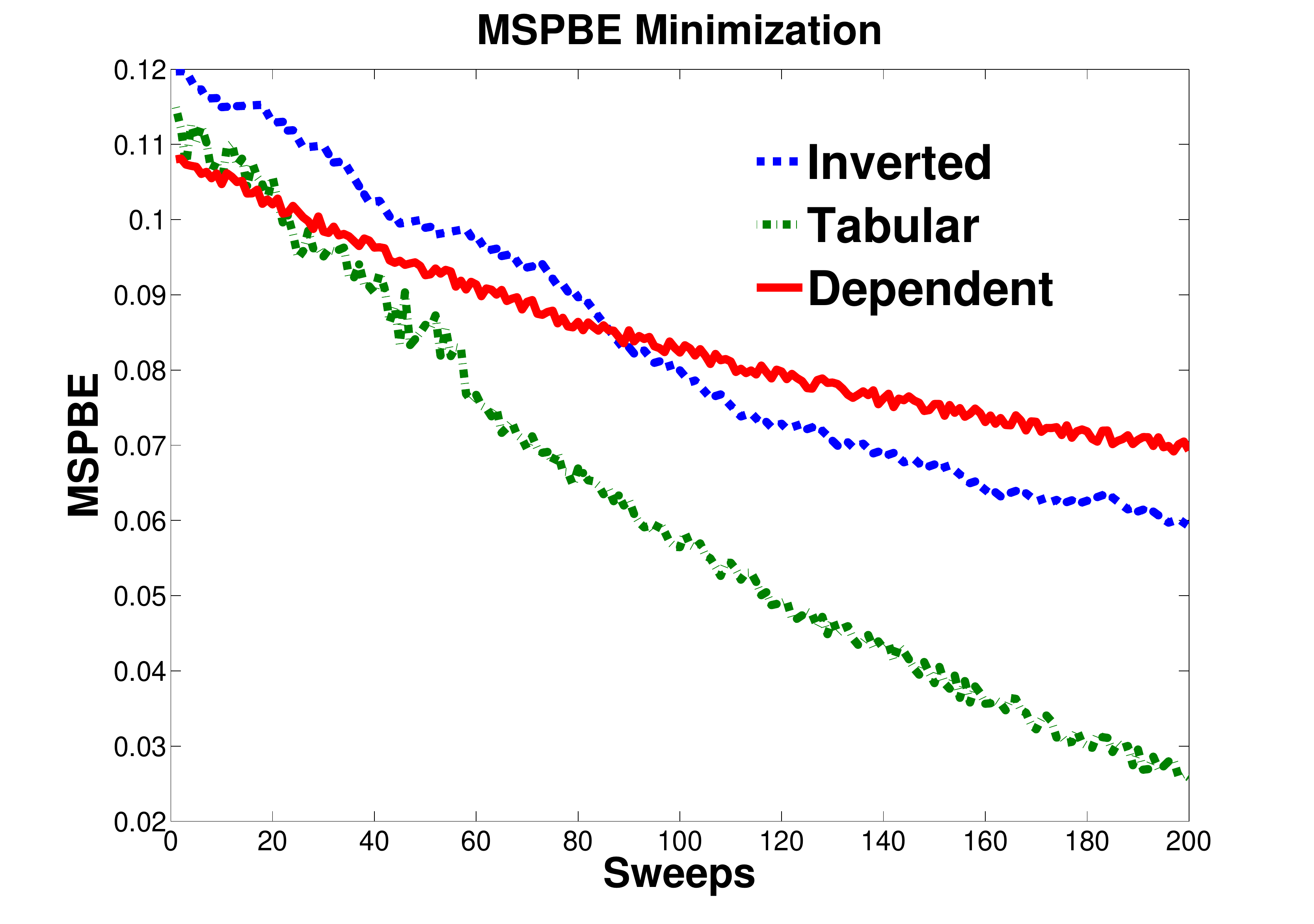}
\caption{Illustrative examples of the convergence of RO-TD using the Star and
Random-walk MDPs.}
\label{fig:STAR}
\par\end{center}

\end{minipage}
\end{figure}

This experiment aims to show the minimization of MSPBE and off-policy
convergence of the RO-TD algorithm. The $7$ state star MDP is a well
known counterexample where TD diverges monotonically and TDC converges.
It consists of $7$ states and the reward w.r.t any transition is
zero. Because of this, the star MDP is unsuitable for LSTD-based algorithms,
including LARS-TD since ${\Phi^{T}}R=0$ always holds. The random-walk
problem is a standard Markov chain with $5$ states and two absorbing
state at two ends. Three sets of different bases $\Phi$ are used
in \cite{FastGradient:2009}, which are tabular features, inverted
features and dependent features respectively. An identical experiment
setting to \cite{FastGradient:2009} is used for these two domains.
The regularization term $h(x)$ is set to $0$ to make a fair comparison
with TD and TDC. $\alpha=0.01$, $\eta=10$ for TD, TDC and RO-TD.
The comparison with TD, TDC and RO-TD is shown in the left sub-figure
of Figure \ref{fig:STAR}, where TDC and RO-TD have almost identical
MSPBE over iterations. The middle sub-figure shows the value of $y_{_{t}}^{T}(A{x_{t}}-b)$
and ${\left\Vert {A{x_{t}}-b}\right\Vert _{2}}$, wherein ${\left\Vert {A{x_{t}}-b}\right\Vert _{2}}$
is always greater than the value of $y_{_{t}}^{T}(A{x_{t}}-b)$. Note
that for this problem, the Slater condition is satisfied so there
is no duality gap between the two curves. As the result shows, TDC
and RO-TD perform equally well, which illustrates the off-policy convergence
of the RO-TD algorithm. The result of random-walk chain is averaged
over $50$ runs. The rightmost sub-figure of Figure \ref{fig:STAR}
shows that RO-TD is able to reduce MSPBE over successive iterations
w.r.t three different basis functions.

\subsection{Feature Selection}

In this section, we use the mountain car example with a variety of
bases to show the feature selection capability of RO-TD. The Mountain
car is an optimal control problem with a continuous two-dimensional
state space. The steep discontinuity in the value function makes learning
difficult for bases with global support. To make a fair comparison,
we use the same basis function setting as in \cite{Kolter09LARSTD},
where two dimensional grids of $2,4,8,16,32$ RBFs are used so that
there are totally $1365$ basis functions. For LARS-TD, $500$ samples
are used. For RO-TD and TDC, $3000$ samples are used by executing
$15$ episodes with $200$ steps for each episode, stepsize $\alpha_{t}=0.001$,
and $\rho_{1}=0.01,\rho_{2}=0.2$. We use the result of LARS-TD and
$l_{2}$ LSTD reported in \cite{Kolter09LARSTD}. As the result shows
in Table~\ref{tab:mcar}, RO-TD is able to perform feature selection
successfully, whereas TDC and TD failed. It is worth noting that comparing
the performance of RO-TD and LARS-TD is not the major focus here,
since LARS-TD is not convergent off-policy and RO-TD's performance
can be further optimized using the mirror-descent approach with the
Mirror-Prox algorithm \cite{sra2011optimization} which incorporates
mirror descent with an extragradient \cite{extragradient:1976}, as
discussed below. 

\begin{table}[h]
\begin{centering}
\centering %
\begin{tabular}{|c|c|c|c|c|c|}
\hline
Algorithm  & LARS-TD  & RO-TD  & $l_{2}$ LSTD  & TDC  & TD\tabularnewline
\hline
Success($20/20$)  & $100\%$  & $100\%$  & $0\%$  & $0\%$  & $0\%$ \tabularnewline
\hline
Steps  & $142.25\pm9.74$  & $147.40\pm13.31$  & -  & -  & - \tabularnewline
\hline
\end{tabular}\caption{Comparison of TD, LARS-TD, RO-TD, $l_{2}$ LSTD, TDC and TD}
\label{tab:mcar}
\par\end{centering}
\end{table}

\begin{table}[tph]
\begin{centering}
\centering %
\begin{tabular}{|c|c|c|c|}
\hline
Experiment\textbackslash{}Method  & RO-GQ($\lambda$)  & GQ($\lambda$)  & LARS-TD \tabularnewline
\hline
Experiment 1  & $6.9\pm4.82$  & $11.3\pm9.58$  & - \tabularnewline
\hline
Experiment 2  & $14.7\pm10.70$  & $27.2\pm6.52$  & - \tabularnewline
\hline
\end{tabular}\caption{Comparison of RO-GQ($\lambda$), GQ($\lambda$), and LARS-TD on Triple-Link
Inverted Pendulum Task}
\label{tab:rogq-triple}
\par\end{centering}
\end{table}

\subsection{High-dimensional Under-actuated Systems}

The triple-link inverted pendulum \cite{si2001online} is a highly
nonlinear under-actuated system with $8$-dimensional state space
and discrete action space. The state space consists of the angles
and angular velocity of each arm as well as the position and velocity
of the car. The discrete action space is $\{0,5{\rm {Newton}},-5{\rm {Newton}}\}$.
The goal is to learn a policy that can balance the arms for $N_{x}$
steps within some minimum number of learning episodes. The allowed
maximum number of episodes is $300$. The pendulum initiates from
zero equilibrium state and the first action is randomly chosen to
push the pendulum away from initial state. We test the performance
of RO-GQ($\lambda$), GQ($\lambda$) and LARS-TD. Two experiments
are conducted with $N_{x}=10,000$ and $100,000$, respectively. Fourier
basis \cite{konidaris:fourier} with order $2$ is used, resulting
in $6561$ basis functions. Table \ref{tab:rogq-triple} shows the results of this experiment,
where RO-GQ($\lambda$) performs better than other approaches, especially
in Experiment 2, which is a harder task. LARS-TD failed in this domain,
which is mainly not due to LARS-TD itself but the quality of samples
collected via random walk.

To sum up, RO-GQ($\lambda$) tends to outperform GQ($\lambda$) in
all aspects, and is able to outperform LARS-TD based policy iteration
in high dimensional domains, as well as in selected smaller MDPs where
LARS-TD diverges (e.g., the star MDP). It is worth noting that the
computation cost of LARS-TD is $O(Ndk^{2})$, where that for RO-TD
is $O(Nd)$. If $k$ is linear or sublinear w.r.t $d$, RO-TD has
a significant advantage over LARS-TD. However, compared with LARS-TD,
RO-TD requires fine tuning the parameters of $\alpha_{t},\rho_{1},\rho_{2}$
and is usually not as sample efficient as LARS-TD. We also find that
tuning the sparsity parameter $\rho_{2}$ generates an interpolation
between GQ($\lambda$) and Q-learning, where a large $\rho_{2}$
helps eliminate the correction term of TDC update and make the update
direction more similar to the TD update.

\section{Summary}

In this chapter we present a novel unified framework for designing regularized
off-policy convergent RL algorithms combining a convex-concave saddle-point
problem formulation for RL with stochastic first-order methods. A
detailed experimental analysis reveals that the proposed RO-TD algorithm
is both off-policy convergent and robust to noisy features.

\chapter{Safe Reinforcement Learning using Projected Natural Actor Critic}
\label{pnac-chapter}

Natural actor-critics form a popular class of policy search algorithms for finding locally optimal policies for Markov decision processes. In this paper we address a drawback of natural actor-critics that limits their real-world applicability---their lack of safety guarantees. We present a principled algorithm for performing natural gradient descent over a constrained domain \footnote{This paper is a revised version of the paper ``Projected Natural Actor-Critic" that was published in NIPS 2013.}. In the context of reinforcement learning, this allows for natural actor-critic algorithms that are guaranteed to remain within a known safe region of policy space. While deriving our class of constrained natural actor-critic algorithms, which we call Projected Natural Actor-Critics (PNACs), we also elucidate the relationship between natural gradient descent and mirror descent.

\section{Introduction}
Natural actor-critics form a class of policy search algorithms for finding locally optimal policies for Markov decision processes (MDPs
) by approximating and ascending the natural gradient \cite{Amari1998} of an objective function.  Despite the numerous successes of, and the continually growing interest in, natural actor-critic algorithms, they have not achieved widespread use for real-world applications. A lack of safety guarantees is a common reason for avoiding the use of natural actor-critic algorithms, particularly for biomedical applications. Since natural actor-critics are \emph{unconstrained} optimization algorithms, there are no guarantees that they will avoid regions of policy space that are known to be dangerous.

For example, proportional-integral-derivative controllers (PID controllers) are the most widely used control algorithms in industry, and have been studied in depth \cite{Astrom1995}.
Techniques exist for determining the set of stable gains (policy parameters) when a model of the system is available \cite{Soylemez2003}. Policy search can be used to find the optimal gains within this set (for some definition of optimality). 
A desirable property of a policy search algorithm in this context would be a guarantee that it will remain within the predicted region of stable gains during its search.

Consider a second example: \emph{functional electrical stimulation} (FES) control of a human arm. By selectively stimulating muscles using subcutaneous probes, researchers have made significant strides toward returning motor control to people suffering from paralysis induced by spinal cord injury \cite{Lynch2008}. There has been a recent push to develop controllers that specify how much and when to stimulate each muscle in a human arm to move it from its current position to a desired position \cite{Chadwick2009}. This closed-loop control problem is particularly challenging because each person's arm has different dynamics due to differences in, for example, length, mass, strength, clothing, and amounts of muscle atrophy, spasticity, and fatigue. Moreover, these differences are challenging to model. Hence, a proportional-derivative (PD) controller, tuned to a simulation of an ideal human arm, required manual tuning to obtain desirable performance on a human subject with biceps spasticity \cite{KathyPD}.

Researchers have shown that policy search algorithms are a viable approach to creating controllers that can automatically adapt to an individual's arm by training on a few hundred two-second reaching movements \cite{ThomasIAAI}. 
However, safety concerns have been raised in regard to both this specific application and other biomedical applications of policy search algorithms. Specifically, the existing state-of-the-art gradient-based algorithms, including the current natural actor-critic algorithms, are unconstrained and could potentially select dangerous policies. For example, it is known that certain muscle stimulations could cause the dislocation of a subject's arm. Although we lack an accurate model of each individual's arm, we can generate conservative safety constraints on the space of policies. Once again, a desirable property of a policy search algorithm would be a guarantee that it will remain within a specified region of policy space (known-safe policies).

In this paper we present a class of natural actor-critic algorithms that perform constrained optimization---given a known safe region of policy space, they search for a locally optimal policy while always remaining within the specified region. We call our class of algorithms \emph{Projected Natural Actor-Critics} (PNACs) since, whenever they generate a new policy, they project the policy back to the set of safe policies. The interesting question is how the projection can be done in a principled manner. We show that natural gradient descent (ascent), which is an \emph{unconstrained} optimization algorithm, is a special case of mirror descent (ascent), which is a \emph{constrained} optimization algorithm. In order to create a projected natural gradient algorithm, we add constraints in the mirror descent algorithm that is equivalent to natural gradient descent. We apply this projected natural gradient algorithm to policy search to create the PNAC algorithms, which we validate empirically.

\section{Related Work}
Researchers have addressed safety concerns like these before \cite{Perkins2002
}. \citet{Bendrahim1997} showed how a walking biped robot can switch to a stabilizing controller whenever the robot leaves a stable region of state space. Similar state-avoidant approaches to safety have been proposed by several others \cite{Arapostathis2005,
Arvelo2012,
Geibel2005}. These approaches do not account for situations where, over an unavoidable region of state space, the actions themselves are dangerous. \citet{Kuindersma2012} developed a method for performing risk-sensitive policy search, which models the variance of the objective function for each policy and permits runtime adjustments of risk sensitivity. However, their approach does not guarantee that an unsafe region of state space or policy space will be avoided.

\citet{Bhatnagar2009} presented projected natural actor-critic algorithms for the average reward setting. As in our projected natural actor-critic algorithms, they proposed computing the update to the policy parameters and then projecting back to the set of allowed policy parameters. However, they did not specify how the projection could be done in a principled manner. We show in Section \ref{sec:CP} that the Euclidean projection can be arbitrarily bad, and argue that the projection that we propose is particularly compatible with natural actor-critics (natural gradient descent).

\citet{Duchi2010} presented mirror descent using the Mahalanobis norm for the proximal function, which is very similar to the proximal function that we show to cause mirror descent to be equivalent to natural gradient descent. However, their proximal function is not identical to ours and they did not discuss any possible relationship between mirror descent and natural gradient descent.

\section{Equivalence of Natural Gradient Descent and Mirror Descent}

We begin by showing an important relationship between natural gradient methods and mirror descent.

\begin{thm}
The natural gradient descent update at step $k$ with metric tensor $G_k \triangleq G(x_k)$:
\begin{equation}
x_{k+1} = x_k - \alpha_k G_k^{-1}\nabla f(x_k),
\label{eq:NG}
\end{equation}
is equivalent to the mirror descent update at step $k$, with $\psi_k(x)=(\sfrac{1}{2})x^\intercal G_k x$.
\end{thm}
\begin{proof}
First, notice that $\nabla \psi_k(x) = G_kx$. Next, we derive a closed-form for $\psi_k^*$:
\begin{align}
\psi_k^*(y) = \max_{x \in \mathbb R^n} \left \{ x^\intercal y - \frac{1}{2}x^\intercal G_k x \right \}.
\label{eq:proof1}
\end{align}
Since the function being maximized on the right hand side is strictly concave, the $x$ that maximizes it is its critical point. Solving for this critical point, we get $x = G_k^{-1}y$. Substituting this into \eqref{eq:proof1}, we find that
$
\psi_k^*(y)=(\sfrac{1}{2})y^\intercal G_k^{-1}y.
$
Hence, $\nabla \psi_k^*(y) = G_k^{-1}y$. Using the definitions of $\nabla \psi_k(x)$ and $\nabla \psi_k^*(y)$, we find that the mirror descent update is
\begin{align}
x_{k+1}=& G_k^{-1} \left ( G_k x_k - \alpha_k \nabla f(x_k) \right )= x_k -\alpha_k G_k^{-1}\nabla f(x_k),
\end{align}
which is identical to \eqref{eq:NG}.{\hfill$\blacksquare$}\end{proof}

Although researchers often use $\psi_k$ that are norms like the $p$-norm and Mahalanobis norm, notice that the $\psi_k$ that results in natural gradient descent is \emph{not} a norm. Also, since $G_k$ depends on $k$, $\psi_k$ is an \emph{adaptive} proximal function \cite{Duchi2010}.

\section{Projected Natural Gradients}
When $x$ is constrained to some set, $X$, $\psi_k$ in mirror descent is augmented with the indicator function $I_X$, where $I_X(x)=0$ if $x \in X$, and $+\infty$ otherwise. The $\psi_k$ that was shown to generate an update equivalent to the natural gradient descent update, with the added constraint that $x \in X$, is $\psi_k(x) = (\sfrac{1}{2})x^\intercal G_kx + I_X(x).$ Hereafter, any references to $\psi_k$ refer to this augmented version.

For this proximal function, the subdifferential of $\psi_k(x)$ is
$
\nabla \psi_k(x)=G_k(x) + \hat N_X(x)= (G_k+\hat N_X)(x),
$
where $\hat N_X(x) \triangleq \partial I_X(x)$ and, in the middle term, $G_k$ and $\hat N_X$ are relations and $+$ denotes Minkowski addition.\footnote{Later, we abuse notation and switch freely between treating $G_k$ as a matrix and a relation. When it is a matrix, $G_kx$ denotes matrix-vector multiplication that produces a vector. When it is a relation, $G_k(x)$ produces the singleton $\{G_kx\}$.} $\hat N_X(x)$ is the normal cone of $X$ at $x$ if $x \in X$ and $\emptyset$ otherwise \cite{Rockafellar1970}. 
\begin{align}
\nabla \psi_k^*(y) = (G_k+ \hat N_X)^{-1}(y).
\label{eq:lkjhasdkltjt}
\end{align}

Let $\Pi_X^{G_k}(y)$, be the set of $x \in X$ that are closest to $y$, where the length of a vector, $z$, is $(\sfrac{1}{2})z^\intercal G_kz$. More formally,
\begin{align}
\Pi_X^{G_k}(y) \triangleq \arg \min_{x \in X} \frac{1}{2}(y-x)^\intercal G_k (y-x).
\label{eq:blah}
\end{align}

\begin{lemma}
\label{lemma:thingy}
$\Pi_X^{G_k}(y) = (G_k+\hat N_X)^{-1}(G_ky)$.
\end{lemma}
\begin{proof}
We write \eqref{eq:blah} without the explicit constraint that $x \in X$ by appending the indicator function:
\begin{align}
\Pi_X^{G_k}(y)
=& \arg \min_{x \in \mathbb R^n} h_y(x),
\end{align}
where $h_y(x)=(\sfrac{1}{2}) (y-x)^\intercal G_k(y-x) + I_X(x)$. Since $h_y$ is strictly convex over $X$ and $+\infty$ elsewhere, its critical point is its global minimizer. The critical point satisfies
\begin{align}
0 \in \nabla h_y(x) =-G_k(y)+G_k(x)+\hat N_X(x).
\end{align}
The globally minimizing $x$ therefore satisfies
$
G_ky \in G_k(x) + \hat N_X(x)= (G_k+\hat N_X)(x).
$
Solving for $x$, we find that $x = (G_k+\hat N_X)^{-1}(G_ky)$.{\hfill$\blacksquare$}\end{proof}

Combining Lemma \ref{lemma:thingy} 
 with \eqref{eq:lkjhasdkltjt}, we find that
$
\nabla \psi^*(y)=\Pi_X^{G_k}(G_k^{-1}y).
$
Hence, mirror descent with the proximal function that produces natural gradient descent, augmented to include the constraint that $x \in X$, is:
\begin{align}
x_{k+1} =& \Pi_X^{G_k} \left (G_k^{-1} \left ( (G_k+\hat N_X)(x_k) - \alpha_k \nabla f(x_k) \right ) \right )\\
=&  \Pi_X^{G_k} \left ( (I+G_k^{-1}\hat N_X)(x_k) - \alpha_k G_k^{-1} \nabla f(x_k) \right ),
\end{align}
where $I$ denotes the identity relation. Since $x_k \in X$, we know that $0 \in \hat N_X(x_k)$, and hence the update can be written as
\begin{align}
x_{k+1} = \Pi_X^{G_k} \left ( x_k - \alpha_k G_k^{-1} \nabla f(x_k) \right ),
\label{eq:PNG}
\end{align}
which we call \emph{projected natural gradient} (PNG).

\section{Compatibility of Projection}
\label{sec:CP}
The standard projected subgradient (PSG) descent method follows the negative gradient (as opposed to the negative natural gradient) and projects back to $X$ using the Euclidean norm. If $f$ and $X$ are convex and the stepsize is decayed appropriately, it is guaranteed to converge to a global minimum, $x^* \in X$. Any such $x^*$ is a fixed point. This means that a small step in the negative direction of any subdifferential of $f$ at $x^*$ will project back to $x^*$.

Our choice of projection, $\Pi_X^{G_k}$, results in PNG having the same fixed points (see Lemma \ref{lemma:compatibleProjection}). This means that, when the algorithm is at $x^*$ and a small step is taken down the natural gradient to $x'$, $\Pi_X^{G_k}$ will project $x'$ back to $x^*$. We therefore say that $\Pi_X^{G_k}$ is compatible with the natural gradient. For comparison, the Euclidean projection of $x'$ will \emph{not} necessarily return $x'$ to $x^*$. 

\begin{lemma}
\label{lemma:compatibleProjection}
The sets of fixed points for PSG and PNG are equivalent.
\end{lemma}
\begin{proof}
A necessary and sufficient condition for $x$ to be a fixed point of  PSG is that $-\nabla f(x) \in \hat N_X(x)$ \cite{Nocedal2006}. 
A necessary and sufficient condition for $x$ to be a fixed point of PNG is
\begin{align}
x =& \Pi_X^{G_k} \left ( x - \alpha_k G_k^{-1} \nabla f(x) \right )= (G_k+\hat N_X)^{-1} \Big ( G_k \left ( x - \alpha_k G_k^{-1} \nabla f(x) \right ) \Big )\\
=& (G_k+\hat N_X)^{-1} \left ( G_kx - \alpha_k \nabla f(x) \right )\\
\Leftrightarrow &G_kx - \alpha_k \nabla f(x) \in  G_k(x) + \hat N_X(x) \\
\Leftrightarrow& -\nabla f(x) \in \hat N_X(x).&\mbox{\hfill$\blacksquare$} 
\end{align}
\end{proof}


To emphasize the importance of using a compatible projection, consider the following simple example. Minimize the function $f(x) = x^\intercal A x + b^\intercal x$, where $A = \mbox{diag}(1,0.01)$ and $b=[-0.2,-0.1]^\intercal$, subject to the constraints $\lVert x \rVert_1 \leq 1$ and $x \geq 0$. We implemented three algorithms, and ran each for $1000$ iterations using a fixed stepsize: 
\begin{enumerate}
\vspace{-.2cm}\item {\bf PSG} - projected subgradient descent using the Euclidean projection.
\vspace{-.1cm}\item {\bf PNG} - projected natural gradient descent using $\Pi_X^{G_k}$.
\vspace{-.1cm}\item {\bf PNG-Euclid} - projected natural gradient descent using the Euclidean projection.
\end{enumerate}
\vspace{-.2cm}The results are shown in Figure 1
. Notice that PNG and PSG converge to the optimal solution, $x^*$. From this point, they both step in different directions, but project back to $x^*$. However, PNG-Euclid converges to a suboptimal solution (outside the domain of the figure). If $X$ were a line segment between the point that PNG-Euclid and PNG converge to, then PNG-Euclid would converge to the pessimal solution within $X$, while PSG and PNG would converge to the optimal solution within $X$. Also, notice that the natural gradient corrects for the curvature of the function and heads directly towards the global unconstrained minimum. Since the natural methods in this example use metric tensor $G=A$, which is the Hessian of $f$, they are essentially an incremental form of Newton's method. In practice, the Hessian is usually not known, and an estimate thereof is used.

\begin{SCfigure}
\centering
\includegraphics[width=0.4\columnwidth, height=0.15\textheight]{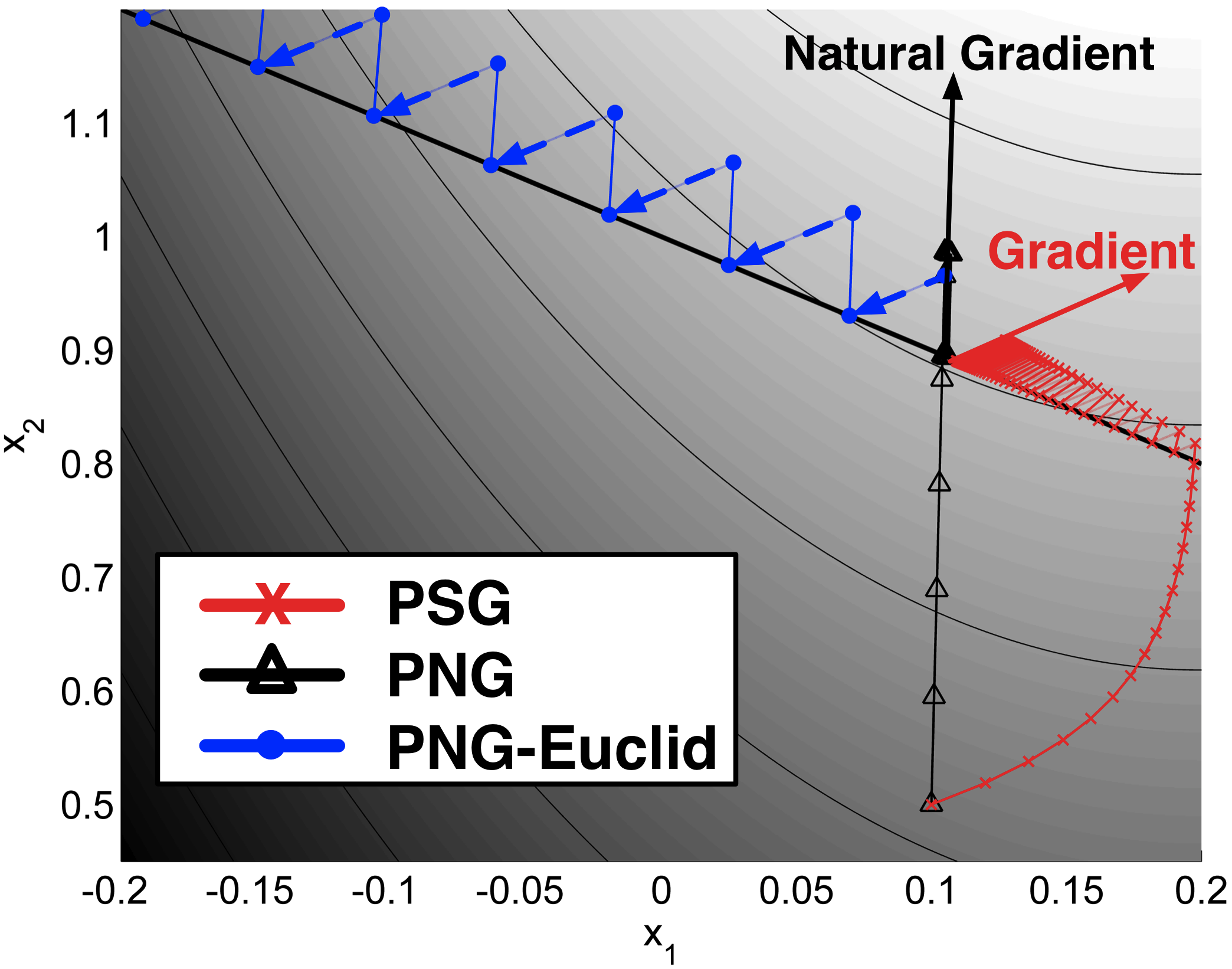}
\label{fig:example}
\caption{The thick diagonal line shows one constraint and dotted lines show projections. Solid arrows show the directions of the natural gradient and gradient at the optimal solution, $x^*$. The dashed blue arrows show PNG-Euclid's projections, and emphasize the the projections cause PNG-Euclid to move \emph{away} from the optimal solution. \vspace{0.15cm}}
\vspace{-.3cm}
\end{SCfigure}

\section{Natural Actor-Critic Algorithms}
\label{sec:NACA}
An MDP is a tuple $M=(\mathcal S, \mathcal A, \mathcal P, \mathcal R, d_0, \gamma)$, where $\mathcal S$ is a set of states, $\mathcal A$ is a set of actions, $\mathcal P(s'|s,a)$ gives the probability density of the system entering state $s'$ when action $a$ is taken in state $s$, $R(s,a)$ is the expected reward, $r$, when action $a$ is taken in state $s$, $d_0$ is the initial state distribution, and $\gamma \in [0,1)$ is a reward discount parameter. A parameterized policy, $\pi$, is a conditional probability density function---$\pi(a|s,\theta)$ is the probability density of action $a$ in state $s$ given a vector of policy parameters, $\theta \in \mathbb R^n$.

Let $J(\theta) = \mbox{E} \left [ \sum_{t=0}^\infty \gamma^t r_t | \theta \right ]$ be the \emph{discounted-reward objective} or the \emph{average reward objective function} with $J(\theta) = \lim_{n\to\infty} \frac{1}{n}\mbox{E} \left [ \sum_{t=0}^n r_t | \theta \right ]$. Given an MDP, $M$, and a parameterized policy, $\pi$, the goal is to find policy parameters that maximize one of these objectives. When the action set is continuous, the search for globally optimal policy parameters becomes intractable, so policy search algorithms typically search for locally optimal policy parameters.

Natural actor-critics, first proposed by \citet{Kakade2002}, are algorithms that estimate and ascend the natural gradient of $J(\theta)$, using the average Fisher information matrix as the metric tensor:
\vspace{-.1cm}\begin{equation}
G_k=G(\theta_k)=\mbox{E}_{s \sim d^\pi,a \sim \pi} \left [ \left (\frac{\partial}{\partial \theta_k} \log \pi(a|s,\theta_k)\right )\left(\frac{\partial}{\partial \theta_k} \log \pi(a|s,\theta_k)\right ) ^ \intercal \right ],
\end{equation}
where $d^\pi$ is a policy and objective function-dependent distribution over the state set \cite{Sutton2000}. 

There are many natural actor-critics, including Natural policy gradient utilizing the Temporal Differences (NTD) algorithm \cite{Morimura2005}, Natural Actor-Critic using LSTD-Q$(\lambda)$ (NAC-LSTD) \cite{Peters2008}, Episodic Natural Actor-Critic (eNAC) \cite{Peters2008}, Natural Actor-Critic using Sarsa$(\lambda)$ (NAC-Sarsa) \cite{Thomas2012}, Incremental Natural Actor-Critic (INAC) \cite{Degris2012}, and Natural-Gradient Actor-Critic with Advantage Parameters (NGAC) \cite{Bhatnagar2009}. All of them form an estimate, typically denoted $w_k$, of the natural gradient of $J(\theta_k)$. That is, $w_k \approx G(\theta_k)^{-1}\nabla J(\theta_k)$. They then perform the policy parameter update, $\theta_{k+1} = \theta_k + \alpha_k w_k.$

\section{Projected Natural Actor-Critics}
If we are given a closed convex set, $\Theta \subseteq \mathbb R^n$, of admissible policy parameters (e.g., the stable region of gains for a PID controller), we may wish to ensure that the policy parameters remain within $\Theta$. The natural actor-critic algorithms described in the previous section do not provide such a guarantee. However, their policy parameter update equations, which are natural gradient ascent updates, can easily be modified to the projected natural gradient ascent update in \eqref{eq:PNG} by projecting the parameters back onto $\Theta$ using $\Pi_\Theta^{G(\theta_k)}$:
\begin{align}
\theta_{k+1} = \Pi_\Theta^{G(\theta_k)} \!\! \left ( \theta_k + \alpha_k w_k \right ).
\end{align}
Many of the existing natural policy gradient algorithms, including NAC-LSTD, eNAC, NAC-Sarsa, and INAC, follow \emph{biased} estimates of the natural policy gradient~\cite{Thomas2012b}. For our experiments, we must use an unbiased algorithm since the projection that we propose is compatible with the natural gradient, but not necessarily biased estimates thereof.

NAC-Sarsa and INAC are equivalent \emph{biased} discounted-reward natural actor-critic algorithms with per-time-step time complexity linear in the number of features. The former was derived by replacing the LSTD-Q$(\lambda)$ component of NAC-LSTD with Sarsa$(\lambda)$, while the latter is the discounted-reward version of NGAC. Both are similar to NTD, which is a biased average-reward algorithm. The \emph{unbiased} \emph{discounted}-reward form of NAC-Sarsa was recently derived \cite{Thomas2012b}. References to NAC-Sarsa hereafter refer to this unbiased variant. In our case studies we use the \emph{projected natural actor-critic using Sarsa$(\lambda)$} (PNAC-Sarsa), the projected version of the unbiased NAC-Sarsa algorithm.

Notice that the projection, $\Pi_\Theta^{G(\theta_k)}$, as defined in \eqref{eq:blah}, is \emph{not} merely the Euclidean projection back onto $\Theta$. For example, if $\Theta$ is the set of $\theta$ that satisfy $A\theta \leq b$, for some fixed matrix $A$ and vector $b$, then the projection, $\Pi_\Theta^{G(\theta_k)}$, of $y$ onto $\Theta$ is a quadratic program,
\begin{align}
\mbox{minimize } f(\theta)=&-y^\intercal G(\theta_k)\theta +\frac{1}{2}\theta^\intercal G(\theta_k)\theta, \hspace{1cm} \mbox{s.t. }A\theta \leq b.
\end{align}

\vspace{-.1cm}In order to perform this projection, we require an estimate of the average Fisher information matrix, $G(\theta_k)$. If the natural actor-critic algorithm does not already include this (like NAC-LSTD and NAC-Sarsa do not), then an estimate can be generated by selecting $G_0 = \beta I$, where $\beta$ is a positive scalar and $I$ is the identity matrix, and then updating the estimate with
\begin{align}
G_{t+1} = (1-\mu_t) G_t + \mu_t \left ( \frac{\partial}{\partial \theta_k} \log \pi(a_t|s_t,\theta_k) \right ) \left ( \frac{\partial}{\partial \theta_k} \log \pi(a_t|s_t,\theta_k) \right )^\intercal,
\end{align}
where $\{\mu_t\}$ is a stepsize schedule \cite{Bhatnagar2009}. Notice that we use $t$ and $k$ subscripts since many time steps of the MDP may pass between updates to the policy parameters.


\section{Case Study: Functional Electrical Stimulation}
In this case study, we searched for proportional-derivative (PD) gains to control a simulated human arm undergoing FES. We used the Dynamic Arm Simulator 1 (DAS1) \cite{Blana2009}, a detailed biomechanical simulation of a human arm undergoing functional electrical stimulation. In a previous study, a controller created using DAS1 performed well on an actual human subject undergoing FES, although it required some additional tuning in order to cope with biceps spasticity \cite{KathyPD}. This suggests that it is a reasonably accurate model of an ideal arm.

The DAS1 model, depicted in Figure 2a, has state $s_t=(\phi_1, \phi_2, \dot \phi_1, \dot \phi_2, \phi_1^{target}, \phi_2^{target})$, where $\phi_1^{target}$ and $\phi_2^{target}$ are the desired joint angles, and the desired joint angle velocities are zero. The goal is to, during a two-second episode, move the arm from its random initial state to a randomly chosen stationary target. The arm is controlled by providing a stimulation in the interval $[0,1]$ to each of six muscles. The reward function used was similar to that of \citet{KathyPD}, which punishes joint angle error and high muscle stimulation. We searched for locally optimal PD gains using PNAC-Sarsa where the policy was a PD controller with Gaussian noise added for exploration.

Although DAS1 does not model shoulder dislocation, we added safety constraints by limiting the $l_1$-norm of certain pairs of gains. The constraints were selected to limit the forces applied to the humerus. These constraints can be expressed in the form $A \theta \leq b$, where $A$ is a matrix, $b$ is a vector, and $\theta$ are the PD gains (policy parameters). We compared the performance of three algorithms:
\begin{enumerate}
\vspace{-.2cm}\item {\bf NAC}: NAC-Sarsa with no constraints on $\theta$.
\vspace{-.1cm}\item {\bf PNAC}: PNAC-Sarsa using the compatible projection, $\Pi_\Theta^{G(\theta_k)}$.
\vspace{-.1cm}\item {\bf PNAC-E}: PNAC-Sarsa using the Euclidean projection.
\end{enumerate}
\vspace{-.2cm}
Since we are not promoting the use of one natural actor-critic over another, we did not focus on finely tuning the natural actor-critic nor comparing the learning speeds of different natural actor-critics. Rather, we show the importance of the proper projection by allowing PNAC-Sarsa to run for a million episodes (far longer than required for convergence), after which we plot the mean sum of rewards during the last quarter million episodes. Each algorithm was run ten times, and the results averaged and plotted in Figure 2b. Notice that PNAC performs worse than the unconstrained NAC. This happens because NAC leaves the safe region of policy space during its search, and converges to a dangerous policy---one that reaches the goal quickly and with low total muscle force, but which can cause large, short, spikes in muscle forces surrounding the shoulder, which violates our safety constraints. We suspect that PNAC converges to a near-optimal policy within the region of policy space that we have designated as safe. PNAC-E converges to a policy that is worse than that found by PNAC because it uses an incompatible projection.

\begin{figure}[]
	\centering
	\begin{subfigure}[t]{0.475\textwidth}
		\centering
		\includegraphics[width=.7\textwidth]{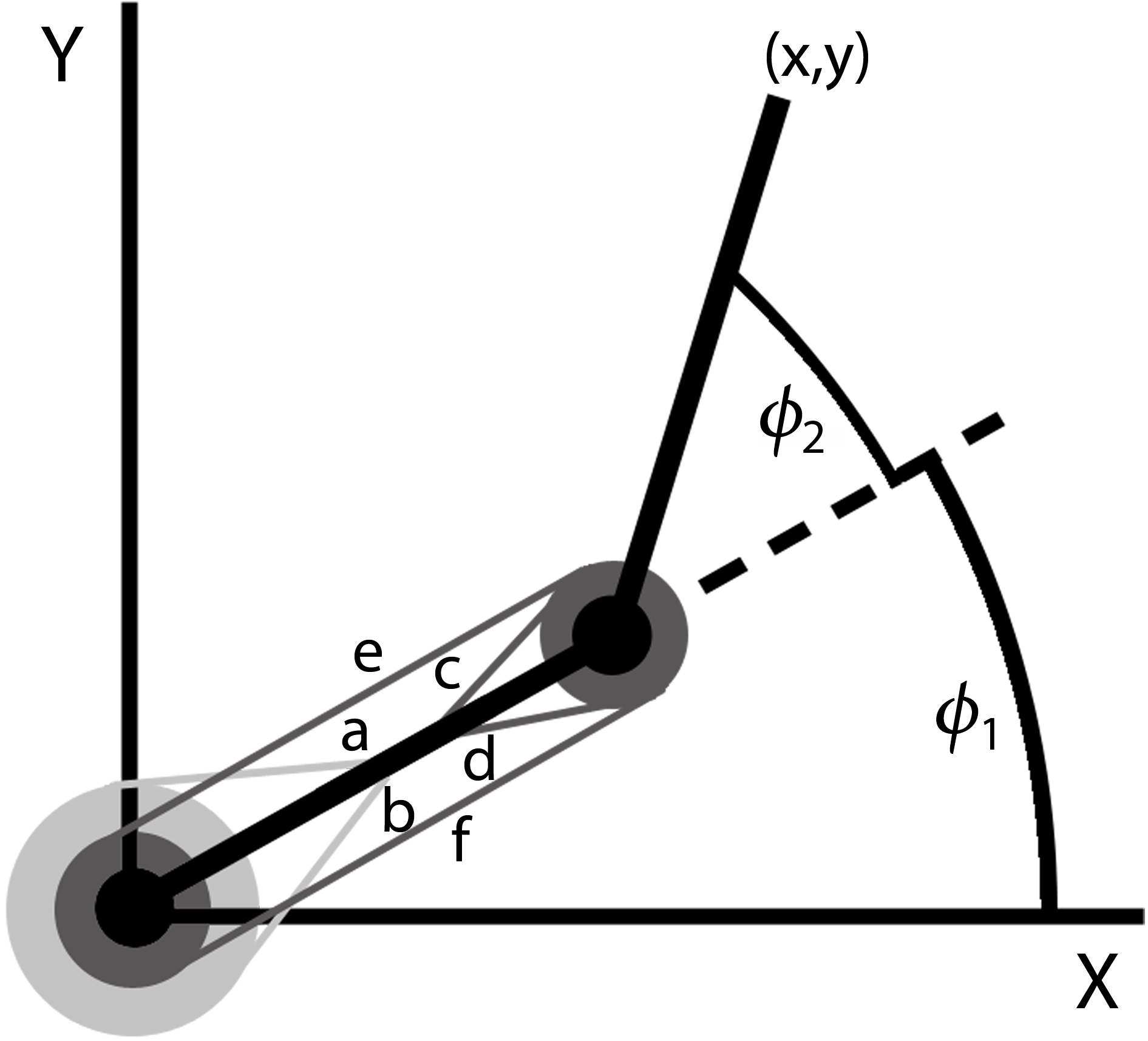}
		\caption*{(Figure 2a) DAS1, the two-joint, six-muscle biomechanical model used. Antagonistic muscle pairs are as follows, listed as (flexor, extensor): monoarticular shoulder muscles (a: anterior deltoid, b: posterior deltoid); monoarticular elbow muscles (c: brachialis, d: triceps brachii (short head)); biarticular muscles (e: biceps brachii, f: triceps brachii (long head)).}
		\label{fig:DAS1}
	\end{subfigure}
	\hspace{.03\textwidth}
	\begin{subfigure}[t]{0.475\textwidth}
		\centering
		\includegraphics[width=1\textwidth]{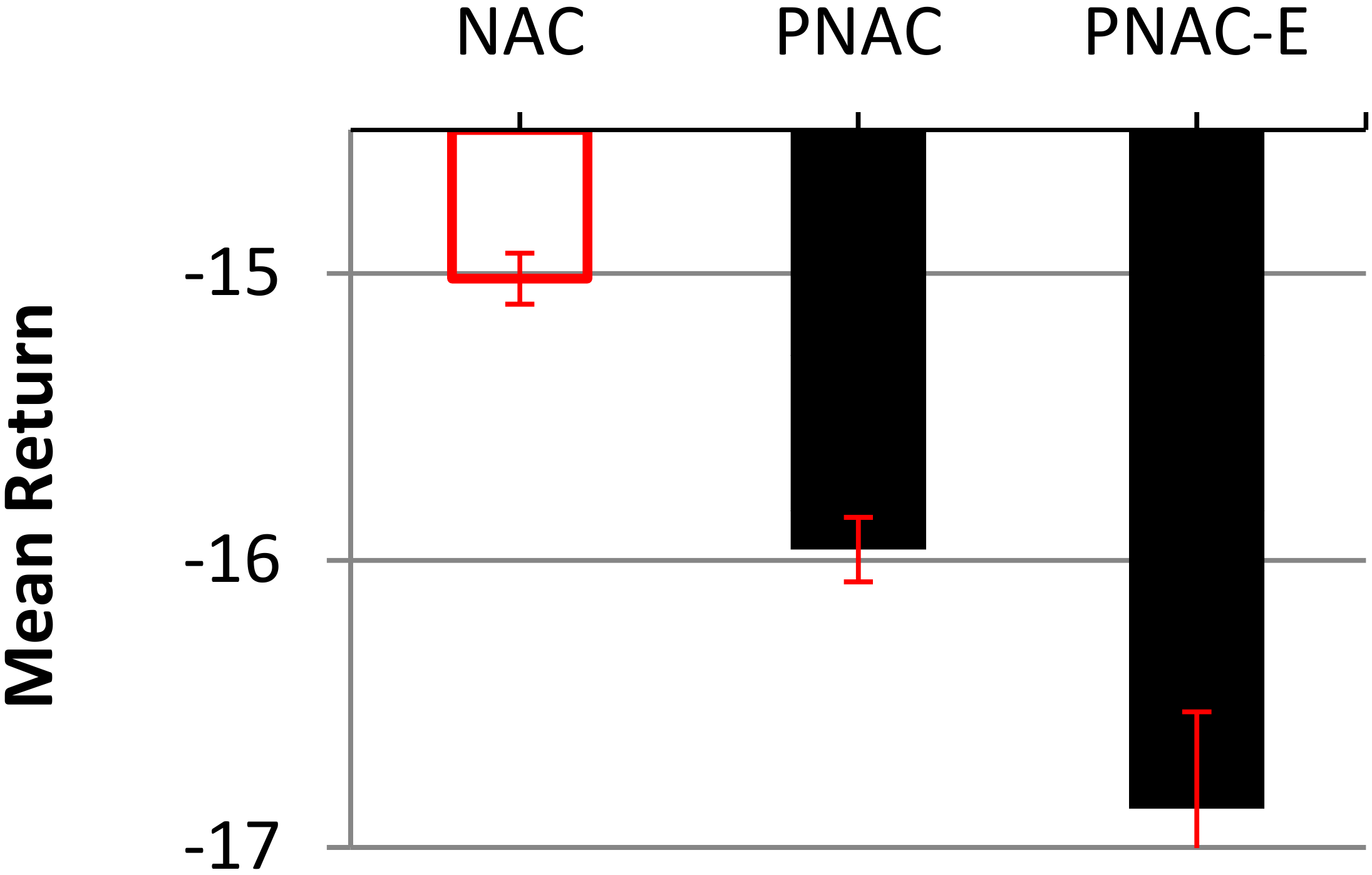}
		\caption*{(Figure 2b) Mean return during the last 250,000 episodes of training using thee algorithms. Standard deviation error bars from the 10 trials are provided. The NAC bar is red to emphasize that the final policy found by NAC resides in the dangerous region of policy space.}
		\label{fig:DAS1Data}
	\end{subfigure}
\vspace{-.4cm}
\end{figure}

\section{Case Study: uBot Balancing}
In the previous case study, the optimal policy lay outside the designated safe region of policy space (this is common when a single failure is so costly that adding a penalty to the reward function for failure is impractical, since a single failure is unacceptable). We present a second case study in which the optimal policy lies within the designated safe region of policy space, but where an unconstrained search algorithm may enter the unsafe region during its search of policy space (at which point large negative rewards return it to the safe region).

 The uBot-5, shown in Figure \ref{fig:uBot}, is an 11-DoF mobile manipulator developed at the University of Massachusetts Amherst \cite{Deegan2010,Kuindersma2009}. During experiments, it often uses its arms to interact with the world. Here, we consider the problem faced by the controller tasked with keeping the robot balanced during such experiments. To allow for results that are easy to visualize in 2D, we use a PD controller that observes only the current body angle, its time derivative, and the target angle (always vertical). This results in the PD controller having only two gains (tunable policy parameters). We use a crude simulation of the uBot-5 with random upper-body movements, and search for the PD gains that minimize a weighted combination of the energy used and the mean angle error (distance from vertical).

We constructed a set of conservative estimates of the region of stable gains, with which the uBot-5 should never fall, and used PNAC-Sarsa and NAC-Sarsa to search for the optimal gains. Each training episode lasted 20 seconds, but was terminated early (with a large penalty) if the uBot-5 fell over. Figure \ref{fig:uBot} (middle) shows performance over 100 training episodes. Using NAC-Sarsa, the PD weights often left the conservative estimate of the safe region, which resulted in the uBot-5 falling over. Figure \ref{fig:uBot} (right) shows one trial where the uBot-5 fell over four times (circled in red). The resulting large punishments cause NAC-Sarsa to quickly return to the safe region of policy space. Using PNAC-Sarsa, the simulated uBot-5 never fell. Both algorithms converge to gains that reside within the safe region of policy space. We selected this example because it shows how, even if the optimal solution resides within the safe region of policy space (unlike the in the previous case study), unconstrained RL algorithms may traverse unsafe regions of policy space during their search.

\begin{figure}
\raisebox{0.8cm}{\includegraphics[width=.18\columnwidth]{uBot.jpg}}
\hspace{.1cm}
 \includegraphics[width=.4\columnwidth]{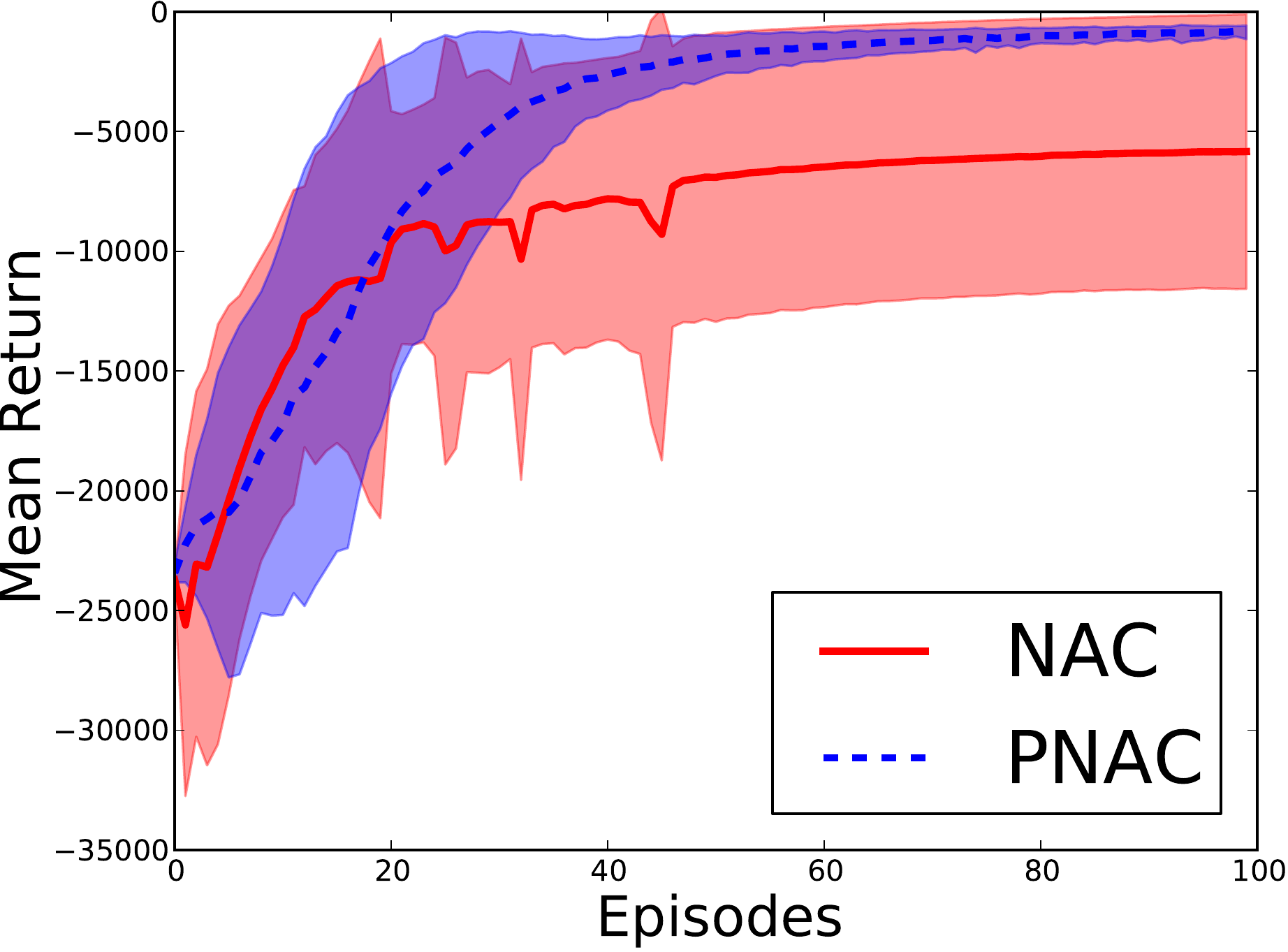}
\hspace{-.1cm}
\raisebox{-.2cm}{\includegraphics[width=.3\columnwidth]{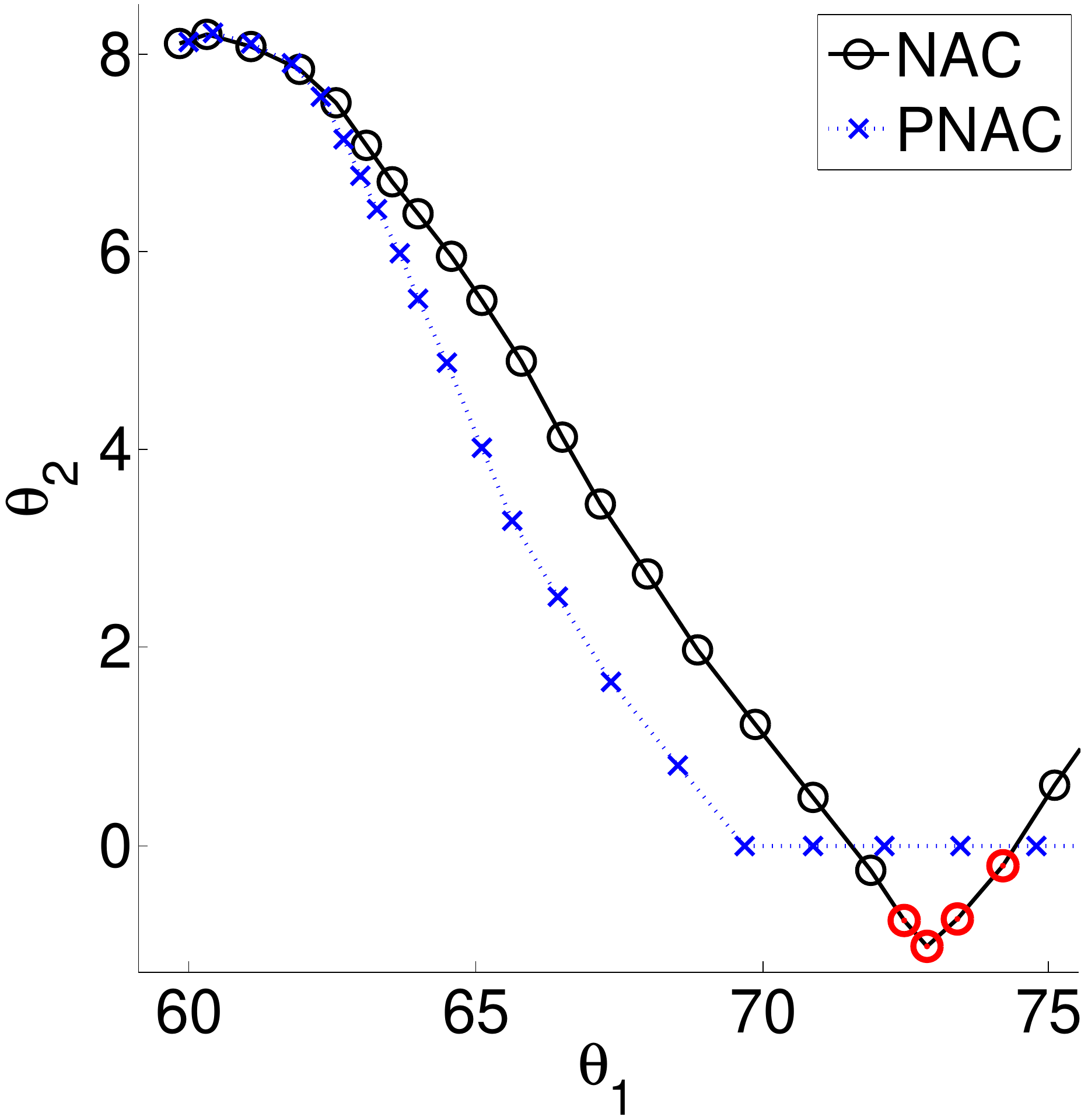}} \hfill
\caption{{\bf Left:} uBot-5 holding a ball. {\bf Middle:} Mean (over 20-trials) returns over time using PNAC-Sarsa and NAC-Sarsa on the simulated uBot-5 balancing task. The shaded region depicts standard deviations. {\bf Right:} Trace of the two PD gains, $\theta_1$ and $\theta_2$, from a typical run of PNAC-Sarsa and NAC-Sarsa. A marker is placed for the gains after each episode, and red markers denote episodes where the simulated uBot-5 fell over.}
\label{fig:uBot}
\end{figure}


%
%
%
%

\section{Summary}

We presented a class of algorithms, which we call \emph{projected natural actor-critics} (PNACs). PNACs are the simple modification of existing natural actor-critic algorithms to include a projection of newly computed policy parameters back onto an allowed set of policy parameters (e.g., those of policies that are known to be safe). We argued that a principled projection is the one that results from viewing natural gradient descent, which is an \emph{unconstrained} algorithm, as a special case of mirror descent, which is a \emph{constrained} algorithm.

We show that the resulting projection is compatible with the natural gradient and gave a simple empirical example that shows why a compatible projection is important. This example also shows how an incompatible projection can result in natural gradient descent converging to a pessimal solution in situations where a compatible projection results in convergence to an optimal solution. We then applied a PNAC algorithm to a realistic constrained control problem with six-dimensional continuous states and actions. Our results support our claim that the use of an incompatible projection can result in convergence to inferior policies. Finally, we applied PNAC to a simulated robot and showed its substantial benefits over unconstrained natural actor-critic algorithms.


\chapter{True Stochastic Gradient Temporal Difference Learning Algorithms}
\label{gtd-chapter}

We now turn to the solution of a longstanding puzzle: how to design a ``true" gradient method for reinforcement learning?
We address long-standing questions in reinforcement learning: (1) Are  there any first-order reinforcement
learning algorithms that can be viewed as ``true" stochastic gradient methods?
If there are, what are their objective functions and what are their convergence rates?
(2) What is the general framework for avoiding biased sampling (instead of double-sampling, which is a stringent sampling requirement)
in reinforcement learning? To this end, we introduce a novel primal-dual
splitting framework for reinforcement learning, which shows that the GTD family of algorithms are
true stochastic algorithms with respect to the primal-dual formulation
of the objective functions such as NEU and MSPBE, which facilitates their
convergence rate analysis and regularization. We also propose  operator splitting
as a unified framework to avoid bias sampling in reinforcement learning.
We  present an illustrative empirical study on simple canonical problems
validating the effectiveness of the proposed algorithms compared with previous
approaches.

\section{Introduction}

First-order temporal difference (TD) learning is a widely used class of techniques
in reinforcement learning. Although least-squares based
temporal difference approaches, such as LSTD \cite{bradtke-barto:LSTD},
LSPE \cite{bertsekas-nedic} and LSPI \cite{lagoudakis:jmlr} perform
well with moderate size problems, first-order temporal difference
learning algorithms scale more gracefully to high dimensional problems.
The initial class of TD methods was known to converge only when samples are drawn
``on-policy". This motivated the development of the gradient TD (GTD)  family of methods \cite{FastGradient:2009}.  A novel saddle-point framework for sparse regularized GTD was proposed recently \cite{ROTD:NIPS2012}. However, there have been several questions
regarding the current off-policy TD algorithms.
(1) The first is the convergence rate of these algorithms. Although
these algorithms are motivated from the gradient of an objective function
such as MSPBE and NEU, they are not true stochastic gradient methods
with respect to these objective functions, as pointed out in \cite{szepesvari2010algorithms}, which make the convergence rate and error bound analysis difficult, although asymptotic analysis
has been carried out using the ODE approach. (2) The second concern is
regarding acceleration. It is believed that TDC performs the best so far of
the GTD family of algorithms. One may intuitively ask if there are any gradient TD
algorithms that can outperform TDC. (3) The third concern is regarding compactness
of the feasible set $\theta$. The GTD family of algorithms all assume
that the feasible set $\theta$ is unbounded, and if the feasible
set $\theta$ is compact, there is no theoretical analysis and convergence
guarantee. (4) The fourth question is on regularization: although
the saddle point framework proposed in \cite{ROTD:NIPS2012} provides an online regularization framework
for the GTD family of algorithms, termed as RO-TD, it is based on the
inverse problem formulation and is thus not quite explicit. One further
question is whether there is a more straightforward algorithm, e.g,
the regularization is directly based on the MSPBE and NEU objective
functions.

Biased sampling is a well-known problem in reinforcement learning.
Biased sampling is caused by the stochasticity of the policy wherein
there are multiple possible  successor states from the current state
where the agent is. If it is a deterministic policy, then there will
be no biased sampling problem. Biased sampling is often caused
by the product of the TD errors, or the product of TD error and the
gradient of TD error w.r.t the model parameter $\theta$. There are
two ways to avoid the biased sampling problem, which can be categorized
into double sampling methods and two-time-scale stochastic approximation
methods.

In this paper, we  propose a novel approach to TD algorithm design in reinforcement learning, based on
introducing the  {\em proximal splitting}  framework \cite{PROXSPLITTING2011}. We show that the GTD family of algorithms are true stochastic gradient descent (SGD) methods, thus making their convergence
rate analysis available. New accelerated off-policy algorithms are
proposed and their comparative study with RO-TD is carried out to show
the effectiveness of the proposed algorithms. We also show that primal-dual
splitting is a unified first-order optimization framework to solve the
biased sampling problem.

Here is a roadmap to the rest of the chapter. Section 2 reviews reinforcement
learning and the basics of proximal splitting formulations and algorithms.
Section 3 introduces a novel problem formulation which we investigate
in this paper. Section 4 proposes a series of new algorithms, demonstrates
the connection with the GTD algorithm family, and also presents accelerated
algorithms. Section 5 presents theoretical analysis of the algorithms.
Finally, empirical results are presented in Section 6 which validate
the effectiveness of the proposed algorithmic framework. Abbreviated technical proofs
of the main theoretical results are provided in a supplementary appendix.

\section{Background}

\subsection{Markov Decision Process and Reinforcement Learning}


In linear value function approximation,
a value function is assumed to lie in the linear span of a basis function
matrix $\Phi$ of dimension $\left|S\right|\times d$, where $d$
is the number of linear independent features. Hence, $V\approx V_{\theta}=\Phi\theta$.
For the $t$-th sample, $\phi_{t}$ (the $t$-th row of $\Phi$),
$\phi'_{t}$ (the $t$-th row of $\Phi'$) are the feature vectors
corresponding to $s_{t},s'_{t}$, respectively. $\theta_{t}$ is the
weight vector for $t$-th sample in first-order TD learning methods,
and ${\delta_{t}}=({r_{t}}+\gamma\phi_{t}^{'T}{\theta_{t}})-\phi_{t}^{T}{\theta_{t}}$
is the temporal difference error. TD learning uses the following update
rule ${\theta_{t+1}}={\theta_{t}}+{\alpha_{t}}{\delta_{t}}{\phi_{t}}$,
where $\alpha_{t}$ is the stepsize. However, TD is only guaranteed
to converge in the on-policy setting, although in many off-policy
situations, it still has satisfactory performance \cite{Kolter:offpolicyTD}.
To this end, Sutton et al. proposed a family of off-policy convergent
algorithms including GTD, GTD2 and TD with gradient correction (TDC).
GTD is a two-time-scale stochastic approximation approach which aims
to minimize the norm of the expected TD update (NEU), which is defined
as
\begin{equation}
{\rm {NEU}}(\theta)=\mathbb{E}{[\delta_{t}(\theta)\phi_{t}]^{T}}\mathbb{E}[\delta_{t}(\theta)\phi_{t}].\label{eq:neu}
\end{equation}
TDC \cite{FastGradient:2009} aims to minimize the mean-square projected
Bellman error (MSPBE) with a similar two-time-scale technique, which
is defined as ${\rm {MSPBE}}(\theta) =$
\begin{equation}
\left\Vert {\Phi\theta-\Pi T(\Phi\theta)}\right\Vert _{\Xi}^{2}={({\Phi^{T}}\Xi(T\Phi\theta-\Phi\theta))^{T}}{({\Phi^{T}}\Xi\Phi)^{-1}}{\Phi^{T}}\Xi(T\Phi\theta-\Phi\theta), \label{eq:mspbe-1}
\end{equation}
where $\Xi$ is a diagonal matrix whose entries $\xi(s)$ are given
by a positive probability distribution over states.

\section{Problem Formulation}

Biased sampling is a well-known problem in reinforcement learning.
Biased sampling is caused by $\mathbb{E}[{\phi_{t}^{'T}}{\phi_{t}^{'}}]$
or $\mathbb{E}[{\phi_{t}^{'}}{\phi_{t}^{'T}}]$, where $\phi_{t}^{'}$
is the feature vector for state $s_{t}^{'}$ in sample $(s_{t},a_{t},r_{t},{s_{t}^{'}})$.
Due to the stochastic nature of the policy, there may be many $s_{t}'$
w.r.t the same $s_{t}$, thus $\mathbb{E}[{\phi_{t}^{'T}}{\phi_{t}^{'}}]$
or $\mathbb{E}[{\phi_{t}^{'}}{\phi_{t}^{'T}}]$ cannot be consistently
estimated via a single sample. This problem hinders the objective
functions to be solved via stochastic gradient descent (SGD) algorithms.
As pointed out in \cite{szepesvari2010algorithms}, although many
algorithms are motivated by well-defined convex objective functions
such as MSPBE and NEU, due to the biased sampling problem, the unbiased
stochastic gradient is impossible to obtain, and thus the algorithms
are not true SGD methods w.r.t. these objective functions.
The biased sampling is often caused by the product of the TD errors,
or the product of TD error and the derivative of TD error w.r.t. the
parameter $\theta$. There are two ways to avoid the biased sampling
problem, which can be categorized into double sampling methods and
stochastic approximation methods. Double sampling, which samples both
$s'$ and $s''$ and thus requires computing $\phi'$ and $\phi''$,
is possible in batch reinforcement learning, but is  usually impractical
in online reinforcement learning. The other approach is
stochastic approximation, which introduces a new variable to
estimate the part containing $\phi_{t}^{'}$, thus avoiding the product
of $\phi_{t}^{'}$ and $\phi_{t}^{''}$. Consider, for example, the
NEU objective function in Section (\ref{eq:neu}). Taking the gradient
w.r.t. $\theta$, we have
\begin{equation}
\label{eq:neu-grad}
-\frac{1}{2}{\rm {NEU}}(\theta)=\mathbb{E}[(\phi_{t}-\gamma\phi_{t}'){\phi_{t}^{T}}]\mathbb{E}[\delta_{t}(\theta)\phi_{t}]
\end{equation}
If the gradient can be written as a single expectation value, then
it is straightforward to use a stochastic gradient method, however,
here we have a product of two expectations, and due to the correlation
between $(\phi_{t}-\gamma\phi_{t}'){\phi_{t}^{T}}$ and $\delta_{t}(\theta)\phi_{t}$,
the sampled product is not an unbiased estimate of the gradient. In
other words, $\mathbb{E}[(\phi_{t}-\gamma\phi_{t}'){\phi_{t}^{T}}]$
and $\mathbb{E}[\delta_{t}(\theta)\phi_{t}]$ can be directly sampled,
yet $\mathbb{E}[(\phi_{t}-\gamma\phi_{t}'){\phi_{t}^{T}}]\mathbb{E}[\delta_{t}(\theta)\phi_{t}]$
can not be directly sampled. To tackle this, the GTD algorithm uses
the two-time-scale stochastic approximation method by introducing
an auxiliary variable $w_{t}$, and thus the method is not a true
stochastic gradient method w.r.t. ${\rm {NEU}}(\theta)$ any more.
This auxiliary variable technique is also used in \cite{ZHIWEI2014}.

The other problem for first-order reinforcement learning algorithms
is that it is difficult to define the objective functions, which is
also caused by the biased sampling problem. As pointed out in \cite{szepesvari2010algorithms},
although the GTD family of algorithms are derived from the gradient w.r.t.
the objective functions such as MSPBE and NEU, because of the biased-sampling
problem, these algorithms cannot be formulated directly as SGD methods
w.r.t. these objective functions.

In sum, due to biased sampling, the RL objective functions cannot
be solved via a stochastic gradient method, and it is also difficult
to find objective functions of existing first-order reinforcement learning
algorithms. Thus, there remains a large gap between first-order reinforcement
learning algorithms and stochastic optimization, which we now show how to bridge.

\section{Algorithm Design}

In what follows, we build on the operator splitting methods introduced in Section~\ref{pd-splitting}, which should be reviewed before reading
the section below.

\subsection{NEU Objective Function}

The primal-dual formulation of the NEU defined in Section (\ref{eq:neu})
is as follows:
\begin{equation}
\mathop{\min}\limits _{\theta\in X}\left({\frac{1}{2}{\rm {NEU}}(\theta)+h(\theta)}\right)=\mathop{\min}\limits _{\theta\in X}\mathop{\max}\limits _{y}\left({\langle {\Phi^{T}}\Xi(R+\gamma{\Phi^{'}}\theta-\Phi\theta),y\rangle -\frac{1}{2}||y||_{2}^{2}+h(\theta)}\right)\label{eq:lower-neu}
\end{equation}
We have $K(\theta)={\Phi^{T}}\Xi(R+\gamma{\Phi^{'}}\theta-\Phi\theta)$
, and $F(\cdot)=\frac{1}{2}||\cdot||_{2}^{2}$ , thus the Legendre
transform is ${F^{*}}(\cdot)=F(\cdot)=\frac{1}{2}||\cdot||_{2}^{2}$.
Thus the update rule is
\begin{equation}
\begin{array}{l}
{y_{t+1}}={y_{t}}+{\alpha_{t}}({\delta_{t}}{\phi_{t}}-{y_{t}}){\rm {,\,}}
{\theta _{t + 1}} = {\rm{pro}}{{\rm{x}}_{{\alpha _t}h}}\left( {{\theta _t} + {\alpha _t}({\phi _t} - \gamma \phi _t^\prime )(y_t^T{\phi _t})} \right)
\end{array}\label{eq:gtd}
\end{equation}
Note that if $h(\theta)=0$ and $X={\mathbb{R}^{d}}$, then we will
have the GTD algorithm proposed in \cite{Sutton:GTD1:2008}.

\subsection{MSPBE Objective Function }
Based on the definition of MSPBE in Section (\ref{eq:neu}), we
can reformulate MSPBE as
\begin{equation}
{\rm {MSPBE}}(\theta)=||{\Phi^{T}}\Xi(T{V_{\theta}}-{V_{\theta}})||_{{{({\Phi^{T}}\Xi\Phi)}^{-1}}}^{2}\label{eq:mspbe-1-1}
\end{equation}
The gradient of MSPBE is correspondingly computed as
\begin{equation}
-\frac{1}{2}{\rm {MSPBE}}(\theta)=\mathbb{E}[({\phi_{t}}-\gamma{\phi_{t}^{'}})\phi_{t}^{T}]\mathbb{E}{[{\phi_{t}}\phi_{_{t}}^{T}]^{-1}}\mathbb{E}[{\delta_{t}}(\theta){\phi_{t}}]\label{eq:mspbe-grad}
\end{equation}

As opposed to computing the NEU gradient, computing Equation
(\ref{eq:mspbe-grad}) involves computing the inverse matrix $\mathbb{E}{[{\phi_{t}}\phi_{_{t}}^{T}]^{-1}}$,
which imposes extra difficulty. To this end, we propose another primal-dual
splitting formulation with weighted Euclidean norm as follows,
\begin{equation}
\mathop{\min}\limits _{x\in X}\frac{1}{2}||x||_{{M^{-1}}}^{2}=\mathop{\min}\limits _{x\in X}\mathop{\max}\limits _{w}\langle x,w \rangle -\frac{1}{2}||w||_{M}^{2}\label{eq:primdual-mspbe}
\end{equation}
where $M={\Phi^{T}}\Xi\Phi$,
and the dual variable is denoted as $w_{t}$ to differentiate it from
$y_{t}$ used for the NEU objective function. Then we have
\begin{equation}
\mathop{\min}\limits _{\theta\in X}\frac{1}{2}{\rm {MSPBE}}(\theta)+h(\theta)=\mathop{\min}\limits _{\theta\in X}\mathop{\max}\limits _{w}\langle {{\Phi^{T}}\Xi(R+\gamma{\Phi^{'}}\theta-\Phi\theta),w}\rangle -\frac{1}{2}||w||_{M}^{2}+h(\theta)\label{eq:lower-mspbe}
\end{equation}
Note that the nonlinear
convex $F(\cdot)=\frac{1}{2}||\cdot||_{M^{-1}}^{2}$ , and thus the
Legendre transform is ${F^{*}}(\cdot)=\frac{1}{2}||\cdot||_{M}^{2}$.
We can see that by using the primal-dual splitting formulation, computing
the inverse matrix $M^{-1}$ is avoided. Thus the update rule
is as follows:
\begin{equation}
\begin{array}{l}
{w_{t+1}}={w_{t}}+{\alpha_{t}}({\delta_{t}}-\phi_{_{t}}^{T}{w_{t}}){\phi_{t}}{\rm {,\,\mathbf{}}}\mbox{{\ensuremath{\theta_{t+1}}} = pro{\ensuremath{x_{{\alpha_{t}}h}}}\ensuremath{\left({{\theta_{t}}+{\alpha_{t}}({\phi_{t}}-\gamma\phi_{t}^{\prime})(w_{t}^{T}{\phi_{t}})}\right)}}\end{array}\label{eq:gtd2}
\end{equation}
Note that if $h(\theta)=0$ and $X={\mathbb{R}^{d}}$, then we will
have the GTD2 algorithm proposed in \cite{FastGradient:2009}. It
is also worth noting that the TDC algorithm seems not to have an explicit proximal
splitting representation, since it incorporates $w_{t}(\theta)=\mathbb{E}{[{\phi_{t}}\phi_{_{t}}^{T}]^{-1}}\mathbb{E}[{\delta_{t}}(\theta){\phi_{t}}]$
into the update of $\theta_{t}$, a quasi-stationary condition which
is commonly used in two-time-scale stochastic approximation approaches.
 An intuitive answer to the advantage of TDC over GTD2 is that
the TDC update of $\theta_{t}$ can be considered as incorporating
the prior knowledge into the update rule: for a stationary $\theta_{t}$,
if the optimal $w_{t}(\theta_{t})$ (termed as $w_{t}^{*}(\theta_{t})$)
has a closed-form solution or is easy to compute, then incorporating
this $w_{t}^{*}(\theta_{t})$ into the update rule tends to accelerate
the algorithm's convergence performance. For the GTD2 update in Equation
(\ref{eq:gtd2}), note that there is a sum of two terms where $w_{t}$
appears: which are $({\phi_{t}}-\gamma\phi_{t}^{\prime})(w_{t}^{T}{\phi_{t}})={\phi_{t}}(w_{t}^{T}{\phi_{t}})-\gamma\phi_{t}^{\prime}(w_{t}^{T}{\phi_{t}})$. Replacing $w_{t}$ in the first term with $w_{t}^{*}(\theta)=\mathbb{E}{[{\phi_{t}}\phi_{_{t}}^{T}]^{-1}}\mathbb{E}[{\delta_{t}}(\theta){\phi_{t}}]$,
we have the update rule as follows
\begin{equation}
\begin{array}{l}
{w_{t+1}}={w_{t}}+{\alpha_{t}}({\delta_{t}}-\phi_{_{t}}^{T}{w_{t}}){\phi_{t}}
{\;,\;}
\mbox{{\ensuremath{\theta_{t+1}}} = pro{\ensuremath{x_{{\alpha_{t}}h}}}\ensuremath{\left({{\theta_{t}}+{\alpha_{t}}({\phi_{t}}-\gamma\phi_{t}^{\prime})(\phi_{t}^{T}{w_{t}})}\right)}}
\end{array}
\end{equation}
Note that if $h(\theta)=0$ and $X={\mathbb{R}^{d}}$, then we will
have TDC algorithm proposed in \cite{FastGradient:2009}. Note that
this technique does not have the same convergence guarantee as the original
objective function. For example, if we use a similar trick on the GTD update
with the optimal $y_{t}(\theta_{t})$ (termed as $y_{t}^{*}(\theta_{t})$)
where $y_{t}^{*}(\theta)=\mathbb{E}[{\delta_{t}}(\theta){\phi_{t}}]$,
then we can have
\begin{equation}
\theta_{t+1}={\rm pro}{\ensuremath{{\rm x}_{{\alpha_{t}}h}}}\ensuremath{\left({\theta_{t}}+{\alpha_{t}}{\delta_{t}}({\phi_{t}}-\gamma\phi_{t}^{'})\right)}
\end{equation}
which is the update rule of residual gradient \cite{Baird:ResidualAlgorithms1995},
and is proven not to converge to NEU any more.%
\footnote{It converges to mean-square TD error (MSTDE), as proven in \cite{maei2011gradient}.%
}

\section{Accelerated Gradient Temporal Difference Learning Algorithms}
In this section we will discuss  the acceleration of GTD2 and TDC.
The acceleration of GTD is not discussed due to space consideration,
which is similar to GTD2. A comprehensive
overview of the convergence rate of different approaches to stochastic
saddle-point problems is given in  \cite{chen2013optimal}. In this section we present  accelerated
algorithms based on the Stochastic Mirror-Prox (SMP) Algorithm \cite{sra2011optimization,juditsky2008solving}.
Algorithm \ref{alg:GTD2mp}, termed as GTD2-MP, is accelerated GTD2
with extragradient. Algorithm \ref{alg:TDCmp}, termed as TDC-MP,
is accelerated TDC with extragradient.
\begin{algorithm}
\caption{Algorithm Template}
\label{alg:tdneu}
Let $\pi$ be some fixed policy of an MDP $M$, $\Phi$ be some fixed basis.
\begin{algorithmic}[1]
\REPEAT
\STATE Compute $\ensuremath{{{\phi_{t}},{\phi_{t}}'}}$ and TD error ${\delta_{t}}={r_{t}}+\gamma{\phi_{t}}^{\prime T}{\theta_{t}}-\phi_{t}^{T}{\theta_{t}}$
\STATE Compute $\ensuremath{\theta_{t+1},w_{t+1}}$ according to each algorithm update rule\
\UNTIL {$t=N$};
\STATE Compute primal average ${\bar{\theta}_{N}}=\frac{1}{N}\sum\limits _{i=1}^{N}{\theta_{i}},{\bar{w}_{N}}=\frac{1}{N}\sum\limits _{i=1}^{N}{w_{i}}$
\end{algorithmic}
\end{algorithm}
\begin{algorithm}
\caption{GTD2-MP}
\label{alg:GTD2mp}
\begin{enumerate}
\item ${w_{t+\frac{1}{2}}}={{w_{t}}+{\beta_{t}}({\delta_{t}}-\phi_{t}^{T}{w_{t}}){\phi_{t}}},\\ \;{\theta_{t+\frac{1}{2}}}={{\rm{pro}}{{\rm{x}}_{{\alpha _t}h}}}\left({{\theta_{t}}+{\alpha_{t}}({\phi_{t}}-\gamma{\phi_{t}}^{\prime})(\phi_{t}^{T}{w_{t}})}\right)$
\item ${\delta_{t+\frac{1}{2}}}={r_{t}}+\gamma{\phi_{t}}^{\prime T}{\theta_{t+\frac{1}{2}}}-\phi_{t}^{T}{\theta_{t+\frac{1}{2}}}$
\item $\begin{array}{l}
{w_{t+1}}={w_{t}}+{\beta_{t}}({\delta_{t+\frac{1}{2}}}-\phi_{t}^{T}{w_{t+\frac{1}{2}}}){\phi_{t}}
{\;,\;}\\
{\theta_{t+1}}={{\rm{pro}}{{\rm{x}}_{{\alpha _t}h}}}\left({{\theta_{t}}+{\alpha_{t}}({\phi_{t}}-\gamma{\phi_{t}}^{\prime})(\phi_{t}^{T}{w_{t+\frac{1}{2}}})}\right)
\end{array}
$\end{enumerate}
\end{algorithm}
\begin{algorithm}
\caption{TDC-MP}
\label{alg:TDCmp}
\begin{enumerate}
\item ${w_{t+\frac{1}{2}}}={{w_{t}}+{\beta_{t}}({\delta_{t}}-\phi_{t}^{T}{w_{t}}){\phi_{t}}},\\ \;{\theta_{t+\frac{1}{2}}}={{\rm{pro}}{{\rm{x}}_{{\alpha _t}h}}}\left({{\theta_{t}}+{\alpha_{t}}\delta_{t}{\phi_{t}}-{\alpha_{t}}\gamma{\phi_{t}}^{\prime}(\phi_{t}^{T}{w_{t}})}\right)$
\item ${\delta_{t+\frac{1}{2}}}={r_{t}}+\gamma{\phi_{t}}^{\prime T}{\theta_{t+\frac{1}{2}}}-\phi_{t}^{T}{\theta_{t+\frac{1}{2}}}$
\item $\begin{array}{l}
{w_{t+1}}={w_{t}}+{\beta_{t}}({\delta_{t+\frac{1}{2}}}-\phi_{t}^{T}{w_{t+\frac{1}{2}}}){\phi_{t}}\;
{\;,\;}\\
{\theta_{t+1}}={{\rm{pro}}{{\rm{x}}_{{\alpha _t}h}}}\left({{\theta_{t}}+{\alpha_{t}}{\delta_{t+\frac{1}{2}}}{\phi_{t}}-{\alpha_{t}}\gamma{\phi_{t}}^{\prime}(\phi_{t}^{T}{w_{t+\frac{1}{2}}})}\right)
\end{array}$\end{enumerate}
\end{algorithm}

\section{Theoretical Analysis}
In this section, we discuss the convergence rate and error bound of
GTD, GTD2 and GTD2-MP. 

\subsection{Convergence Rate}
\textbf{Proposition 1} The convergence rates of the GTD/GTD2
algorithms with primal average are $O(\frac{{{L_{{F^{*}}}}+{L_{K}}+\sigma}}{{\sqrt{N}}})$,
where ${L_{K}}=||{\Phi^{T}}\Xi(\Phi-\gamma{\Phi^{'T}})|{|^{2}}$,
for GTD, ${L_{{F^{*}}}}=1$ and for GTD2, ${L_{{F^*}}} = ||{\Phi ^T}\Xi \Phi |{|_2}$, $\sigma$ is defined in the Appendix due to space limitations.

Now we consider the convergence rate of GTD2-MP.

\textbf{Proposition 2 }The convergence rate of the GTD2-MP algorithm
is $O(\frac{{{L_{{F^{*}}}}+{L_{K}}}}{N}+\frac{\sigma}{{\sqrt{N}}})$.


See supplementary materials for an abbreviated proof.
\textbf{Remark}: The above propositions imply that when the noise level is low, the GTD2-MP algorithm is able to converge at the rate of $O(\frac{1}{N})$, whereas the convergence rate of GTD2 is $O(\frac{1}{{\sqrt N }})$. However, when the noise level is high, both algorithms' convergence rates reduce to $O(\frac{\sigma }{{\sqrt N }})$.

\subsection{Value Approximation Error Bound}
\label{eq:mspbe}

\textbf{Proposition 3}: For GTD/GTD2, the prediction error of $||V-{V_{\theta}}||$
is bounded by $||V-{V_{\theta}}|{|_{\infty}}\le\frac{{L_{\phi}^{\Xi}}}{{1-\gamma}}\cdot O\left({\frac{{{L_{{F^{*}}}}+{L_{K}}+\sigma}}{{\sqrt{N}}}}\right)$
; For GTD2-MP, it is bounded by $||V-{V_{\theta}}|{|_{\infty}}\le\frac{{L_{\phi}^{\Xi}}}{{1-\gamma}}\cdot O\left({\frac{{{L_{{F^{*}}}}+{L_{K}}}}{N}+\frac{\sigma}{{\sqrt{N}}}}\right)$,
where $L_{\phi}^{\Xi}={\max_{s}}||{({\Phi^{T}}\Xi\Phi)^{-1}}\phi(s)|{|_{1}}$.

\textbf{Proof}: see Appendix.

\subsection{Related Work}

Here we will discuss previous related work. To the best of our knowledge,
the closest related work is the RO-TD algorithm, which first introduced
the convex-concave saddle-point framework to regularize the TDC family
of algorithms. \footnote{Although only regularized TDC was proposed in \cite{ROTD:NIPS2012},
the algorithm can be easily extended to regularized GTD and GTD2.}
The major difference is that RO-TD is motivated by the linear inverse problem formulation of TDC algorithm and uses its dual norm representation as the objective function, which does not explore the auxiliary variable $w_t$. In contrast, by introducing the operator splitting framework, we demonstrate that the GTD family of algorithms can be nicely explained as a ``true'' SGD approach, where the auxiliary variable $w_t$ has a nice explanation.

Another interesting question is whether ADMM is suitable for the operator
splitting algorithm here. Let's take NEU for example. The ADMM formulation
is as follows, where we assume $K(\theta ) = K\theta$ for simplicity, and other scenarios can be derived similarly,
\begin{equation}
\mathop{\min}\limits _{\theta,z}\left({F(z)+h(\theta)}\right){\rm {s}}.{\rm {t}}.{\rm {}}z=K\theta\label{eq:admm}
\end{equation}
The update rule is as follows, where $\alpha_{t}$ is the stepsize
\begin{equation}
\begin{array}{l}
{\theta_{t+1}}=\arg\mathop{\min}\limits _{\theta}\left({h(\theta)+\left\langle {{y_{t}},K\theta-{z_{t}}}\right\rangle +\frac{1}{2}||K\theta-{z_{t}}|{|^{2}}}\right)\\
{z_{t+1}}=\arg\mathop{\min}\limits _{z}\left({F(z)+\left\langle {{y_{t}},K{\theta_{t+1}}-z}\right\rangle +\frac{1}{2}||K{\theta_{t+1}}-z|{|^{2}}}\right)\\
{y_{t+1}}={y_{t}}+\alpha_{t}(K{\theta_{t+1}}-{z_{t+1}})
\end{array}\label{eq:admm-update}
\end{equation}
At first glance the operator of $F(\cdot)$ and $K\theta$ seem to
be split, however, if we compute the closed-form update rule of $\theta_{t}$,
we can see that the update of $\theta_{t}$ includes ${({K^{T}}K)^{-1}}$,
which involves both biased-sampling and computing the inverse matrix,
thus regular ADMM does not seem to be practical for this first-order
reinforcement learning setting. However, using the pre-conditioning technique
introduced in \cite{esser2010general}, ADMM can be reduced to the
primal-dual splitting method as pointed out in \cite{POCK2011SADDLE}.

\section{Experimental Study}
\subsection{Off-Policy Convergence: Baird Example}
The Baird example is a well-known example where TD diverges and TDC converges. The stepsizes are set to be constants where ${\beta _t} = \mu {\alpha _t}$ as shown in Figure \ref{fig:star}. From Figure \ref{fig:star}, we can see that GTD2-MP
and TDC-MP have a significant advantage over the GTD2 and TDC algorithms
wherein both the MSPBE and the variance are substantially reduced.
%
\begin{figure}
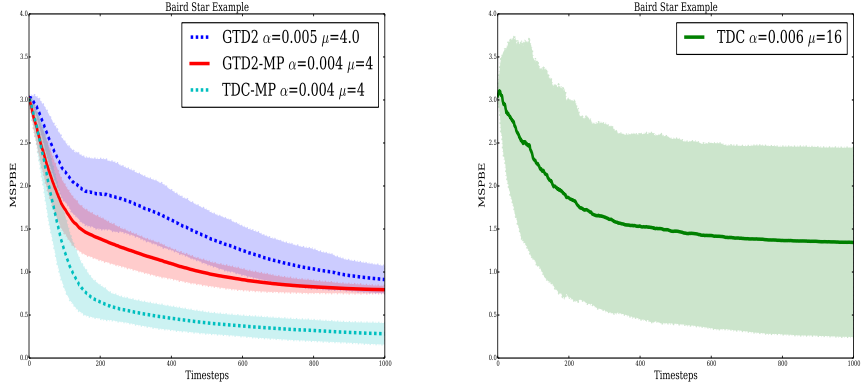

\centering
\begin{minipage}{1.15\textwidth}
\includegraphics[width= .45\textwidth, height=2.25in]{baird_1.pdf}
\includegraphics[width= .45\textwidth, height=2.25in]{baird_tdc.pdf}
\end{minipage}
\caption{Off-Policy Convergence Comparison}
\label{fig:star}
\end{figure}

\subsection{Regularization Solution Path: Two-State Example}
Now we consider the two-state MDP in \cite{DantzigRL:2012}. The transition
matrix and reward vector are $[0,1;0,1]$ and $R = {[0, - 1]^T},\gamma=0.9$, and a one-feature basis $\Phi={[1,2]^{T}}$.  The objective function are
$\theta  = \arg \mathop {\min }\limits_\theta  \left( {\frac{1}{2}L(\theta ) + \rho ||\theta |{|_1}} \right)$,
where $L(\theta)$ is ${{\rm {NEU}}(\theta)}$ and ${{\rm {MSPBE}}(\theta)}$.
The objective functions are termed as $l_{1}$-NEU and $l_{1}$-MSPBE
for short. In Figure \ref{fig:2state}, both $l_{1}$-NEU and $l_{1}$-MSPBE have well-defined solution paths w.r.t $\rho$, whereas Lasso-TD may
have multiple solutions if the $P$-matrix condition is not satisfied
\cite{Kolter09LARSTD}.
%
%
\begin{figure}
\centering
\begin{minipage}{1\textwidth}
\includegraphics[width= .32\textwidth, height=1.5in]{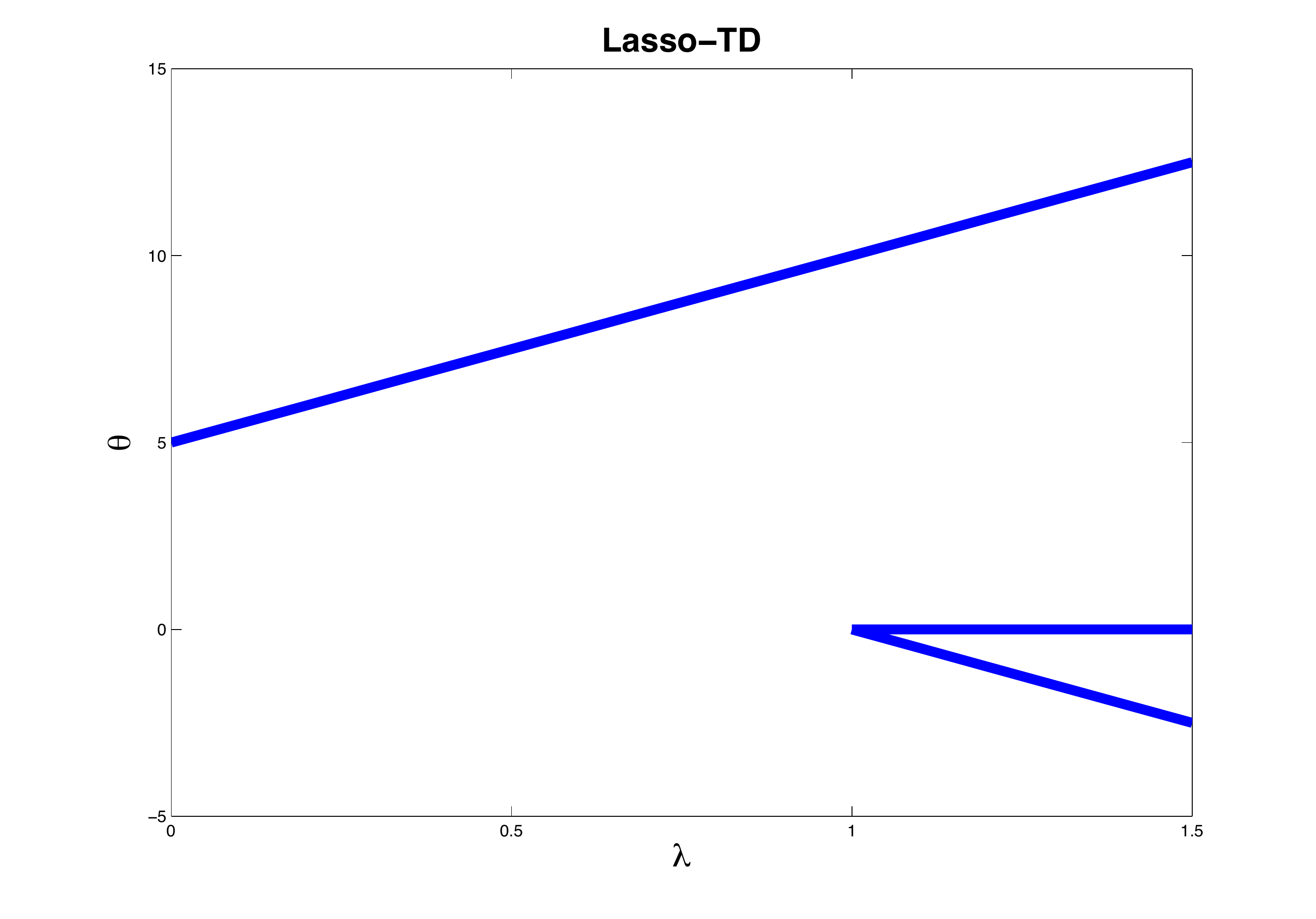}
\includegraphics[width= .32\textwidth, height=1.5in]{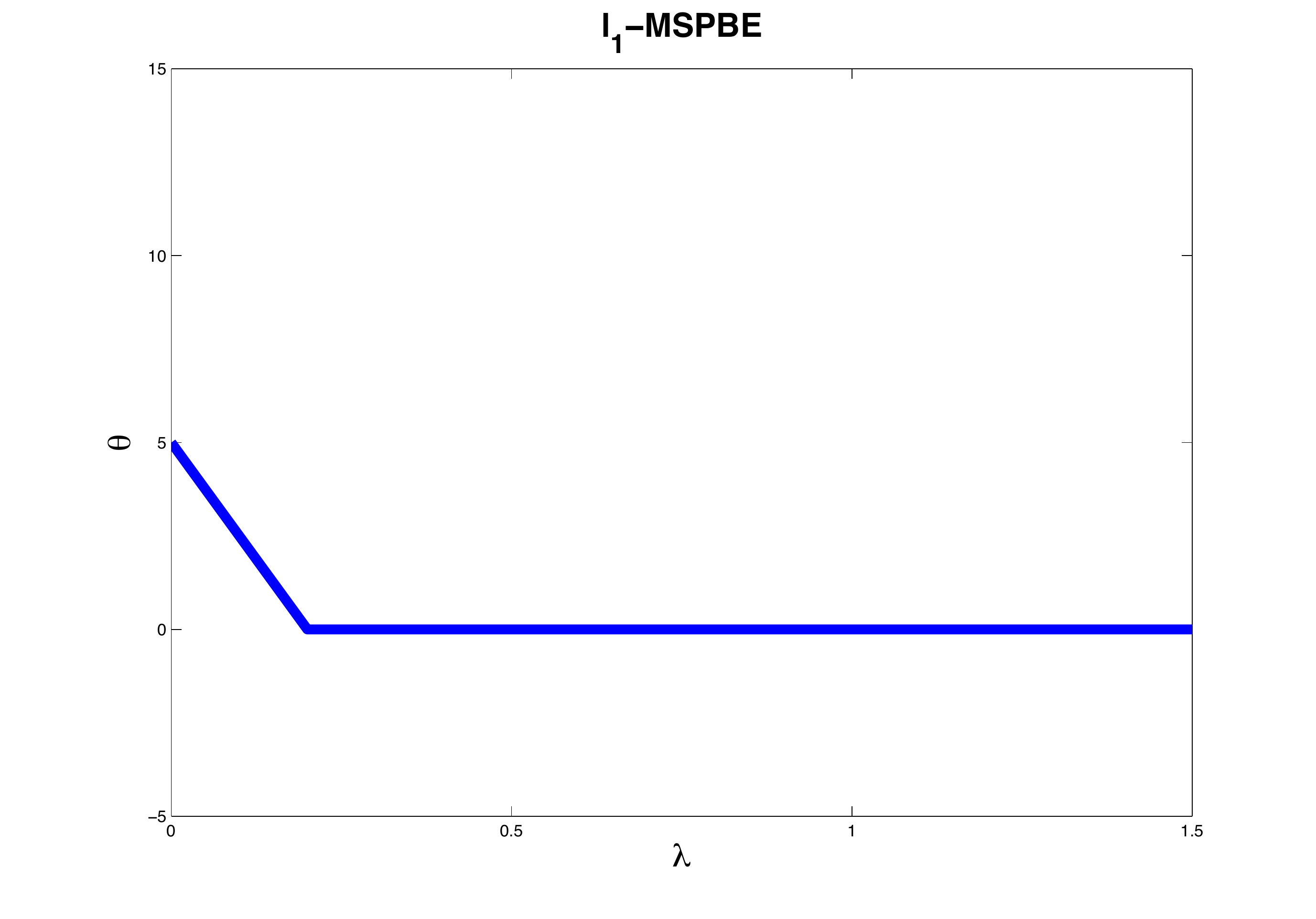}
\includegraphics[width= .32\textwidth, height=1.5in]{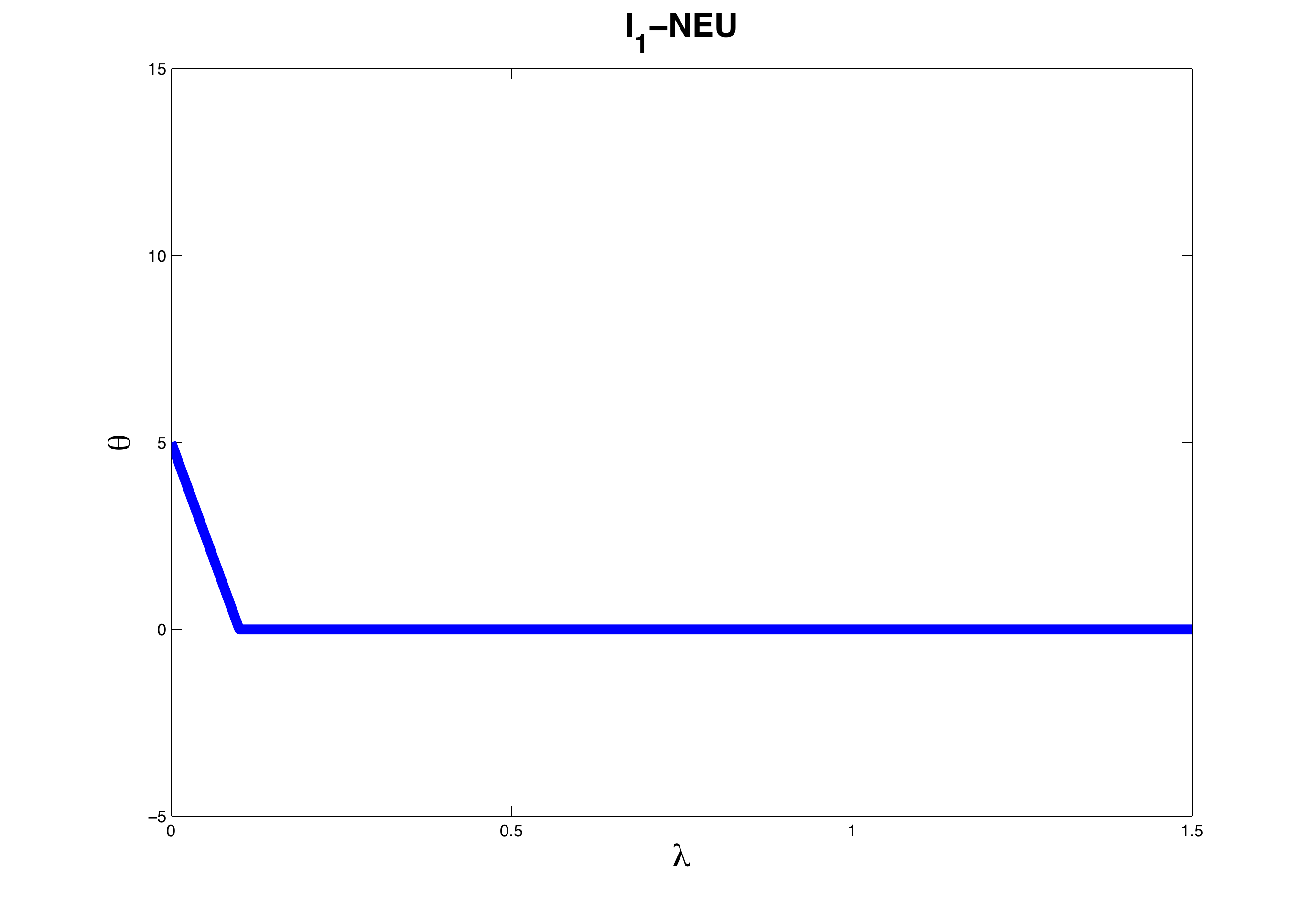}
\end{minipage}
\caption{Solution Path Comparison}
\label{fig:2state}
\end{figure}

\subsection{On-Policy Performance: $400$-State Random MDP}
In this experiment we compare the on-policy performance of the four algorithms. We use the random generated MDP with $400$ states and $10$ actions in \cite{dann2014tdsurvey}. Each state is represented by a feature vector with $201$ features, where $200$ features are generated by sampling from a uniform distribution the $201$-th feature is a constant. The stepsizes are set to be constants where ${\beta _t} = \mu {\alpha _t}$ as shown in Figure \ref{fig:400state}. The parameters of each algorithm are chosen via comparative studies similar to \cite{dann2014tdsurvey}.
The result is shown in Figure \ref{fig:400state}. The results for each algorithm are averaged on $100$ runs, and the parameters of each algorithm are chosen via experiments. TDC shows high variance and chattering effect of MSPBE curve on this domain.
Compared with GTD2, GTD2-M1P is able to reduce the MSPBE significantly. Compared with TDC, TDC-MP  not only reduces the MSPBE, but also  the variance and the ''chattering" effect.
\begin{figure}[h]
\centering{}\includegraphics[width=.8\textwidth,height=2.5in]{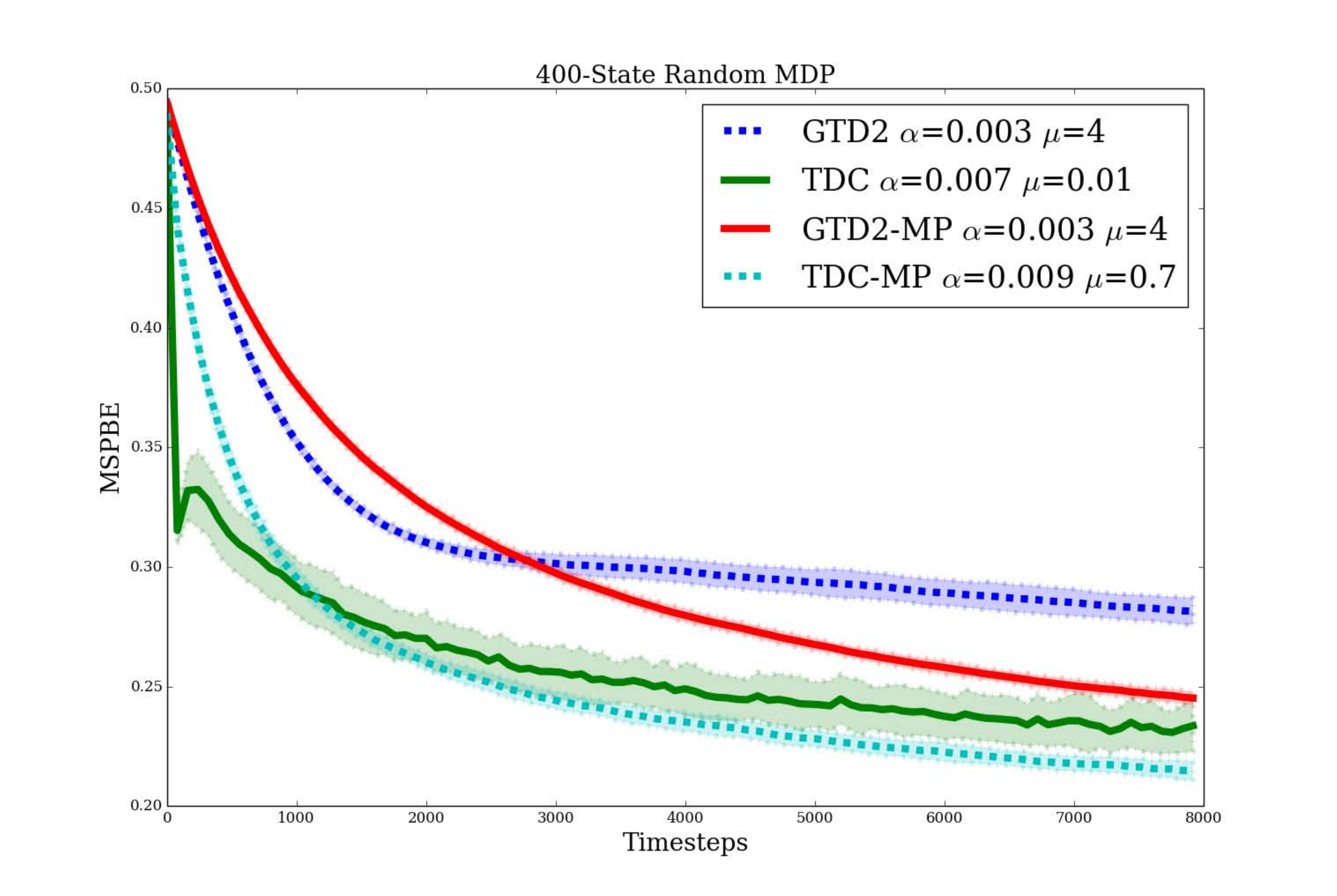}
\caption{Comparison of $400$-State Random MDP}
\label{fig:400state}
\end{figure}


\section{Summary}
\label{gtd-summary}

This chapter shows that the GTD/GTD2 algorithms are true stochastic gradient methods {\em w.r.t. the primal-dual formulation of their corresponding objective functions, which enables their convergence rate analysis and
regularization}. Second, it proposes operator splitting as a broad framework to solve the biased-sampling
problem in reinforcement learning. Based on the unified primal-dual splitting framework, it also proposes
accelerated algorithms with both rigorous theoretical analysis and illustrates their
improved performance w.r.t. previous methods. Future research is ongoing to explore other operator splitting techniques
beyond primal-dual splitting as well as incorporating random projections \cite{lstdrp:nips2010}, and investigating kernelized algorithms  \cite{barreto2013kernelrl,taylor2009kernelrl}. Finally, exploring the convergence rate of the TDC algorithm is also important
and interesting.

\chapter{Variational Inequalities: The Emerging Frontier of Machine Learning}
\label{vis}

This paper describes a new framework for reinforcement learning based on {\em primal dual} spaces connected by a Legendre transform. The ensuing theory yields surprising and beautiful solutions to several important questions that have remained unresolved: (i) how to design reliable, convergent, and stable reinforcement learning algorithms (ii) how to guarantee that reinforcement learning satisfies pre-specified ``safety" guarantees, and remains in a stable region of the parameter space (iv) how to design ``off-policy" TD-learning algorithms in a reliable and stable manner, and finally, (iii) how to integrate the study of reinforcement learning into the rich theory of stochastic optimization. In this paper, we gave detailed answers to all these questions using the powerful framework of {\em proximal operators}. The single most important idea that emerges is the use of {\em primal dual spaces} connected through the use of a {\em Legendre} transform. This allows temporal-difference updates to occur in dual spaces, allowing a variety of important technical advantages. The Legendre transform, as we show, elegantly generalizes past algorithms for solving reinforcement learning problems, such as {\em natural gradient} methods, which we show relate closely to the previously unconnected framework of {\em mirror descent} methods. Equally importantly, proximal operator theory enables the systematic development of {\em operator splitting} methods that show how to safely and reliably decompose complex products of gradients that occur in recent variants of gradient-based temporal-difference learning. This key technical contribution makes it possible to finally show to design ``true" stochastic gradient methods for reinforcement learning. Finally, Legendre transforms enable a variety of other benefits, including modeling sparsity and domain geometry. Our work builds extensively on recent work on the convergence of saddle-point algorithms, and on the theory of {\em  monotone operators} in Hilbert spaces,  both in optimization and for variational inequalities.  The latter represents possibly the most exciting future research direction, and we give a more detailed description of this ongoing research thrust.

\section{Variational Inequalities}

Our discussion above has repeatedly revolved around the fringes of variational inequality theory. Methods like extragradient \cite{extragradient:1976} and the mirror-prox algorithm were originally proposed  to solve variational inequalities and related saddle point problems. We are currently engaged in redeveloping the proposed ideas more fully within the fabric of variational inequality (VI). Accordingly, we briefly describe the framework of VIs, and give the reader a brief tour of this fascinating extension of the basic underlying framework of optimization. We lack the space to do a thorough review. That is the topic of another monograph to be published at a later date, and several papers on this topic are already under way.

At the dawn of a new millennium, the Internet dominates our economic, intellectual and social lives. The  concept of equilibrium plays a key role in understanding not only the Internet, but also other networked systems, such as human migration \cite{nagurney:migration}, evolutionary dynamics and the spread of infectious diseases \cite{novak:book}, and social networks \cite{kleinberg:text}. Equilibria are also a central idea in  game theory  \cite{fudenberg-levine:book,algorithmic-game-theory}, economics \cite{econ-text}, operations research \cite{puterman}, and many related areas.  We are currently exploring two powerful mathematical tools for the study of equilibria -- variational inequalities (VIs) and projected dynamical systems (PDS) \cite{nagurney:vibook,nagurney:pdsbook} -- in developing a new  machine learning framework for solving equilibrium problems in a rich and diverse range of practical applications.   As Figure~\ref{vi-problems} illustrates, finite-dimensional VIs provide a mathematical framework that unifies many disparate equilibrium problems of significant importance, including (convex) optimization, equilibrium problems in economics, game theory and networks, linear and nonlinear complementarity problems, and solutions of systems of nonlinear equations.

\begin{figure}[h]
\begin{center}
\begin{minipage}[t]{0.85\textwidth}
\includegraphics[width=\textwidth,height=3in]{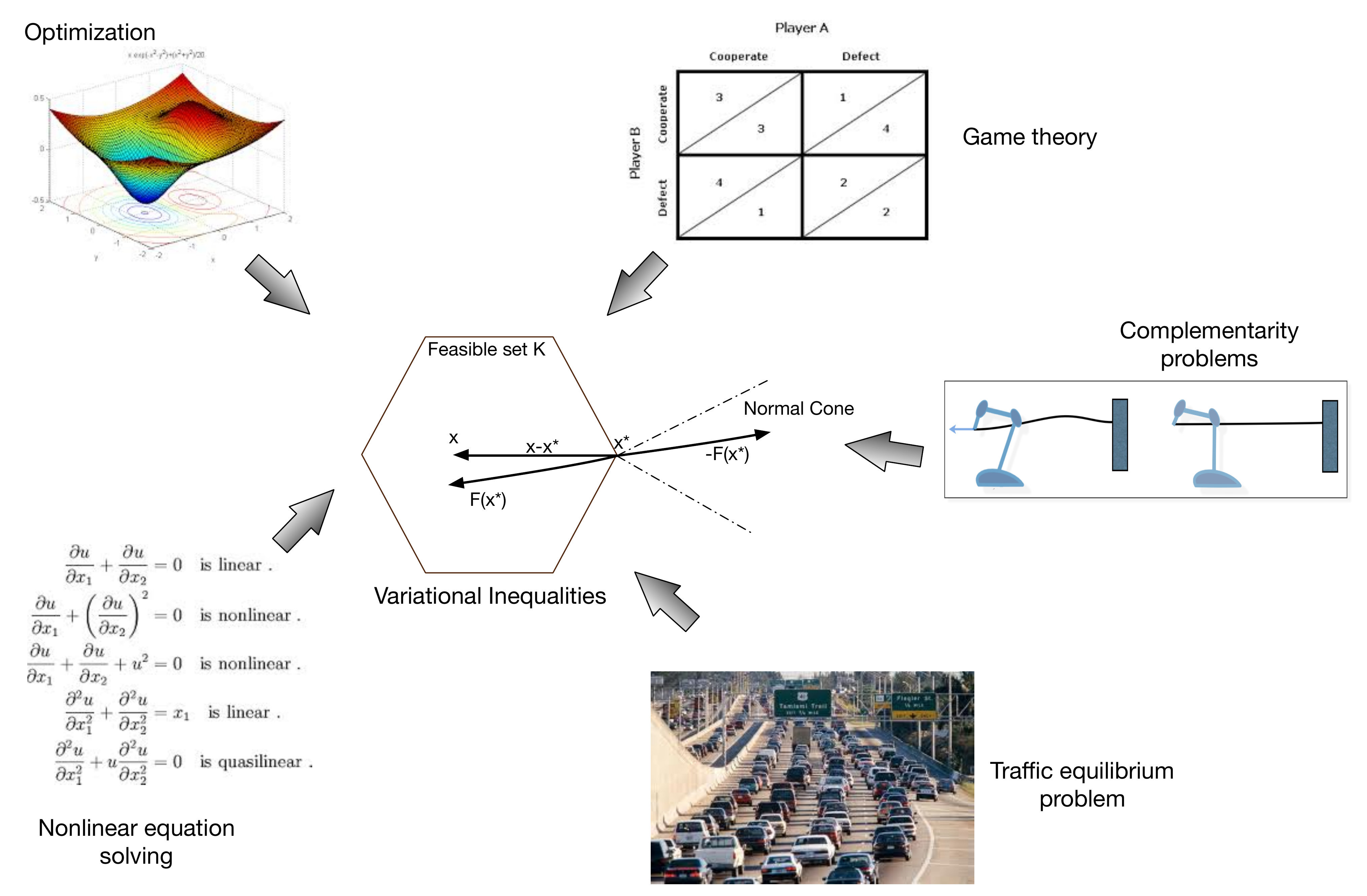}
\end{minipage} \hspace{0.2in}
\end{center}
\vskip -0.2in
\caption{A variety of  real-world problems can be modeled as solving variational inequalities.}
\label{vi-problems}
\end{figure}

Variational inequalities (VIs), in the infinite-dimensional setting,  were originally proposed by Hartman and Stampacchia \cite{hartman-stampacchia:acta} in the mid-1960s in the context of solving partial differential equations in mechanics. Finite-dimensional VIs rose in popularity in the 1980s partly as a result of work by Dafermos \cite{dafermos}. who showed that the traffic network equilibrium problem  could be formulated as a finite-dimensional VI. This advance inspired much follow-on research, showing that a variety of equilibrium problems in economics, game theory, sequential decision-making etc. could also be formulated as finite-dimensional VIs -- the books by Nagurney \cite{nagurney:vibook} and Facchinei and Pang \cite{facchinei-pang:vi} provide  a detailed introduction to the theory and applications of finite-dimensional VIs.     Projected dynamical systems (PDS) \cite{nagurney:pdsbook} are a class of ordinary differential equations (ODEs) with a discontinuous right-hand side. Associated with every finite-dimensional VI is a PDS, whose stationary points are the solutions of the VI. While VIs provide a static analysis of equilibria, PDS enable a microscopic examination of the dynamic processes that lead to or away from stable equilibria. . There has been longstanding interest in AI in the development of gradient-based learning algorithms for finding Nash equilibria in multiplayer games, e.g. \cite{bowling:aij,fudenberg-levine:book,singh:uai2000}. A gradient method for finding Nash equilibria can be formalized by a set of ordinary differential equations, whose phase space portrait solution reveals the dynamical process of convergence to an equilibrium point, or lack thereof. A key complication in this type of analysis is that the classical dynamical systems approach does not allow incorporating constraints on values of variables, which are omnipresent in equilibria problems, not only in games, but also in many other applications in economics, network flow, traffic modeling etc. In contrast, the right-hand side of a PDS is a discontinuous projection operator that allows enabling constraints to be modeled.

One of the original algorithms for solving finite-dimensional VIs is the {\em extragradient method} proposed by Korpelevich \cite{korpelevich}. It has been applied to structured prediction models in machine learning by Taskar et al. \cite{taskar:jmlr}.  Bruckner et al. \cite{bruckner:jmlr} use a modified extragradient method for solving the spam filtering problem modeled as a prediction game.  We are developing a new family of extragradient-like methods based on well-known numerical methods for solving ordinary differential equations, specifically the {\em Runge Kutta} method \cite{numerical-recipes}. In optimization, the extragradient algorithm was generalized to the non-Euclidean case by combining it with the mirror-descent method \cite{nemirovski-yudin:book}, resulting in the so-called ``mirrror-prox" algorithm \cite{mirror-prox, mirror-prox2}. We have extended the mirror-prox method by combining it with Runge-Kutta methods for solving high-dimensional VI problems over the simplex and other spaces.  We show the enhanced performance of Runge-Kutta extragradient methods on a range of benchmark variational inequalities drawn from standard problems in the optimization literature.

\subsection{Definition}

\begin{figure}[t]
\begin{center}
\begin{minipage}[t]{0.75\textwidth}
\includegraphics[width=\textwidth,height=1.5in]{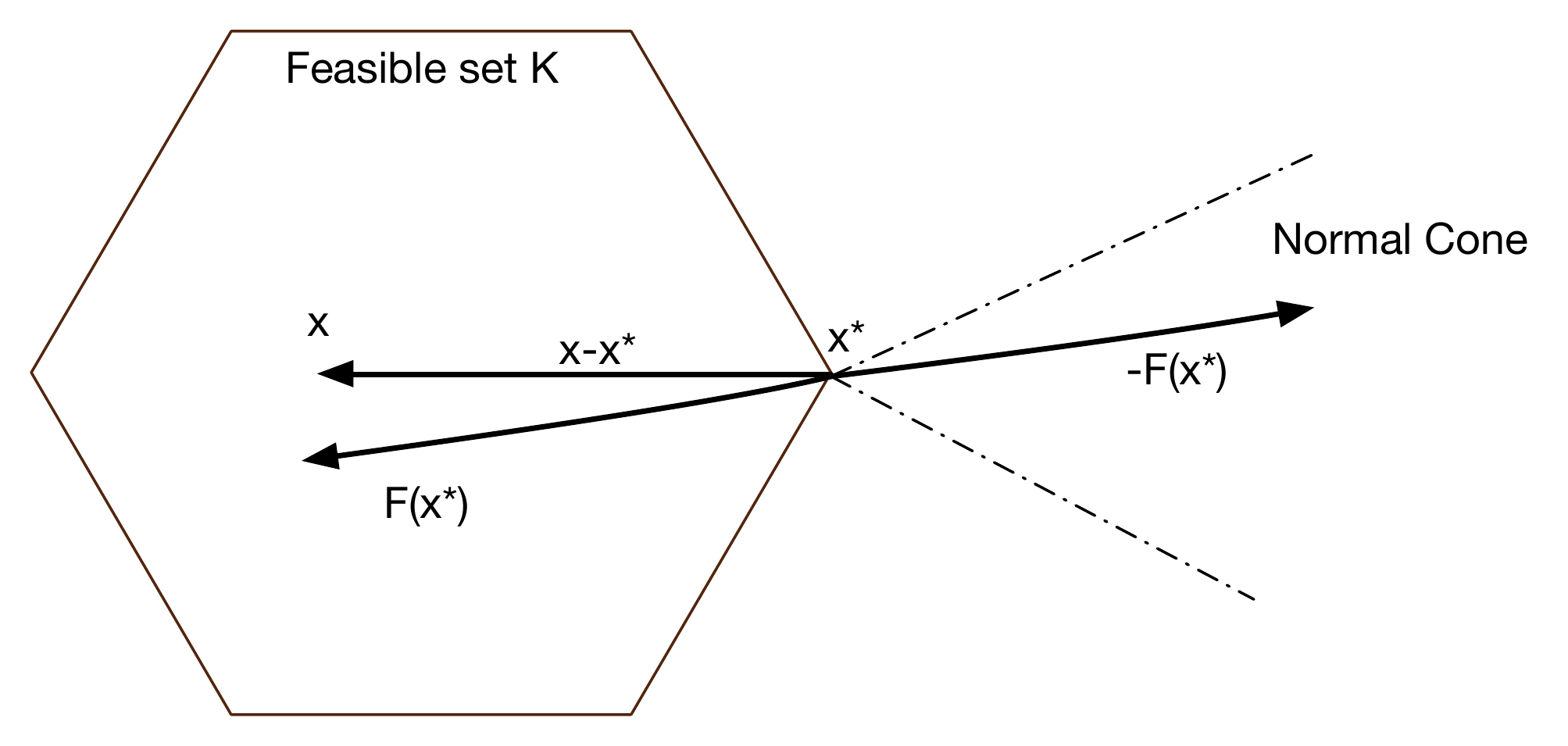}
\end{minipage}
\end{center}
\caption{This figure provides a geometric interpretation of the variational inequality $VI(F,K)$. The mapping $F$ defines a vector field over the feasible set $K$ such that at the solution point $x^*$, the vector field $F(x^*)$ is directed inwards at the boundary, and  $-F(x^*)$ is an element of the normal cone $C(x^*)$ of $K$ at $x^*$.}
\label{vi-geom}
\end{figure}


The formal definition of a VI as follows:\footnote{Variational problems can be defined more abstractly in Hilbert spaces. We  confine our discussion to $n$-dimensional Euclidean spaces.}

\begin{definition}
The finite-dimensional variational inequality problem VI(F,K) involves finding a vector $x^* \in K \subset \mathbb{R}^n$ such that

\[ \langle F(x^*), x - x^* \rangle \geq 0, \ \forall x \in K \]

where $F: K \rightarrow \mathbb{R}^n$ is a given continuous function and $K$ is a given closed convex set, and $\langle ., . \rangle$ is the standard inner product in $\mathbb{R}^n$.

\end{definition}

Figure~\ref{vi-geom} provides a geometric interpretation of a variational inequality. \footnote{In Figure~\ref{vi-geom}, the normal cone  $C(x^*)$ at the vector $x^*$ of a convex set $K$ is defined as $C(x^*) = \{y \in \mathbb{R}^n | \langle y, x - x^* \rangle \leq 0, \forall x \in K \}$. } The following general result characterizes when solutions to VIs exist:

\begin{theorem}
Suppose $K$ is compact, and that $F: K \rightarrow \mathbb{R}^n$ is continuous. Then, there exists a solution to VI($F,K$).
\end{theorem}

As Figure~\ref{vi-geom} shows, $x^*$ is a solution to $VI(F,K)$ if and only if the angle between the vectors $F(x^*)$ and $x - x^*$, for any vector $x \in K$, is less than or equal to $90^0$. To build up some intuition, the reduction of a few well-known problems to a VI is now provided.

\begin{theorem}
Let $x^*$ be a solution to the optimization problem of minimizing a continuously differentiable function $f(x)$, subject to $x \in K$, where $K$ is a closed and convex set. Then, $x^*$ is a solution to $VI(\nabla f, K)$, such that $\langle \nabla f(x^*), x - x^* \rangle \geq 0, \ \forall x \in K$.
\end{theorem}

{\bf Proof:} Define $\phi(t) = f(x^* + t (x - x^*))$. Since $\phi(t)$ is minimized at $t=0$, it follows that $0 \leq \phi'(0) = \langle \nabla f(x^*), x - x^* \rangle \geq 0, \ \forall x \in K$, that is $x^*$ solves the VI.

\begin{theorem}
If $f(x)$ is a convex function, and $x^*$ is the solution of $VI(\nabla f, K)$, then $x^*$ minimizes $f$.
\end{theorem}

{\bf Proof:} Since $f$ is convex, it follows that any tangent lies below the function, that is $f(x) \geq f(x^*) + \langle \nabla f(x^*), x - x^* \rangle, \ \forall x \in K$. But, since $x^*$ solves the VI, it follows that $f(x^*)$ is a lower bound on the value of $f(x)$ everywhere, or that $x^*$ minimizes $f$.

A rich class of problems called {\bf complementarity problems} (CPs) also can be reduced to solving a VI. When the feasible set $K$ is a cone, meaning that if $x \in K$, then $\alpha x \in K, \alpha \geq 0$, then the VI becomes a CP.

\begin{definition}
Given a cone $K \subset \mathbb{R}^n$, and a mapping $F: K \rightarrow \mathbb{R}^n$, the complementarity problem CP(F,K) is to find an $x \in K$ such that $F(x) \in K^*$, the dual cone to $K$, and $\langle x, F(x) \rangle \geq 0$. \footnote{Given a cone $K$, the dual cone $K^*$ is defined as $K^* = \{ y \in \mathbb{R}^n | \langle y, x \rangle \geq 0, \forall x \in K \}$.}
\end{definition}

A number of special cases of CPs are important. The nonlinear complementarity problem (NCP) is to find $x^* \in \mathbb{R}^n_+$ (the non-negative orthant)  such that $F(x^*) \geq 0$ and $\langle F(x^*), x^* \rangle = 0$. The solution to an NCP and the corresponding $VI(F, \mathbb{R}^n_+)$ are the same, showing that NCPs reduce to VIs. In an NCP, whenever the mapping function $F$ is affine, that is $F(x) = Mx + b$, where $M$ is an $n \times n$ matrix, then the corresponding NCP is called a linear complementarity problem (LCP) \cite{murty:lcpbook}.  Recent work on learning sparse models using $L_1$ regularization has exploited the fact that the standard LASSO objective \cite{lasso} of $L_1$ penalized regression can be reduced to solving an LCP \cite{lasso-lcp}. This reduction to LCP has been used in recent work on sparse value function approximation as well in a method called LCP-TD \cite{jeff:lcp-rl}.
A final crucial property of VIs is that they can be formulated as finding {\bf fixed points}.

\begin{theorem}
\label{projthm}
The vector $x^*$ is the solution of VI(F,K) if and only if, for any $\gamma > 0$, $x^*$ is also a fixed point of the map  $x^* = \Pi_K(x^* - \gamma F(x^*))$,
where $\Pi_K$ is the projector onto convex set $K$.

\end{theorem}

In terms of the geometric picture of a VI illustrated in Figure~\ref{vi-geom}. this property means that the solution of a VI occurs at a vector $x^*$ where the vector field $F(x^*)$ induced by $F$ on $K$ is normal to the boundary of $K$ and directed inwards, so that the projection of $x^* - \gamma F(x^*)$ is the vector $x^*$ itself. This property forms the basis for the projection class of methods that solve for the fixed point.

\subsection{Equilibrium Problems in Game Theory}

The VI framework provides a mathematically elegant approach to model equilibrium problems in game theory  \cite{fudenberg-levine:book,algorithmic-game-theory}. A {\em Nash game} consists of $m$ players, where player $i$ chooses a strategy $x_i$ belonging to a closed convex set $X_i \subset \mathbb{R}^n$. After executing the joint action, each player is penalized (or rewarded) by the amount $F_i(x_1, \ldots, x_m)$, where $F_i: \mathbb{R}^{n_i} \rightarrow \mathbb{R}$ is a continuously differentiable function. A set of strategies $x^* = (x^*_1, \ldots, x^*_m) \in \prod_{i=1}^M X_i$ is said to be in equilibrium if no player can reduce the incurred penalty (or increase the incurred reward) by unilaterally deviating from the chosen strategy. If each $F_i$ is convex on the set $X_i$, then the set of strategies $x^*$ is in equilibrium if and only if $\langle (x_i - x^*_i), \nabla_i F_i(x^*_i) \rangle \geq 0$.  In other words, $x^*$ needs to be a solution of the VI $\langle (x - x^*), f(x^*) \rangle \geq 0$, where $f(x) = (\nabla F_1(x), \ldots, \nabla F_m(x))$. Nash games are closely related to {\em saddle point} problems \cite{mirror-prox,mirror-prox2,bo-sm:nips2012}. where we are given a function $F: X \times Y \rightarrow \mathbb{R}$, and the objective is to find a solution $(x^*, y^*) \in X \times Y$ such that
\[ F(x^*, y) \leq F(x^*, y^*) \leq F(x, y^*), \ \ \forall x \in X, \ \forall y \in Y \]
Here, $F$ is convex in $x$ for each fixed $y$, and concave in $y$ for each fixed $x$. Many equilibria problems in economics can be modeled using VIs \cite{nagurney:vibook}.

\section{Algorithms for Variational Inequalities}
\label{viola}

We briefly describe two algorithms for solving variational inequalities below: the projection method and the extragradient method. We conclude with a brief discussion of how these relate to reinforcement learning.

\subsection{Projection-Based Algorithms for VIs}
\label{egvi}

The basic projection-based method (Algorithm 1) for solving VIs is based on Theorem~\ref{projthm} introduced earlier.

\begin{algorithm}
\caption{The Basic Projection Algorithm for solving VIs.}

{\bf INPUT:} Given VI(F,K), and a symmetric positive definite matrix $D$.

\begin{algorithmic}[1]

\STATE Set $k=0$ and $x_k \in K$.

\REPEAT

\STATE Set $x_{k+1} \leftarrow \Pi_{K,D} (x_k - D^{-1} F(x_k))$.

\STATE Set $k \leftarrow k+1$.

\UNTIL{$x_k = \Pi_{K,D} (x_k - D^{-1} F(x_k))$}.

\STATE Return $x_k$

\end{algorithmic}
\end{algorithm}

Here, $\Pi_{K,D}$ is the projector onto convex set $K$ with respect to the natural norm induced by $D$, where $\| x \|^2_D  = \langle x, D x \rangle$. It can be shown that the basic projection algorithm solves any $VI(F,K)$ for which the mapping $F$ is {\em strongly monotone} \footnote{ A mapping $F$ is {\em strongly monotone} if
$\langle F(x) - F(y), x  - y \rangle \geq \mu \| x - y \|^2_2, \mu > 0, \forall x,y \in K$.}  and {\em Lipschitz}.\footnote{A mapping $F$ is {\em Lipschitz} if $\| F(x) - F(y) \|_2 \leq L \|x - y \|_2, \forall x,y \in K$. }A simple strategy is to set $D = \alpha I$, where $\alpha > \frac{L^2}{2 \mu}$, and $L$ is the Lipschitz smoothness constant, and $\mu$ is the strong monotonicity constant. The basic projection-based algorithm has two critical limitations: it requires that the mapping $F$ be strongly monotone. If, for example, $F$ is the gradient map of a continuously differentiable function, strong monotonicity implies the function must be strongly convex.  Second, setting the parameter $\alpha$ requires knowing the Lipschitz smoothness $L$ and the strong monotonicity parameter $\mu$. The extragradient method of Korpolevich \cite{korpelevich} addresses some of these concerns, and is defined as Algorithm 2 below.

%

\begin{algorithm}
\caption{The Extragradient Algorithm for solving VIs.}

{\bf INPUT:} Given VI(F,K), and a scalar $\alpha$.

\begin{algorithmic}[1]

\STATE Set $k=0$ and $x_k \in K$.

\REPEAT

\STATE  Set $y_{k} \leftarrow \Pi_{K} (x_k - \alpha  F(x_k))$.

\STATE \label{projstep} Set $x_{k+1} \leftarrow \Pi_{K} (x_k - \alpha F(y_k))$.

\STATE Set $k \leftarrow k+1$.

\UNTIL{$x_k = \Pi_{K} (x_k - \alpha F(x_k))$}.

\STATE Return $x_k$

\end{algorithmic}
\end{algorithm}

Figure~\ref{eg-vi-example} shows a simple example where Algorithm 1 fails to converge, but Algorithm 2 does. If the initial point $x_0$ is chosen to be on the boundary of $X$, using Algorithm 1, it stays on it and fails to converge to the solution of this VI (which is at the origin). If $x_0$ is chosen to be in the interior of $K$, Algorithm 1 will  move towards the boundary. In contrast, using Algorithm 2, the solution can be found for any starting point.  The extragradient algoriithm derives its name from the property that it requires an ``extra gradient" step (step 4 in Algorithm 2), unlike the basic projection algorithm given earlier as Algorithm 1. The principal advantage of the extragradient method is that it can be shown to converge under a considerably weaker condition on the mapping $F$, which now has to be merely monotonic: $\langle F(x) - F(y), x - y \rangle \geq 0$. The earlier Lipschitz condition is still necessary for convergence.

\begin{figure}[h]
\begin{center}
\begin{minipage}[t]{0.4\textwidth}
\includegraphics[width=\textwidth,height=1.5in]{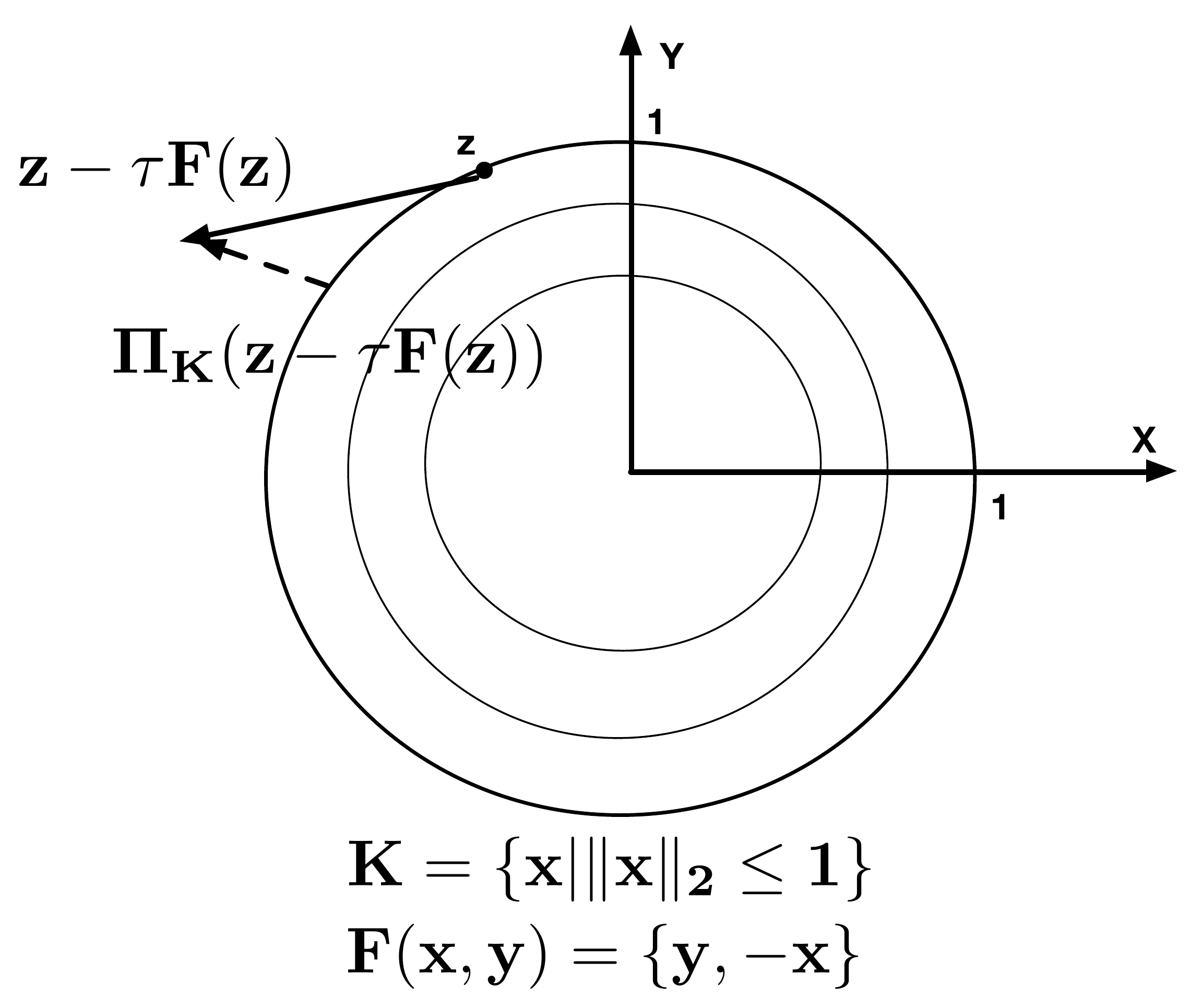}
\end{minipage}
\begin{minipage}[t]{0.4\textwidth}
\includegraphics[width=\textwidth,height=1in]{eg-one-step.pdf}
\end{minipage}
\end{center}
\caption{Left: This figure illustrates a VI where the basic projection algorithm (Algorithm 1) fails, but the extragradient algorithm (Algorithm 2) succeeds \cite{bertsekas:pdc}. Right: One iteration of the extradient algorithm.}
\label{eg-vi-example}
\end{figure}

The extragradient algorithm has been the topic of much attention in optimization since it was proposed, e.g.,  see \cite{iusem,khobotov,marcotte,peng:extragradient,nesterov:dualeg,solodov-svaiter}. Khobotov \cite{khobotov} proved that the extragradient method converges under the weaker requirement of {\em pseudo-monotone} mappings, \footnote{A mapping $F$ is {\em pseudo-monotone} if $\langle F(y), x - y \rangle \geq 0 \Rightarrow \langle F(x), x - y \rangle \geq 0, \ \forall x, y \in K$.}  when the learning rate is automatically adjusted based on a local measure of the Lipschitz constant. Iusem \cite{iusem} proposed a variant whereby the current iterate was projected onto a  hyperplane separating the current iterate from the final solution, and subsequently projected from the hyperplane onto the feasible set. Solodov and Svaiter \cite{solodov-svaiter} proposed another hyperplane method, whereby the current iterate is projected onto the intersection of the hyperplane and the feasible set. Finally, the extragradient method was generalized to the non-Euclidean case by combining it with the mirror-descent method \cite{nemirovski-yudin:book}, resulting in the so-called ``mirrror-prox" algorithm \cite{mirror-prox}.

\subsection{Variational Inequaities and Reinforcement Learning}

Variational inequalities also provide a useful framework for reinforcement learning \cite{ndp:book,sutton-barto:book}. In this case, it can be shown that the mapping $F$ for the VI defined by reinforcement learning is {\em affine} and represents a (linear) complementarity problem. For this case, a number of special properties can be exploited in designing a faster more scalable class of algorithms.
Recall from Theorem~\ref{projthm} that each VI(F,K)  corresponds to solving a particular fixed point problem $x^* = \Pi_K(x^* - \gamma F(x^*))$, which led to the projection  algorithm (Algorithm 1). Generalizing this, consider solving for the fixed point of a {\em projected equation} $x^* = \Pi_{\hat{S}} T(x^*)$  \cite{bertsekas:vi-tr09,gordon:lcp} for a functional mapping $T: \mathbb{R}^n \rightarrow \mathbb{R}^n$, where $\Pi_{\hat{S}}$ is the projector onto a low-dimensional convex subspace $\hat{S}$ w.r.t. some positive definite matrix $\Xi$, so that
\begin{equation}
\hat{S} = \{\Phi r | r \in \hat{R} \}, \ \ \hat{R} = \{r | \Phi r \in \hat{S} \}  \Rightarrow \Phi r^* = \Pi_{\hat{S}} T(\Phi r^*)
\end{equation}
Here. $\Phi$ is an $n \times s$ matrix where $s \ll n$, and the goal is to make the computation depend on $s$, not $n$. Note that $x^* = \Pi_{\hat{S}} T(x^*)$ if and only if
\[ \langle (x^* - T(x^*),  \Xi (x - x^*) \rangle  \geq 0. \ \forall x \in \hat{S} \]
Following \cite{bertsekas:vi-tr09}, note that this is a variational inequality of the form $\langle F(x^*), (x - x^*) \rangle \geq 0$ if we identify $F(x) = \Xi (x - T(x))$, and in the lower-dimensional space, $\langle F(\Phi r^*), \Phi(r - r^*) \rangle, \ \forall r \in \hat{R}$. Hence, the projection algorithm takes on the form:
\[ x_{k+1} = \Pi_{\hat{S}} (x_k -  \gamma D^{-1} \langle \Phi,  F(\Phi x_k)) \rangle \]
It is shown in \cite{bertsekas:vi-tr09} that if $T$ is a contraction mapping, then $F(x) = \Xi (x - T(x))$ is strongly monotone. Hence, the above projection algorithm will converge to the solution $x^*$ of the VI for any given starting point $x_0 \in \hat{S}$. Now, the only problem is how to ensure the computation depends only on the size of the projected lower-dimensional space (i.e., $s$, not $n$). To achieve this, let us assume that the mapping $T(x) = Ax + b$ is affine, and that the constraint region $\hat{R}$ is polyhedral. In this case, we can use the following identities:
\[ \langle \Phi, F(\Phi, x) \rangle = \Phi^T \Xi F(\Phi x) = Cr - d, \ \ C = \Phi^T \Xi (I - A) \Phi, \ \ d = \Phi^T \Xi b \]
and the projection algorithm for this affine case can be written as:
\[ x_{k+1} = \Pi_{\hat{S}} (x_k - \gamma D^{-1} (C r - d) ) \]
One way to solve this iteratively is to compute a progressively more accurate approximation $C_k \rightarrow C$ and $d_k \rightarrow d$  by sampling from the rows and columns of the matrix  $A$ and vector $b$, as follows:
\[ C_k  = \frac{1}{k+1} \sum_{t=0}^k \phi(i_t) (\phi(i_t) - \frac{a_{i_t j_t}}{p_{i_t j_t}} \phi(j_t))^T, \ \ d_k = \frac{1}{k+1} \sum_{t=0}^k \phi(i_t) b_{i_t} \]
where the row sampling generates the indices $(i_0, i_1, \ldots)$ and the column sampling generates the transitions $((i_0, j_0), (i_1, j_1), \ldots, )$ in such a way that the relative frequency of row index $i$ matches the diagonal element $\xi_i$ of the positive definite matrix $\Xi$. Given $C_k$ and $d_k$, the solution can be found by $x^* \approx C_k^{-1} d_k$, or by using an incremental method. Computation now only depends on the dimension $s$ of the lower-dimensional space, not on the original high-dimensional space. Gordon \cite{gordon:lcp} proposes an alternative approach separating the projection of the current iterate on the low-dimensional subspace spanned by $\Phi$ from its projection onto the feasible set. Both of these approaches \cite{bertsekas:vi-tr09,gordon:lcp} have been only studied with the simple projection method (Algorithm 1), and can be generalized to a more powerful class of VI methods that we are currently developing.

\section*{Acknowledgements}

We  like to acknowledge the useful feedback of  past and present members of the Autonomous Learning Laboratory at the University of Massachusetts, Amherst.  Principal funding for this research was provided by the National Science Foundation under the grant NSF IIS-1216467. Past work by the first author has been funded by NSF grants IIS-0534999 and IIS-0803288.

\chapter{Appendix: Technical Proofs}

\section{Convergence Analysis of Saddle Point Temporal Difference Learning}

\subsection*{Proof of Proposition 1}
We give a descriptive proof here. We first present the monotone operator corresponding to the bilinear saddle-point problem and then extend it to stochastic approximation case with certain restrictive assumptions, and use the result in  \cite{sra2011optimization}.

The monotone operator $\Phi (x,y)$ with saddle-point problem $SadVal = {\inf _{x \in X}}{\sup _{y \in Y}}\phi (x,y)$ is a point-to-set operator
\[\Phi (x,y) = \{ {\partial _x}\phi (x,y)\}  \times \{  - {\partial _y}\phi (x,y)\} \]
Where ${\partial _x}\phi (x,y)$ is the subgradient of $\phi(x,\cdot)$ over $x$ and ${\partial _y}\phi (x,y)$ is the subgradient of $\phi(\cdot,y)$ over $y$. For the bilinear problem in Equation (\ref{eq:minimax}), the corresponding $\Phi(x,y)$ is
\begin{equation}\label{eq:bilinearphi}
\Phi (x,y) = ({A^T}y,b - Ax)
\end{equation}
Now we verify that the problem (\ref{eq:lossfunc}) can be reduced to a standard bilinear minimax problem. To prove this we only need to prove $X$ in our RL problem is indeed a closed compact convex set. This is easy to verify as we can simply define $X = \{ x|{\left\| x \right\|_2} \le R\} $ where $R$ is large enough. In fact, the sparse regularization $h(x)= \rho ||x||_1$ helps $x_t$ stay within this $l_2$ ball.
Now we extend to the stochastic approximation case wherein the objective function $f(x)$ is given by the stochastic oracle, and in our case, it is $Ax-b$, where $A, b$ are defined in Equation (\ref{eq:UNIFY}), with Assumption 3 and further assuming that the noise $\varepsilon_t$ for $t$-th sample is i.i.d noise, then with the result in \cite{juditsky2008solving}, we can prove that the RO-TD algorithm converges to the global minimizer of
\[{x^*} = \arg {\min _{x \in X}}{\left\| {Ax - b} \right\|_m} + \rho {\left\| x \right\|_1}\]
Then we prove the error level of approximate saddle-point ${\bar x_t},{\bar y_t}$ defined in (\ref{eq:averaging}) is $\alpha_{t} L^2$. With the subgradient boundedness assumption and using the result in Proposition 1 in \cite{nedic2009subgradient}, this can be proved.

\section{Convergence Analysis of True Gradient Temporal Difference Learning}

We first present the assumptions for the MDP and basis functions,
which are similar to \cite{FastGradient:2009,ROTD:NIPS2012}.

\textbf{Assumption 1} (\textbf{MDP}): The underlying Markov Reward
Process (MRP) $M=(S,P,R,\gamma)$ is finite and mixing, with stationary
distribution $\pi$. The training sequence $({s_{t}},{a_{t}},s_{t}^{'})$
is an i.i.d sequence. 

\textbf{Assumption 2} (\textbf{Basis Function}): The inverses $\mathbb{E}{[\phi_{t}{\phi_{t}^{T}}]^{-1}}$
and ${[{\phi_{t}}({\phi_{t}}-\gamma{\phi_{t}}^{'T})]^{-1}}$ exist.
This implies that $\Phi$ is a full column rank matrix. Also, assume
the features $(\phi_{t},\phi_{t}^{'})$ have uniformly bounded second
moments, and ${\left\Vert {\phi_{t}}\right\Vert _{\infty}}<+\infty,{\left\Vert {\phi{'_{t}}}\right\Vert _{\infty}}<+\infty$.
\\
 Next we present the assumptions for the stochastic saddle point problem
formulation, which are similar to \cite{chen2013optimal,juditsky2008solving}.

\textbf{Assumption 3} (\textbf{Compactness}): Assume for the primal-dual
loss function,
\begin{equation}
\mathop{\min}\limits _{\theta\in X}\mathop{\max}\limits _{y\in Y}\left({L(\theta,y)=\left\langle {K(\theta),y}\right\rangle -{F^{*}}(y)}+h(\theta)\right),
\label{eq:primal-dual2}
\end{equation}
the sets $X,Y$ are closed compact sets.

\textbf{Assumption 4} ($F^{*}(\cdot)$): We assume that $F^{*}(\cdot)$
is a smooth convex function with Lipschitz continuous gradient, i.e.,
$\exists L_{F^{*}}$ such that $\forall x,y\in X$
\begin{equation}
{F^{*}}(y)-{F^{*}}(x)-\left\langle {\nabla{F^{*}}(x),y-x}\right\rangle \le\frac{{L_{{F^{*}}}}}{2}||y-x|{|^{2}}
\end{equation}

\textbf{Assumption 5} ($K(\theta)$): $K(\theta)$ is a linear mapping
which can be extended to Lipschitz continuous vector-valued mapping
defined on a closed convex cone $C_{K}$. Assuming $K(\theta)$ to
be $C_{K}$-convex, i.e., $\forall\theta,\theta'\in X,\lambda\in[0,1],$
\[
K(\lambda\theta+(1-\lambda)\theta'){\le_{{C_{K}}}}\lambda K(\theta)+(1-\lambda)K(\theta'),
\]

where $a{\le_{{C_{K}}}}b$ means that $b-a\in{C_{K}}$.

\textbf{Assumption 6} (\textbf{Stochastic Gradient}): In the stochastic saddle
point problem, we assume that there exists a stochastic oracle $\mathcal{SO}$
that is able to provide unbiased estimation with bounded variance
such that

\begin{equation}
\begin{array}{l}
\mathbb{E}[{{\cal F}^{*}}({y_{t}})]={\nabla{F^{*}}({y_{t}})}\\
\mathbb{E}[||{{\cal F}^{*}}({y_{t}})-\nabla{F^{*}}({y_{t}})||^2] \le{\sigma^2_{{F^{*}}}}
\end{array}
\end{equation}

\begin{equation}
\begin{array}{l}
\mathbb{E}[{\cal K_{\theta}}(\theta_{t})]=K({\theta_{t}})\\
\mathbb{E}[||{\cal K_{\theta}}(\theta_{t})-K({\theta_{t}})|{|^{2}}]\le\sigma_{K,\theta}^{2}
\end{array}
\end{equation}

\begin{equation}
\begin{array}{l}
\mathbb{E}[{{\cal K}_{y}(\theta_{t})^{T}}{y_{t}}]=\nabla K{(\theta_{t})^{T}}{y_{t}}\\
\mathbb{E}[||{{\cal K}_{y}(\theta_{t})^{T}}{y_{t}}-\nabla K{(\theta_{t})^{T}}{y_{t}}|{|^{2}}]\le\sigma_{K,y}^{2}
\end{array}
\end{equation}

where $\sigma_{{F^{*}}}$ , $\sigma_{K,\theta}$ and $\sigma_{K,y}$
are non-negative constants.
We further define
\begin{equation}
\sigma=\sqrt{\sigma_{{F^{*}}}^{2}+\sigma_{K,\theta}^{2}}+{\sigma_{K,y}}\label{eq:sigma}
\end{equation}

\subsection{Convergence Rate}

Here we discuss the convergence rate of the proposed algorithms. First let us review the nonlinear primal form
\begin{equation}
\mathop{\min}\limits _{\theta\in X}\left(\Psi(\theta)={F(K(\theta))+h(\theta)}\right)\label{eq:primal-1}
\end{equation}
The corresponding primal-dual formulation \cite{BOOK2011PROXSPLIT,PROXSPLITTING2011,POCK2011SADDLE}
of Equation (\ref{eq:primal-1}) is Equation (\ref{eq:primal-dual}).
Thus we have the general update rule as
\begin{equation}
\begin{array}{l}
{y_{t + 1}} = {y_t} + {\alpha _t}{{\cal K}_\theta }({\theta _t}) - {\alpha _t}{{\cal F}^*}({y_t})
{\rm {,\;}}
{\theta _{t + 1}} = {\rm{pro}}{{\rm{x}}_{{\alpha _t}h}}({\theta _t} - {\alpha _t}{{\cal K}_y}{({\theta _t})^T}{y_t})
\end{array}
\end{equation}

\textbf{Lemma 1 (Optimal Convergence Rate):}
The optimal convergence rate of \ref{eq:primal-dual} is
$O(\frac{{L_{{F^{*}}}}}{{N^{2}}}+\frac{{L_{K}}}{N}+\frac{\sigma}{{\sqrt{N}}})$.

\textbf{Proof:}
Equation (\ref{eq:primal-1}) can be easily converted to the
following primal-dual formulation
\begin{equation}
\mathop{\min}\limits _{y\in Y}\mathop{\max}\limits _{\theta\in X}\left({\left\langle {-K(\theta),y}\right\rangle +{F^{*}}(y)-h(\theta)}\right)\label{eq:primal-dual-2}
\end{equation}
Using the bounds proved in \cite{nesterov2004introductory,juditsky2008solving,chen2013optimal},
the optimal convergence rate of stochastic saddle-point problem is
$O(\frac{{L_{{F^{*}}}}}{{N^{2}}}+\frac{{L_{K}}}{N}+\frac{\sigma}{{\sqrt{N}}})$.

The GTD/GTD2 algorithms can be considered as using Polyak's algorithm
without the primal average step. Hence, by adding the primal average
step, GTD/GTD2 algorithms will become standard Polyak's algorithms
\cite{polyak1992acceleration}, and thus the convergence rates are
$O(\frac{{{L_{{F^{*}}}}+{L_{K}}+\sigma}}{{\sqrt{N}}})$ according
to \cite{RobustSA:2009}. So we  have the following propositions.

\textbf{Proposition 1} The convergence rates of GTD/GTD2 algorithms
with primal average are $O(\frac{{{L_{{F^{*}}}}+{L_{K}}+\sigma}}{{\sqrt{N}}})$,
where ${L_{K}}=||{\Phi^{T}}\Xi(\Phi-\gamma{\Phi^{'T}})|{|^{2}}$,
for GTD, ${L_{{F^{*}}}}=1$ and for GTD2, ${L_{{F^{*}}}}=||M|{|_{2}}$.

Now we consider the acceleration using the SMP algorithm, which incorporates
the extragradient term. According to \cite{juditsky2008solving} which extends the SMP algorithm to solving saddle-point problems and variational inequality
problems, the convergence rate is accelerated to $O(\frac{{{L_{{F^{*}}}}+{L_{K}}}}{N}+\frac{\sigma}{{\sqrt{N}}})$. Consequently,

\textbf{Proposition 2 }The convergence rate of the GTD2-MP algorithm is
$O(\frac{{{L_{{F^{*}}}}+{L_{K}}}}{N}+\frac{\sigma}{{\sqrt{N}}})$.

\subsection{Value Approximation Error Bound}
\label{eq:mspbe-2}

One key question is how to give the error bound of $||V-{V_{\theta}}||$
given that of $||K(\theta)||$. Here we use the result in \cite{DantzigRL:2012},
which is similar to the one in \cite{Yu08newerror}.

\textbf{Lemma 2 \cite{DantzigRL:2012}:} For any ${V_{\theta}}=\Phi\theta$,
the following component-wise equality holds
\begin{equation}
V-{V_{\theta}}={(I-\gamma{\Pi^{\Xi}}P)^{-1}}\left({\left({V-{\Pi^{\Xi}}V}\right)+\Phi{({\Phi^{T}}\Xi\Phi)^{-1}}K(\theta)}\right)\label{eq:ds_v_ax-b-1}
\end{equation}
\textbf{Proof}:

Use the equality $V=TV$ and ${V_{\theta}}={\Pi^{\Xi}}{V_{\theta}}$,
where the first equality is the Bellman equation, and the second is that
$V_{\theta}$ lies within the spanning space of $\Phi$.

we have
\begin{equation}
\begin{array}{l}
V-{\Pi^{\Xi}}V\\
=V-{\Pi^{\Xi}}TV+({V_{\theta}}-{\Pi^{\Xi}}T{V_{\theta}})-({V_{\theta}}-{\Pi^{\Xi}}T{V_{\theta}})\\
=(I-\gamma{\Pi^{\Xi}}P)(V-{V_{\theta}})+{\Pi^{\Xi}}({V_{\theta}}-T{V_{\theta}})
\end{array}
\end{equation}
After rearranging the equation, we have Equation (\ref{eq:ds_v_ax-b-1}) and find that
\begin{equation}
{\Pi^{\Xi}}({V_{\theta}}-T{V_{\theta}})=-\Phi{({\Phi^{T}}\Xi\Phi)^{-1}}K(\theta)
\end{equation}

\textbf{Proposition }3: For GTD/GTD2, the prediction error of $||V-{V_{\theta}}||$
is bounded by $||V-{V_{\theta}}|{|_{\infty}}\le\frac{{L_{\phi}^{\Xi}}}{{1-\gamma}}\cdot O\left({\frac{{{L_{{F^{*}}}}+{L_{K}}+\sigma}}{{\sqrt{N}}}}\right)$
; For GTD2-MP, it is bounded by $||V-{V_{\theta}}|{|_{\infty}}\le\frac{{L_{\phi}^{\Xi}}}{{1-\gamma}}\cdot O\left({\frac{{{L_{{F^{*}}}}+{L_{K}}}}{N}+\frac{\sigma}{{\sqrt{N}}}}\right)$,
where $L_{\phi}^{\Xi}={\max_{s}}||{({\Phi^{T}}\Xi\Phi)^{-1}}\phi(s)|{|_{1}}$.

\textbf{Proof:}

From Lemma 2, we have
\begin{equation}
||V-{V_{\theta}}|{|_{\infty}}\le||{(I-\gamma{\Pi^{\Xi}}P)^{-1}}|{|_{\infty}}\cdot\left({||V-{\Pi^{\Xi}V}|{|_{\infty}}+L_{\phi}^{\Xi}||K(\theta)|{|_{\infty}}}\right)
\end{equation}
Using the results in Proposition 1 and Proposition 2, we have for GTD
and GTD2,
\begin{equation}
||V-{V_{\theta}}|{|_{\infty}}\le||{(I-\gamma{\Pi^{\Xi}}P)^{-1}}|{|_{\infty}}\cdot\left({||V-{\Pi^{\Xi}V}|{|_{\infty}}+L_{\phi}^{\Xi}\cdot O\left({\frac{{{L_{{F^{*}}}}+{L_{K}}+\sigma}}{{\sqrt{N}(1-\gamma)}}}\right)}\right)
\end{equation}
For GTD2-MP,
\begin{equation}
||V-{V_{\theta}}|{|_{\infty}}\le||{(I-\gamma{\Pi^{\Xi}}P)^{-1}}|{|_{\infty}}\cdot\left({||V-{\Pi^{\Xi}V}|{|_{\infty}}+L_{\phi}^{\Xi}\cdot O(\frac{{{L_{{F^{*}}}}+{L_{K}}}}{N}+\frac{\sigma}{{\sqrt{N}}})}\right)
\end{equation}
If we further assume a rich expressive hypothesis
space $\mathcal{H}$, i.e., ${\Pi^{\Xi}}P=P,{\Pi^{\Xi}}R=R$, $||V-{\Pi^{\Xi}}V|{|_{\infty}}=0,||{(I-\gamma{\Pi^{\Xi}}P)^{-1}}|{|_{\infty}}=\frac{1}{{1-\gamma}}$,
then for GTD and GTD2, we have
\begin{equation}
||V-{V_{\theta}}|{|_{\infty}}\le\frac{{L_{\phi}^{\Xi}}}{{1-\gamma}}\cdot O\left({\frac{{{L_{{F^{*}}}}+{L_{K}}+\sigma}}{{\sqrt{N}}}}\right)
\end{equation}
For GTD2-MP, we have
\begin{equation}
||V-{V_{\theta}}|{|_{\infty}}\le\frac{{L_{\phi}^{\Xi}}}{{1-\gamma}}\cdot O\left({\frac{{{L_{{F^{*}}}}+{L_{K}}}}{N}+\frac{\sigma}{{\sqrt{N}}}}\right)
\end{equation}

\bibliographystyle{unsrtnat}

\bibliography{thesisbib,sridhar-references-vi,vi,sridhar-tls-references,sridhar-references,sridhar-references2,boucher_uai2014,ma,allcitations,bregman,mirror-descent,online-convex,rdbook,compressive-sensing,mybib,bregman,mirror-descent,rdbook,compressed-sensing,mlft-paper}

\end{document}